\documentclass[11pt]{article}
\usepackage[letterpaper,top=1in,bottom=1in,left=1in,right=1in]{geometry}
\usepackage[parfill]{parskip}
\usepackage[round]{natbib}
\usepackage{amssymb,amsmath,amsthm,bbm,mathrsfs,mathtools}
\mathtoolsset{showonlyrefs}
\newcommand{\E}{\mathbb{E}} 
\renewcommand{\P}{\mathbb{P}} 
\newcommand{\R}{\mathbb{R}}
\newcommand{\N}{\mathbb{N}}
\newcommand{\trans}{\mathsf{T}}
\newcommand{\ind}{\mathbbm{1}}
\renewcommand{\S}{\mathcal{S}}
\newcommand{\A}{\mathcal{A}}

\DeclareMathOperator*{\argmin}{arg\,min}

\DeclarePairedDelimiter{\ceil}{\lceil}{\rceil}
\newtheorem{clm}{Claim}
\newtheorem{cor}{Corollary}
\newtheorem{lem}{Lemma}
\newtheorem{prop}{Proposition}
\newtheorem{thm}{Theorem}
\newtheorem{ass}{Assumption}
\newtheorem{defn}{Definition}
\newtheorem{rem}{Remark}
\usepackage{enumitem}
\usepackage[linesnumbered,ruled,noend,nosemicolon]{algorithm2e}
\SetKw{Return}{return}
\SetKw{And}{and}
\SetKw{Or}{or}
\makeatletter
\renewcommand{\@algocf@capt@plain}{above}
\makeatother
\SetInd{0pt}{10pt}
\usepackage{hyperref}
\hypersetup{
    colorlinks=true,
    linkcolor=blue,
   citecolor=blue
}
\usepackage{multirow}
\usepackage{makecell}
\usepackage{color}
\usepackage{authblk}
\usepackage[bottom]{footmisc}

\begin{document}

\title{Regret Bounds for Stochastic Shortest Path Problems with Linear Function Approximation}

\author[1]{Daniel Vial\thanks{Corresponding author. Email: \texttt{vial2@illinois.edu}.}}
\author[2]{Advait Parulekar}
\author[2]{Sanjay Shakkottai}
\author[1]{R.\ Srikant}
\affil[1]{University of Illinois at Urbana-Champaign}
\affil[2]{University of Texas at Austin}

\maketitle

\begin{abstract}
We propose an algorithm that uses linear function approximation (LFA) for stochastic shortest path (SSP). Under minimal assumptions, it obtains sublinear regret, is computationally efficient, and uses stationary policies. To our knowledge, this is the first such algorithm in the LFA literature (for SSP or other formulations). Our algorithm is a special case of a more general one, which achieves regret square root in the number of episodes given access to a certain computation oracle.
\end{abstract}

\section{Introduction} \label{secIntro}

To cope with the massive state spaces of modern reinforcement learning (RL) applications, a plethora of recent papers have studied function approximation. A particularly tractable case is linear function approximation (LFA). Here one assumes the transition kernel and cost vector are linear in known $d$-dimensional feature vectors, where typically $d \ll S$ and $A$ (the number of states and actions). In the online setting, an agent interacts with the Markov decision process over $T$ time steps (for infinite horizon average and discounted cost problems) or $K$ episodes (for finite horizon and stochastic shortest path problems). At a high level, one seeks algorithms with two properties:
\begin{itemize}[leftmargin=*,align=left,itemsep=0pt,topsep=0pt]
\item \textit{Statistically efficient:} regret independent of $S$ and $A$, sublinear (ideally, square root) in $T$ or $K$, and polynomial in $d$ and any other parameters.
\item \textit{Computationally efficient:} time and space complexity independent of $S$ and polynomial in $d$, $A$, $T$ or $K$, and any other parameters.
\end{itemize}

For the finite horizon problem, several algorithms have been shown to achieve both properties. A key question we address is whether such algorithms exist in settings where stationary policies are optimal, i.e., for stochastic shortest path (SSP) and average/discounted cost problems. To our knowledge, this problem is (essentially) open: state-of-the-art algorithms are computationally inefficient for average/discounted costs, and none have been proposed for SSP. (There is one exception for average cost but it uses non-stationary policies.) See Section \ref{secRelated} for related work. 

In addition to this theoretical point of interest, there is practical motivation for understanding stationary policy settings. First, stationary policies are simpler to deploy and (compared to large, but finite, horizons) less costly to store. Second, they do not require a notion of ``time zero,'' which may be ill-defined in practice. Third, RL applications like games with a random number of moves are best modeled in the stationary policy setting, in particular SSP, where the agent tries to minimize its cost before reaching a goal state.

{\bf Contributions:} Motivated by these theoretical and practical concerns, we provide the first algorithm for episodic SSP with LFA. More generally, this is the first statistically and computationally efficient LFA algorithm that uses stationary policies (in any setting). In more detail, our contributions are as follows:
\begin{itemize}[leftmargin=*,align=left,itemsep=0pt,topsep=0pt]
\item \textit{Optimistic approximate fixed points (OAFPs):} In Section \ref{secOAFP}, we show that under the LFA assumption, the optimal policy in an SSP can be computed from the fixed point of a $d$-dimensional Bellman operator, denoted by $G$ (see Proposition \ref{propFeatureBellman}). This is a simple observation, but it leads to an important definition of OAFPs (see Definition \ref{defnOAFP}). Roughly, these are $d$-dimensional vectors that have small Bellman error with respect to a data-driven operator $\hat{G}_t$ that we interpret as an optimistic approximation of $G$.

\item \textit{Regret bound with oracle:} In Section \ref{secRegret}, we assume access to an oracle that computes OAFPs from trajectories and propose Algorithm \ref{algMain}, which uses the oracle to update its policy. When the LFA assumption holds, the minimal cost for non-goal states $c_{min}$ is positive, and a proper policy exists (see Assumptions \ref{assSSPreg}-\ref{assSSPlin}), Theorem \ref{thmRegret} shows Algorithm \ref{algMain} achieves sublinear regret, with the exponent determined by the oracle's quality ($\sqrt{K}$ in the best case -- see Corollary \ref{corBestCase}). This reduces the problem of regret minimization to that of finding OAFPs (which exist with high probability, by the same theorem).

\item \textit{Oracle implementations:} In Section \ref{secOracle}, we show how to compute OAFPs. Combined with the results of Section \ref{secRegret}, this yields an efficient end-to-end algorithm with the following regret scaling in $K$:
\begin{itemize}[leftmargin=10pt,itemsep=0pt,topsep=0pt]
\item $K^{5/6}$ if Assumptions \ref{assSSPreg}-\ref{assSSPlin} hold (Theorem \ref{thmProb}).
\item $K^{3/4}$ if Assumptions \ref{assSSPreg}-\ref{assSSPlin} hold and all stationary policies are proper (Theorem \ref{thmProp}). 
\item $\sqrt{K}$ if Assumptions \ref{assSSPreg}-\ref{assSSPlin} hold and the features are orthogonal in a certain sense (Theorem \ref{thmOrth}).
\end{itemize}

\item \textit{Extensions:} In Section \ref{secExtensions}, we provide generalizations of Theorems \ref{thmProp} and \ref{thmOrth} and remove the $c_{min} > 0$ assumption. The latter point shows we can obtain sublinear regret and computational efficiency with stationary policies under minimal SSP assumptions.
\end{itemize}

\subsection{Related work} \label{secRelated}

{\bf Finite horizon LFA:} Several efficient algorithms have been proposed (of course, the policies are not stationary). To our knowledge, the earliest are \cite{jin2020provably,yang2020reinforcement,zanette2020frequentist}, which (like us) assume linear costs and transitions: $c(s,a) = \phi(s,a)^\trans \theta$ and $P(s'|s,a) = \phi(s,a)^\trans \mu(s')$ for known $\phi(s,a) \in \R^d$. The most relevant is \cite{jin2020provably}, which proposed an optimistic, least squares version of backward induction; our algorithm is the value iteration analogue. Subsequent work is too vast to survey here, but for later discussion, we note \cite{zhang2021variance,zhou2021nearly} proposed Berstein-style confidence sets for the related linear mixture model (see, e.g., \cite{ayoub2020model,jia2020model}), where $P(s'|s,a) = \varphi(s'|s,a)^\trans \vartheta$ for known $\varphi(s'|s,a)$.

{\bf Infinite horizon LFA:} Comparatively little is known for infinite horizons. \cite{wei2021learning,wu2021nearly} studied average costs under the minimal assumption that the optimal policy's long-term average reward is independent of the initial state (see references therein for work with stronger assumptions). The first algorithm in \cite{wei2021learning} has $\sqrt{T}$ regret assuming access to a certain fixed point oracle (analogous to our Algorithm \ref{algMain}) but no efficient oracle is provided, the second is computationally efficient with $T^{3/4}$ regret but uses non-stationary policies and requires knowledge of $T$ (unlike ours), and the third requires stronger assumptions. \cite{wu2021nearly} proved $\sqrt{T}$ regret for the linear mixture model, but the algorithm is inefficient due to computation\footnote{\cite[Appendix B]{zhou2020provably} provides a scheme to estimate the sums, but only in some special cases, and the estimation error is not accounted for in the regret analysis.} of $\sum_{s' \in \S}  h(s') \varphi(s'|s,a)$ for certain $h \in \R^\S$. Analogous algorithms are proposed in \cite{zhou2020nearly,zhou2020provably} for discounted costs, which have $\sqrt{T}$ regret but are inefficient for the same reason. Also, the discounted cost regret formulation is a bit unsatisfying, as it compares to the optimal policy along the algorithm's trajectory. Thus, one that stays in a bad set of states and only learns on this set can still have low regret. Finally, as mentioned above, we are not aware of any LFA papers that consider SSP.

{\bf Tabular SSP, $c_{min}^{-1}$ dependent:} \cite{tarbouriech2020no} proved $\tilde{O} ( D^{3/2} S \sqrt{ A K / c_{min} }  )$ regret, where $D$ is the SSP diameter (see their Assumption 2). \cite{rosenberg2020near} improved this to $\tilde{O} ( B_\star^{3/2} S \sqrt{ A K / c_{min} } )$, where $B_\star \leq D$ is the maximal cost-to-go of the optimal policy. Both algorithms use Hoeffding-style confidence sets and can be generalized to the case $c_{min} = 0$, though regret increases to $K^{2/3}$ (see Section \ref{secExtensions}).

{\bf Tabular SSP, $c_{min}^{-1}$ independent:} \cite{rosenberg2020near} also proved $\tilde{O} ( B_\star S \sqrt{  A K } )$ regret when $c_{min} = 0$, and the lower bound $\tilde{\Omega} ( B_\star \sqrt{ SA K } )$. Removing the $c_{min}^{-1}$ dependence required Berstein-style confidence sets, which \cite{chen2021implicit,cohen2021minimax,tarbouriech2021stochastic,jafarnia2021online} also employed. The former three showed UCB-based algorithms achieve the lower bound; the latter showed posterior sampling obtains $\tilde{O} ( B_\star S \sqrt{  A K } )$ regret. See references therein for prior work on SSP variants (e.g., adversarially changing costs).

\section{Preliminaries} \label{secPrelim}

{\bf Notation:} For $m \in \N$, we let $[m] = \{1,\ldots,m\}$. We write $\ind(\cdot)$ for the indicator function. We let $e_i$ be the vector with $j$-th element $e_i(j) = \ind ( i = j )$. For $x \in \R^d$ and positive definite $Y \in \R^{d \times d}$, $\| x \|_Y = \sqrt{ x^\trans Y x }$.

{\bf SSP:} An SSP instance is defined by $(\S,\A,P,c,s_{goal})$, where $\S$ is a set of $S = |\S| < \infty$ states, $\A$ is a set of $A = |\A| < \infty$ actions, $P$ is the transition kernel, $c$ is the cost vector, and $s_{goal} \in \S$ is an absorbing zero-cost state, i.e., $P(s_{goal}|s_{goal},a) = 1$ and $c(s_{goal},a) = 0$ for any $a \in \A$. Any stationary and deterministic policy $\pi : \S \rightarrow \A$ induces a trajectory $\{ s_t^\pi \}_{t=1}^\infty$, where $s_1^\pi$ is some initial state and $s_{t+1}^\pi \sim P ( \cdot | s_t^\pi , \pi ( s_t^\pi ) )$ for $t \in \N$. We call $\pi$ proper if $s_{goal}$ is reached with probability $1$ from any $s_1^\pi \in \S$; otherwise, we call it improper. We make the following assumption regarding proper policies, which we discuss in Remark \ref{remAssReg} below.

\begin{ass}[Basic properties] \label{assSSPreg}
There exists at least one proper policy, and for some $c_{min} > 0$ and any $(s,a) \in ( \S \setminus \{ s_{goal} \} ) \times \A$, $c(s,a) \in [c_{min},1]$.
\end{ass}

For any $\pi : \S \rightarrow \A$, we define the (possibly infinite) cost-to-go function $J^\pi : \S \rightarrow \R$ by
\begin{equation} \label{eqCostToGo}
 J^\pi(s) = \lim_{T \rightarrow \infty} \E \left[ \sum_{t=1}^T c( s_t^\pi , \pi(s_t^\pi) ) \middle| s_1^\pi  = s \right] .
\end{equation}
Given Assumption \ref{assSSPreg}, the optimal policy $\pi^\star$, i.e., the $\pi$ that minimizes $J^\pi(s)$ over all $s$, is stationary, deterministic, and proper \citep{bertsekas1991analysis}. It also satisfies the Bellman optimality equations
\begin{equation} \label{eqBellman}
 J^\star(s) = \min_{a \in \A} Q^\star(s,a) , \quad \pi^\star(s) \in \argmin_{a \in \A}  Q^\star(s,a)  ,
\end{equation}
where $J^\star = J^{\pi^\star}$ and the optimal state-action cost-to-go function $Q^\star : \S \times \A \rightarrow \R$ is given by
\begin{equation} \label{eqQstar}
 Q^\star (s,a) = c(s,a) + \sum_{ s' \in \S }  J^\star(s') P(s'|s,a) .
\end{equation}
Finally, we define $B_\star = \max_{s \in \S} J^\star(s)$.

\begin{rem}[Positive costs]\label{remAssReg}
We require $c(s,a) \geq c_{min}$ to show that episodes incurring finite total cost must terminate in finite time. In Section \ref{secExtensions}, we remove this assumption while still achieving sublinear regret and computational efficiency with stationary policies.
\end{rem}

{\bf Linearity:} As discussed in the introduction, we make the following assumption to enable LFA.
\begin{ass}[Linearity] \label{assSSPlin}
For some $d \geq 2$, there exists known $\{ \phi(s,a) \}_{ (s,a) \in \S  \times \A } \subset \R^d$, unknown $\theta \in \R^d$, and unknown $\{ \mu(s') \}_{ s' \in \S } \subset \R^d$, such that, for any $(s,a,s') \in ( \S \setminus \{ s_{goal} \} ) \times \A \times \S$,
\begin{gather}
c(s,a) = \phi(s,a)^\trans \theta ,\quad P ( s' | s,a ) = \phi(s,a)^\trans \mu(s') , \label{eqSSPlin} \\
\| \phi(s,a) \|_2 \leq 1 , \quad \| \theta \|_2 \leq \sqrt{d} , \label{eqSSPnorm1} \\
 \left\| \sum_{s' \in \S} h(s') \mu(s')  \right\|_2 \leq \sqrt{d} \| h \|_\infty\ \forall\ h \in \R^{\S} . \label{eqSSPnorm2}
\end{gather}
\end{ass}

This assumption naturally generalizes that of \cite{jin2020provably} to SSP. We also assume $d \geq 2$, which, given \eqref{eqSSPlin}, only eliminates a trivial case where $\phi(s,a)$ is independent of $(s,a)$. Finally, we assume without further loss of generality that $\phi(s_{goal},a) = 0\ \forall\ a \in \A$.

\begin{rem}[Tabular case] \label{remTabular}
Any SSP with $c(s,a) \in [0,1]$ satisfies Assumption \ref{assSSPlin} with $d = SA$, $\phi(s,a) = e_{(s,a)}$, $\theta = c$, and $\mu(s')  = \{ P(s' |s,a) \}_{(s,a) \in \S \times \A}$. 
\end{rem}

\begin{rem}[Realizability]
As shown in Appendix \ref{appProofFeatureBellman}, Assumption \ref{assSSPlin} implies $Q^\star$ is linear in the features. Ideally, we would only assume this, but recent work for finite horizons (a special case of SSP) has shown this problem is fundamentally harder \citep{du2020good,wang2021exponential,weisz2021exponential}.
\end{rem}

{\bf Regret:} We consider a protocol with $K$ episodes. For each $k \in [K]$, the agent begins at step $h=1$ at initial state $s_h^k$. At step $h$, the agent takes action $a_h^k$, incurs cost $c(s_h^k,a_h^k)$, and transitions to $s_{h+1}^k \sim P(\cdot | s_h^k, a_h^k)$. If $s_{h+1}^k = s_{goal}$, the episode terminates (without taking action $a_{h+1}^k$). We assume $s_1^k \neq s_{goal}$ without loss of generality but make no further assumptions on the sequence of initial states $\{ s_1^k \}_{k=1}^K$. We let $(s_t,a_t,s_t')$ denote the $t$-th state-action-state triple observed across all episodes. Hence, for each $t$, $s_t' \sim P(\cdot| s_t,a_t)$, and $s_t' = s_{t+1}$ \textit{unless} an episode ends at time $t$ (in which case $s_t' = s_{goal}$ and $s_{t+1} = s_1^{k+1}$, where $k$ is the episode that ended at $t$). We also let $T$ denote the random total number of steps across all $K$ episodes.\footnote{We reiterate $T$ is random for SSP, unlike the fixed $T$ used in the infinite horizon discussion of Section \ref{secIntro}.} As in the tabular SSP literature, we define the regret
\begin{equation}
R(K) = \sum_{t=1}^T c(s_t,a_t) - \sum_{k=1}^K J^\star(s_1^k),
\end{equation}
which is the difference between the total cost of the agent and the expected total cost of a ``genie'' who knows the optimal policy \textit{a priori} and runs it for $K$ episodes with the same initial states. 

\begin{rem}[Challenge 1] \label{remInfReg}
Unlike finite horizon LFA, no episode is guaranteed to end, since the agent may use improper policies. In this case, $T = \infty$ and we suffer infinite regret. Thus, we will need to detect improper policies and fix them within episodes, a challenge that does not arise for finite horizon LFA.
\end{rem}

\section{Optimistic approximate fixed point} \label{secOAFP}

To motivate the definition of OAFPs, we begin with the simple observation that for a linear SSP, the optimal policy can be computed from a feature space version of the Bellman operator (when the model is known). The proof is elementary; see Appendix \ref{appProofFeatureBellman}.
\begin{prop}[Feature space fixed point] \label{propFeatureBellman} 
Let Assumptions \ref{assSSPreg} and \ref{assSSPlin} hold. Define $G : \R^d \rightarrow \R^d$ by
\begin{equation}
 G w = \theta + \sum_{ s \in \S }  \min_{a \in \A} \phi(s,a)^\trans w \mu(s)\ \forall\ w \in \R^d .
\end{equation}
Then $w^\star = \theta + \sum_{ s \in \S }  J^\star(s) \mu(s)$ is a fixed point of $G$ (i.e., $G w^\star = w^\star$), $J^\star(s) = \min_{a \in \A} \phi(s,a)^\trans w^\star$, and
\begin{equation} \label{eqOptFromWd}
 \pi^\star(s) \in \argmin_{a \in \A} \phi(s,a)^\trans w^\star\ \forall\ s \in \S .
\end{equation}
\end{prop}

When the model is unknown, we instead must estimate $G w$ from data. Formally, let $\{ s_\tau , a_\tau , s_\tau' \}_{\tau=1}^t$ denote the first $t$ state-action-state triples as in Section \ref{secPrelim}, and define $\Lambda_t = I + \sum_{\tau=1}^t \phi(s_\tau,a_\tau) \phi(s_\tau,a_\tau)^\trans$. Then the regularized least-squares estimate of $G w$ is
\begin{equation} 
\tilde{G} w = {\Lambda}_t^{-1} \sum_{\tau=1}^t \phi(s_\tau,a_\tau) \Big( c(s_\tau,a_\tau) + \min_{a \in \A} \phi(s_\tau',a)^\trans w \Big) .
\end{equation}
Due to Assumption \ref{assSSPlin} and concentration, we should expect $\tilde{G} w \approx G w$ for any (bounded) $w$. Thus, it seems reasonable to find a fixed point $\tilde{w}^\star$ of $\tilde{G}$ and define policies like \eqref{eqOptFromWd}, with $w^\star$ replaced by $\tilde{w}^\star$.

This is roughly our approach, though we will modify $\tilde{G}$ in two ways. First, as is common for LFA, we subtract linear bandit-style bonuses \citep{abbasi2011improved} to encourage exploration. Namely, we consider the following optimistic estimate of $\min_a \phi(s,a)^\trans w$:
\begin{equation} \label{eqFtDefn}
f_t(s,w) = \min_{a \in \A} \left( \phi(s,a)^\trans w - \alpha_t \| \phi(s,a) \|_{\Lambda_t^{-1}}  \right) ,
\end{equation}
where $\alpha_t > 0$ is an exploration parameter. Second, and again common for LFA, we ``clip'' this estimate between $0$ and some $B_t > 0$ (see Remark \ref{remClipping}) to ensure bounded random variables, i.e., we define
\begin{equation} \label{eqGtDefn}
g_t(s,w) = \min \{ \max \{ f_t(s,w) , 0 \} , B_t \} .
\end{equation}
This yields the operator $\hat{G}_t : \R^d \rightarrow \R^d$ given by
\begin{equation}
 \hat{G}_t w = \Lambda_t^{-1} \sum_{\tau=1}^t \phi(s_\tau,a_\tau) \left( c(s_\tau,a_\tau) + g_t(s_\tau' , w ) \right) .
\end{equation}
Thus far, everything has naturally generalized finite horizon LFA. However, in the SSP setting, we will encounter several additional challenges.

\begin{rem}[Challenge 2] \label{remClipping}
In finite horizon LFA, one sets $B_t = H$ (the \emph{known} horizon). Since the optimal value is $[0,H]$-valued, clipping as in \eqref{eqGtDefn} only improves the optimal value estimate. In contrast, the analogous quantity in SSP is $B_\star$, which is \emph{unknown}. Hence, we will need to learn an upper bound $B_t \geq B_\star$ to ensure the clipping does not distort our $J^\star$ estimate.
\end{rem}

\begin{rem}[Challenge 3] \label{remRandomParam}
In light of Remark \ref{remClipping}, $B_t$, and thus $\alpha_t$ (which needs to scale with $B_t$ to ensure optimism), become trajectory-dependent random variables. This stands in contrast to other LFA settings, where the exploration parameter is deterministic.
\end{rem}

\begin{rem}[Challenge 4] \label{remFixedPt}
 In SSP, we need to find fixed points, which we will do by showing the iterates of $\hat{G}_t$ converge (see Remark \ref{remGtConvergence}). In contrast, finite horizon LFA uses a simple backward induction procedure, which basically iterates the operator $H$ times and does not require any sort of convergence.
\end{rem}

To overcome these issues, we break the problem into two parts, which treat Challenges 1-3 and 4, respectively. First, in Section \ref{secRegret}, we assume an oracle provides OAFPs, which we use to solve the regret minimization problem. Second, in Section \ref{secOracle}, we show how to compute OAFPs.

We define OAFPs as follows. In essense, we require the estimate \eqref{eqFtDefn} to be optimistic with respect to $J^\star$ (when $w$ in \eqref{eqFtDefn} is the OAFP), and the vector to be a fixed point of $\hat{G}_t$ up to some tolerance.
\begin{defn}[OAFP] \label{defnOAFP}
We say that $w \in \R^d$ is an optimistic approximate fixed point (OAFP) if
\begin{equation} \label{eqOAFP}
f_t(s,w) \leq J^\star(s)\ \forall\ s \in \S , \quad  \| \hat{G}_t w - w \|_{\Lambda_t} \leq \alpha_t .
\end{equation}
Note that by Cauchy-Schwarz, the latter bound implies
\begin{equation}\label{eqPhiFixedPoint}
| \phi(s,a)^\trans ( \hat{G}_t w - w ) | \leq \alpha_t \| \phi(s,a) \|_{\Lambda_t^{-1}} .
\end{equation}
\end{defn}

Finally, we note that due to the bonuses and clipping, $\hat{G}_t$ need not concentrate near $G$. Instead, Lemma \ref{lemMainEkTail2} in Appendix \ref{appProofGeneral} shows it concentrates near $U_t$, where $U_t w = \theta + \sum_{s \in \S} g_t(s,w) \mu(s)$. More specifically, we show that with high probability, for any bounded $w$,
\begin{equation} \label{eqOpConcMain}
| \phi(s,a)^\trans ( \hat{G}_t w - U_t w ) | = O(  \sqrt{\log t}  ) \| \phi(s,a) \|_{\Lambda_t^{-1}} .
\end{equation}
To prove \eqref{eqOpConcMain}, we use covering arguments to take union bounds over $w$, and the random functions $g_t$. This is similar to \cite{jin2020provably}, though we have the added complication of random (and dependent) $B_t$ and $\alpha_t$. For later use, we also note that by Assumption \ref{assSSPlin},
\begin{align} \label{eqPhiUtMain}
\phi(s,a)^\trans U_t w = c(s,a) + \E_{s'} g_t(s',w) ,
\end{align}
where $\E_{s'}$ is expectation with respect to $s' \sim P(\cdot|s,a)$.

\section{Regret minimization with oracle} \label{secRegret}

We can now describe Algorithm \ref{algMain}, which assumes access to an OAFP oracle -- i.e., a black box that, given $\{ s_\tau , a_\tau , s_\tau' \}_{\tau=1}^t$, returns an OAFP $w_t$ per Definition \ref{defnOAFP}.

{\bf Inputs:} The inputs are a failure probability $\delta$ and a sequence $\{ \kappa_t \}_{t=1}^\infty$ that will be used to shape $\alpha_t$ in \eqref{eqFtDefn} (we cannot define $\alpha_t$ \textit{a priori} due to Remark \ref{remRandomParam}).

{\bf Intervals:} As in tabular SSP, we split time into intervals indexed by $l$. The $l$-th interval will end at time $M_l$, which will either correspond to the end of an episode or an intra-episode policy update (see Remark \ref{remInfReg}). At each such $M_l$, we will call the oracle for an OAFP $w_{M_l}$, which will define the policy executed in interval $l+1$.

{\bf Initialization:} Lines \ref{lnInitStart}-\ref{lnInitEnd} initialize the regularizer $\Lambda_0 = I$, a (candidate) $B_\star$ upper bound $B_0$ (see Remark \ref{remClipping}), and the time and interval indices $t$ and $l$. We also set $w_0 = \alpha_0 = M_0 = 0$ to ensure the forthcoming notation is well-defined.

{\bf Episodic protocol:} Lines \ref{lnMainEpStart}-\ref{lnEpProtocol}, \ref{lnTriples}-\ref{lnLamUpdate}, and \ref{lnIncrementHT} implement the protocol from Section \ref{secPrelim}. Additionally, Line \ref{lnAction} chooses the action to minimize the optimistic cost-to-go estimate \eqref{eqFtDefn} with respect to the most recent OAFP $w_{M_{l-1}}$, and Line \ref{lnLamUpdate} updates $\Lambda_t$. When the last episode ends (if it ever does), Lines \ref{lnMainTermCond}-\ref{lnMainNoTerm} record the total number of intervals $L$ and the total time $T$.

{\bf Cost-to-go bound:} If the cost-to-go estimate exceeds $B_{t-1}$, then since $f_{M_{l-1}}(s_t',w_{M_{l-1}}) \leq B_\star$ by Definition \ref{defnOAFP}, we know $B_{t-1}$ was \textit{not} an upper bound for $B_\star$, so we double it (Lines \ref{lnBtDoubleCond}-\ref{lnBtDouble}). Otherwise, we let $B_t = B_{t-1}$ (Lines \ref{lnBtKeepCond}-\ref{lnBtKeep}). Having defined $B_t$, we use it and the input $\kappa_t$ to define $\alpha_t$ (Line \ref{lnAlpha}).
 
{\bf Policy update conditions:} Line \ref{lnUpdateCond} checks four conditions that require policy updates. The first three cause updates after the first observation, an episode ends, or $B_{t-1}$ doubles. The fourth, taken from \cite{abbasi2011improved}, is that the determinant of $\Lambda_t$ doubles. The idea is that, before this doubling occurs,
\begin{equation} \label{eqDoublingConsequence}
\| \phi(s,a) \|_{ \Lambda_t^{-1} } \leq \sqrt{2} \| \phi(s,a) \|_{ \Lambda_{M_{l-1}}^{-1} }\ \forall\ (s,a) \in \S \times \A ,
\end{equation}
which is analogous to tabular RL algorithms that wait to update until the number of visits to some $(s,a)$ double (e.g., \cite{jaksch2010near}).

{\bf Policy update:} If any of the conditions are met, Lines \ref{lnOracle}-\ref{lnIntEnd1} call the oracle for an OAFP $w_{M_l}$ and end the current interval. Note that in the next interval, the policy in Line \ref{lnAction} will use this OAFP.

\IncMargin{1.2em}
\begin{algorithm}[t] \label{algMain}

\caption{Regret minimization with oracle}

\KwIn{$\delta \in (0,1)$, $\{ \kappa_t \}_{t=1}^\infty \subset [1,\infty)$}

$\Lambda_0 = I_d$ (regularizer), $B_0 = c_{min}$ ($B_\star$ bound) \label{lnInitStart}

$w_0 = 0_d$ (OAFP), $\alpha_0 = 0$ (explore parameter)

$M_0 = 0$ (time $0$-th interval ended) 

$t= 1$ (current time), $l = 1$ (current interval) \label{lnInitEnd}

\For{episode $k = 1 , \ldots , K$}{ \label{lnMainEpStart}

$h=1$ (current step), observe $s_h^k \in \S\setminus \{ s_{goal} \}$ 

\While{$s_h^k \neq s_{goal}$}{ 

Choose $a_h^k \in \A$ to minimize \label{lnAction}
\begin{equation}
\phi(s_h^k,a_h^k)^\trans w_{M_{l-1}} - \alpha_{M_{l-1}} \| \phi(s_h^k,a_h^k) \|_{\Lambda_{M_{l-1}}^{-1}}
\end{equation}

Observe $c(s_h^k,a_h^k)$ and $s_{h+1}^k \sim P(\cdot|s_h^k,a_h^k)$ \label{lnEpProtocol}

\uIf{$k = K$ and $s_{h+1}^k = s_{goal}$}{ \label{lnMainTermCond}

$L = l$ (total number intervals), $M_L = t$

$T = t$ (total time elapsed)

}
\Else{ \label{lnMainNoTerm}

$(s_t,a_t,s_t') = (s_h^k,a_h^k,s_{h+1}^k)$ \label{lnTriples}

$\Lambda_t = \Lambda_{t-1} + \phi(s_t,a_t) \phi(s_t,a_t)^\trans$ \label{lnLamUpdate}

\If{$f_{M_{l-1}}(s_t',w_{M_{l-1}}) > B_{t-1}$}{ \label{lnBtDoubleCond}

$B_t = 2 B_{t-1}$ \label{lnBtDouble}

} 
\Else{ \label{lnBtKeepCond}

$B_t = B_{t-1}$ \label{lnBtKeep}

} 

$\alpha_t = (B_t+1) \kappa_t \sqrt{ \log( t (B_t+1) \kappa_t / \delta)}$ \label{lnAlpha}

\If{$t=1$ \Or $s_t' = s_{goal}$ \Or $B_t \neq B_{t-1}$ \Or $det(\Lambda_t) \geq 2 \det(\Lambda_{M_{l-1}})$}{ \label{lnUpdateCond}

Call oracle for OAFP $w_t$ (Def.\ \ref{defnOAFP}) \label{lnOracle}

$M_l = t$, $l \leftarrow l+1$ \label{lnIntEnd1}

} \label{lnMainPolicyEnd}

$t \leftarrow t+1$, $h \leftarrow h+1$ \label{lnIncrementHT}

}

}

}

\end{algorithm}
\DecMargin{1.2em}

We now present the main result of this section (Theorem \ref{thmRegret}), which assumes the input $\kappa_t$ to Algorithm \ref{thmRegret} scales as $t^\lambda$ for some $\lambda \in [0,\frac{1}{2})$. Provided this holds, the theorem shows that Algorithm \ref{algMain} obtains $K^{\frac{1}{2} + \lambda}$ regret, i.e., smaller $\kappa_t$ yields lower regret. The tradeoff is that smaller $\kappa_t$ means smaller $\alpha_t$ (see Line \ref{lnAlpha}), so lower regret requires the OAFP to be a tighter fixed point and yield sharper cost-to-go estimates (see Definition \ref{defnOAFP}). Note this is only a computational issue (not a statistical one), because Theorem \ref{thmRegret} shows that OAFPs exist even when $\lambda = 0$ (where we obtain the optimal $\sqrt{K}$ rate). Hence, the tradeoff is not too relevant in this section, though it will be in Section \ref{secOracle}.

\begin{thm}[General result]\label{thmRegret}
Suppose Assumptions \ref{assSSPreg} and \ref{assSSPlin} hold and $\kappa_t \in [ 9 d , \Psi t^\lambda \log (t+1) ]$ for some $\Psi> 0$ independent of $t$ and some absolute constant $\lambda \in [0,\frac{1}{2})$. With probability at least $1-\delta$, there exists an OAFP for all $t \in [T]$ and
\begin{align}
R(K) = \tilde{O} \Big( & ( B_\star^{\frac{3}{2}+\lambda} + B_\star^{\frac{1}{2}+\lambda} ) d^{\frac{1}{2}} \Psi ( K / c_{min} )^{\frac{1}{2}+\lambda} \\
&  + (B_\star+1)^{\frac{2}{1-2\lambda}} d^{\frac{1}{1-2\lambda} } \Psi^{\frac{2}{1-2\lambda}} c_{min}^{-\frac{1+2\lambda}{1-2\lambda}} \Big) .
\end{align}
\end{thm}

Thus, Algorithm \ref{algMain} ensures $\sqrt{K}$ regret when given an oracle that returns OAFPs for $\lambda = 0$. This is analogous to \cite{wei2021learning,zanette2020learning}, which provide $\sqrt{K}$ regret for average cost and finite horizon problems when given certain optimization oracles. More specifically, in the best case $\kappa_t = 9d$ permitted by Theorem \ref{thmRegret}, we have the following corollary.
\begin{cor}[Best case] \label{corBestCase}
Suppose Assumptions \ref{assSSPreg} and \ref{assSSPlin} hold, $B_\star \geq 1$, and $\kappa_t = 9d$. With probability at least $1-\delta$, there exists an OAFP for all $t \in [T]$ and
\begin{equation}
R(K) = \tilde{O} \Big( \sqrt{ B_\star^3 d^3 K / c_{min} } + B_\star^{2} d^3 / c_{min}\Big) .
\end{equation}
\end{cor}

Note Corollary \ref{corBestCase} also assumes $B_\star \geq 1$, which is natural (otherwise, $J^\star$ can arbitrarily smaller than the cost upper bound $1 \geq c(s,a)$). Of course, the $B_\star < 1$ case can be recovered from Theorem \ref{thmRegret}. Forthcoming results also assume $B_\star \geq 1$, but we report bounds for the general case in the appendix.

There are no existing LFA bounds for SSP to compare with, so we consider the tabular case. Here we obtain $R(K) = \tilde{O} ( \sqrt{ B_\star^3 S^3 A^3 K / c_{min} } )$ for large $K$, which matches the best Hoeffding algorithms in terms of $B_\star$, $K$, and $c_{min}$ (see Section \ref{secRelated}). Our scaling $(SA)^{3/2}$ is worse but can be improved to within a $\sqrt{A}$ factor, i.e., to $SA$ (see Remark \ref{remSharpeningTabular2} in Appendix \ref{appRegretProof}). We note a similar gap arises when specializing \cite{jin2020provably}'s bound to the tabular finite horizon setting.

\begin{proof}[Theorem \ref{thmRegret} proof sketch]
The proof is in Appendix \ref{appRegretProof} but we discuss the key ideas for the regret bound here. For simplicity, we set $\lambda = 0$ and show $R(K) = O ( \sqrt{K} )$ while hiding terms independent of $K$. Again for simplicity, we use $f_t$ and its clipping $g_t$ interchangeably.

{\bf Regret decomposition:} Fix $\tilde{T} \in \N$ and let $\tilde{K}$ and $\tilde{L}$ denote the number of episodes and intervals completed by time $T \wedge \tilde{T}$. In light of Remark \ref{remInfReg}, we will bound regret by time $T \wedge \tilde{T}$, show it is finite, and let $\tilde{T} \rightarrow \infty$. More specifically, let $\tilde{R}(\tilde{T}) = \tilde{R}_1(\tilde{T}) + \tilde{R}_2(\tilde{T})$, where we define the per-interval regret
\begin{equation} \label{eqPerInterval}
\tilde{R}_1(\tilde{T}) = \sum_{l=0}^{\tilde{L}-1}  \sum_{t = 1+M_l}^{M_{l+1}}  c(s_t,a_t) - J^\star ( s_{1+M_l} )  ,
\end{equation}
and the ``excess regret'' from intra-episode updates
\begin{equation}
\tilde{R}_2(\tilde{T}) = \sum_{l=0}^{\tilde{L}-1} J^\star ( s_{1+M_l} ) - \sum_{k=1}^{\tilde{K}} J^\star(s_1^k) .
\end{equation}

{\bf Cost-to-go bound:} To bound both terms, we require a bound on $B_t$. Since $f_{M_{l-1}}(s_t',w_{M_{l-1}}) \leq B_\star$ by Definition \ref{defnOAFP}, as soon as $B_{t-1}$ exceeds $B_\star$, the condition Line \ref{lnBtDoubleCond} will stop occurring. This implies $B_t \leq 2 B_\star$.

{\bf Per-interval regret:} First note that by \eqref{eqOAFP},
\begin{equation} \label{eqPerIntervalOpt}
\tilde{R}_1(\tilde{T}) \leq \sum_{l=0}^{\tilde{L}-1}  \sum_{t = 1+M_l}^{M_{l+1}}  c(s_t,a_t) - f_{M_l}(s_{1+M_l}, w_{M_l} ) .
\end{equation}
Now fix $l$ and $t$ as the double summation. Then by the chosen policy (Line \ref{lnAction} of Algorithm \ref{algMain}), we know
\begin{equation}
f_{M_l}(s_t,w_{M_l}) \approx \phi(s_t,a_t)^\trans w_{M_l} ,
\end{equation}
where $\approx$ hides the bonus term $\alpha_{M_l} \| \phi(s_t,a_t) \|_{\Lambda_{M_l}^{-1}}$. Again up to the bonus, \eqref{eqPhiFixedPoint} and \eqref{eqOpConcMain} imply
\begin{equation}
\phi(s_t,a_t)^\trans w_{M_l} \approx \phi(s_t,a_t)^\trans U_t w_{M_l} .
\end{equation}
Finally, by \eqref{eqPhiUtMain}, up to a conditionally zero-mean term,
\begin{equation}
\phi(s_t,a_t)^\trans U_t w_{M_l} \approx c(s_t,a_t) + f_{M_l}(s_{t+1}, w_{M_l}) .
\end{equation}
Combining the last three inequalities, we obtain
\begin{equation}
c(s_t,a_t) - f_{M_l}(s_t, w_{M_l}) \approx - f_{M_l}(s_{t+1}, w_{M_l}) .
\end{equation}
Iterating in \eqref{eqPerIntervalOpt}, this implies $\tilde{R}_1(\tilde{T}) \approx 0$, where $\approx$ hides a sum of $T \wedge \tilde{T}$ zero-mean terms and bonuses. Both are $\tilde{O}(\sqrt{ T \wedge \tilde{T} })$, because $\alpha_t = \tilde{O}(1)$ by $\lambda = 0$ and the above proof that $B_t \leq 2 B_\star = O(1)$.

{\bf Excess regret:} By definition, $\tilde{R}_2(\tilde{T}) \leq B_\star(\tilde{L} - \tilde{K})$, where $\tilde{L} - \tilde{K}$ is the number of intra-episode episodes, i.e., the number of times $B_t$ or $det(\Lambda_t)$ double. The former occurs $O(1)$ times since $B_t = O(1)$ and the latter $\tilde{O}(1)$ times since $det(\Lambda_t) = O(t)$. Thus, $\tilde{R}_2(\tilde{T})$ is dominated by $\tilde{R}_1(\tilde{T})$ (in terms of $\tilde{T}$).

{\bf Completing the proof:} So far, we have argued $\tilde{R}(\tilde{T}) = O ( \sqrt{ T \wedge \tilde{T}} )$. By definition, we also know
\begin{equation}
(T \wedge \tilde{T} ) c_{min} \leq \sum_{t=1}^{T \wedge \tilde{T}} c(s_t,a_t) = \tilde{R}(\tilde{T}) + \sum_{k=1}^{\tilde{K}} J^\star(s_1^k) .
\end{equation}
Combining, we obtain $T \wedge \tilde{T}  = O ( \sqrt{T \wedge \tilde{T}}  + K )$, which implies $T \wedge \tilde{T} = O(K)$. Thus, choosing $\tilde{T} \gg K$, we conclude $T = T \wedge \tilde{T} = O(K)$, so $R(K) = \tilde{R}(\tilde{T})$ and $\sqrt{ T \wedge \tilde{T}} = O(\sqrt{K})$. Plugging into the bound $\tilde{R}(\tilde{T}) = O ( \sqrt{ T \wedge \tilde{T}} )$ completes the proof.
\end{proof}

\begin{rem}[Finite $T$] \label{remMistake}
It is tempting to choose $\tilde{T} = \infty$ at the start of the proof, show $T = O ( \sqrt{T} + K )$ as above, and conclude $T = O(K) < \infty$. However, such logic is circular: it \emph{assumes} $T$ is finite (e.g., to justify adding/subtracting $T$ terms) in order to \emph{prove} it is finite. We point out this mistake (which some tabular SSP papers have made) so future work can avoid it.
\end{rem}

\begin{rem}[Complexity] \label{remComp}
Algorithm \ref{algMain}'s runtime is dominated by computation of $\{ \pi(s_t) \}_{t=1}^T$, which is $O( A d^2 T)$ when $\Lambda_t^{-1}$ and $det(\Lambda_t)$ are iteratively updated. In the proof, we show $T$ is polynomial in all parameters (see Remark \ref{remTbound} in Appendix \ref{appRegretProof}), so given an efficient oracle, Algorithm \ref{algMain} is itself efficient.
\end{rem}

\section{Oracle implementation} \label{secOracle}

We next discuss how to compute OAFPs. The obvious approach is to iterate $\hat{G}_t$. This indeed yields optimistic estimates, i.e., the first inequality in \eqref{eqOAFP} will hold.
\begin{lem}[Informal version of Corollary \ref{corGtOptimism} from Appendix \ref{appOracleProof}] \label{lemGtOptimismInformal} 
With high probability, if $\alpha_t = \Omega(\sqrt{\log t})$,
\begin{equation}
f_t(s, \hat{G}_t^{n-1} 0 ) \leq J^\star(s)\ \forall\ s \in \S , n \in \N , t \in [T] .
\end{equation}
\end{lem}
\begin{proof}[Proof sketch] 
When $n = 0$, the bound is immediate, since $f_t(s,\hat{G}_t^0 0) = f_t(s,0) \leq 0$. If true for $n$, then
\begin{equation}
g_t(s' , \hat{G}_t^{n-1} 0 ) \leq \max \{ f_t(s' , \hat{G}_t^{n-1} 0 ) , 0 \} \leq J^\star(s') .
\end{equation}
Thus, by \eqref{eqPhiUtMain} and Bellman optimality \eqref{eqBellman},
\begin{equation}
\phi(s,a)^\trans U_t ( \hat{G}_t^{n-1} 0 ) \leq c(s,a) + \E_{s'} J^\star(s') = Q^\star(s,a) ,
\end{equation}
so by \eqref{eqOpConcMain}, $\alpha_t = \Omega(\sqrt{\log t})$, and \eqref{eqBellman},
\begin{align}
f_t(s, \hat{G}_t^n 0 ) & \leq \min_{a \in \A} Q^\star(s,a) = J^\star(s) . \qedhere
\end{align}
\end{proof}

\IncMargin{1.2em}
\begin{algorithm}[t]
\caption{Computing OAFPs} \label{algIterate}

Set $n=1$, compute $\hat{G}_t^n 0 = \hat{G}_t 0$ and $\hat{G}_t^{n-1} 0 = 0$

\While{$\| \hat{G}_t^n 0 - \hat{G}_t^{n-1} \|_{\Lambda_t} > \alpha_t$}{

Set $n \leftarrow n+1$, compute $\hat{G}_t^n 0 = \hat{G}_t ( \hat{G}_t^{n-1} 0 )$

}

\Return{$w_t = \hat{G}_t^{n-1} 0$}
\end{algorithm}
\DecMargin{1.2em}

We thus propose Algorithm \ref{algIterate} for OAFP computation, which iterates $\hat{G}_t$ until the second inequality in \eqref{eqOAFP} holds (the first holds by Lemma \ref{lemGtOptimismInformal}). Our next theorem shows that, for appropriate $\kappa_t$, it terminates in polynomial iterations. Combined with Remark \ref{remComp}, this shows Algorithms \ref{algMain} and \ref{algIterate} provide an end-to-end statistically/computationally efficient scheme that uses stationary policies -- a \textit{first} in the LFA literature.

\begin{thm}[End-to-end algorithm]\label{thmProb}
Suppose Assumptions \ref{assSSPreg} and \ref{assSSPlin} hold, $B_\star \geq 1$, $\kappa_t = 54 d t^{1/3}$, and Algorithm \ref{algIterate} is the oracle. With probability at least $1-\delta$, Algorithm \ref{algIterate} returns an OAFP within $O(d t^{1/6})$ iterations for each $t \in [T]$ it is called, and
\begin{equation}
R(K) = \tilde{O} \Big( B_\star^{\frac{11}{6}} d^{\frac{3}{2}} ( K / c_{min} )^{\frac{5}{6}} + B_\star^6 d^9 c_{min}^{-5} \Big) .
\end{equation}
\end{thm}
\begin{proof}[Proof sketch]
The proof (and those of Theorems \ref{thmProp} and \ref{thmOrth}) can be found in Appendix \ref{appOracleProof}. Given Theorem \ref{thmRegret} and Lemma \ref{lemGtOptimismInformal}, the remaining challenge is to show Algorithm \ref{algIterate} terminates, i.e., $\| \hat{G}_t^n 0 - \hat{G}_t^{n-1} 0 \|_{\Lambda_t} \leq \alpha_t$ for some $n = O ( dt^{1/6} )$. Equivalently, if we ignore the regularizer, then by definition of $\|\cdot\|_{\Lambda_t}$, we aim to show
\begin{align} \label{eqExploredGoal}
\sum_{\tau=1}^t ( \phi(s_\tau,a_\tau)^\trans ( \hat{G}_t^{n} 0 - \hat{G}_t^{n-1} 0 ) )^2 \leq \alpha_t^2 .
\end{align}
To bound the $\tau$-th summand, we show $U_t$ converges, $\hat{G}_t$ tracks $U_t$, and use the triangle inequality.

{\bf $U_t$ converges:} \cite[Lemma 4.3]{bonet2007speed} implies the standard Bellman iterates converge at rate $\frac{SA}{n}$. By \eqref{eqPhiUtMain}, $\phi(s,a)^\trans U_t^n 0$ are basically the same iterates (up to bonuses and clipping), which means they converge at rate $\frac{SA}{n}$ as well. The constant $SA$ is infeasible, but with a more careful analysis, we can exploit the low rank structure to show (in terms of $d$ and $n$)
\begin{equation} \label{eqUtConvergeMain}
\max_{(s,a) \in \S \times \A} | \phi(s,a)^\trans ( U_t^n 0 - U_t^{n-1} 0 ) | = O \Big( \frac{ d^2}{n} \Big) .
\end{equation}

{\bf $\hat{G}_t$ tracks $U_t$:} Let $x_n = \hat{G}_t^n 0$ and $y_n = U_t^n 0$. By \eqref{eqOpConcMain},
\begin{equation}
| \phi(s,a)^\trans ( x_{n+1} - U_t x_n )| = \tilde{O}(1) | \phi(s,a) \|_{\Lambda_t^{-1}} .
\end{equation}
On the other hand, \eqref{eqPhiUtMain} implies
\begin{equation}
| \phi(s,a)^\trans ( U_t x_n - y_{n+1} )  | \leq \E_{s'} | g_t(s', x_n) - g_t(s', y_n) | .
\end{equation}
Combining and using the triangle inequality, we obtain
\begin{align}
&  | \phi(s,a)^\trans ( x_{n+1}   - y_{n+1} ) |= \tilde{O}(1) | \phi(s,a) \|_{\Lambda_t^{-1}} \\
& \qquad + \E_{s'} | g_t(s', x_n) - g_t(s', y_n)  | . \label{eqExploredSketch} 
\end{align}
Finally, a straightforward calculation yields
\begin{equation} 
| g_t(s', x_n) - g_t(s', y_n)  | \leq \max_{a'} | \phi(s',a')^\trans (x_n-y_n) | .
\end{equation}
This suggests bounding the average in \eqref{eqExploredSketch} by the max (over $s' \in \S$) and iterating. However, such a bound involves $\max_{(s,a) \in \S \times \A} \| \phi(s,a) \|_{\Lambda_t^{-1}}$, which is too large. The crucial idea is to take max only over \textit{explored} states, namely, $s'$ such that $\max_{a'} \| \phi(s',a') \|_{\Lambda_t^{-1}} \ll \alpha_t^{-1}$. The key implication is that if $s'$ is \textit{un}explored, then $\alpha_t \| \phi(s',a') \|_{\Lambda_t^{-1}} \gg 1$ for some $a'$, so $f_t(s',x_n) \leq 0$ by definition and $g_t(s',x_n) = 0$ by clipping (and similar for $y_n$). This insight allows us to iterate the above, but only over $(s,a) \in \S_t \times \A$, to obtain
\begin{equation}
\max_{(s,a) \in \S_t \times \A} | \phi(s,a)^\trans (x_n-y_n) | = \tilde{O} ( n / \alpha_t ) .
\end{equation}
Plugging into \eqref{eqExploredSketch} and recalling $x_n = \hat{G}_t^n 0$ and $y_n = U_t^n 0$, this extends to \textit{all} $(s,a) \in \S \times \A$ as follows:
\begin{equation} \label{eqGtracksUmain}
| \phi(s,a)^\trans ( \hat{G}_t^n 0 - U_t^n 0 ) | = \tilde{O} \Big( \| \phi(s,a) \|_{\Lambda_t^{-1}} + \frac{n}{\alpha_t} \Big) .
\end{equation}

{\bf Completing the proof:} By \eqref{eqUtConvergeMain} and \eqref{eqGtracksUmain}, the $\tau$-th summand in \eqref{eqExploredGoal} is $\tilde{O}( \| \phi(s,a) \|_{\Lambda_t^{-1}}^2 + ( \frac{n}{\alpha_t} + \frac{1}{n} )^2 )$ (in terms of $n$ and $t$). This yields a sum of squared bonuses, which is independent of $t$, plus $O ( t ( \frac{1}{n} + \frac{n}{\alpha_t} )^2 )$. Finally, since $\alpha_t = O ( t^{1/3} )$ by choice of $\kappa_t$, after $n = O(t^{1/6})$ iterations, $t ( \frac{1}{n} + \frac{n}{\alpha_t} )^2 = O ( t^{2/3} ) = O ( \alpha_t^2 )$.
\end{proof}

\begin{rem}[Clipping]
Most LFA papers use clipping to show an event like \eqref{eqOpConcMain} occurs with high probability, then bound regret on this event, after which clipping becomes somewhat of a nuisance. In contrast, the proof sketch \emph{exploits} it on the high probability event.
\end{rem}

\begin{rem}[Convergence] \label{remGtConvergence}
The proof sketch shows $\| x_t \|_{\Lambda_t} = O ( t^{1/3} )$, where $x_t = \hat{G}_t^{n_t} 0 - \hat{G}_t^{n_t-1} 0$ is the fixed point error after $n_t = O(t^{1/6})$ iterations. Note the norm equivalence $\|x_t\|_{\Lambda_t} = O ( \sqrt{t} ) \|x_t\|_2$ always holds, so if it is reasonably tight (i.e., if $\| x_t \|_2 = o(t^{-1/3}) \| x_t \|_{\Lambda_t}$), then $x_t \rightarrow 0$ as $t \rightarrow \infty$ (i.e., Algorithm \ref{algIterate} yields a fixed point asymptotically in $t$).
\end{rem}

If we strengthen Assumption \ref{assSSPreg} to mandate that \textit{all} stationary policies are proper, we can improve Theorem \ref{thmProb}'s regret bound. While this assumption is arguably strong, it seems perfectly reasonable for, e.g., games that eventually end. The benefit is that the Bellman operator $\mathcal{T} : \S \times \A \rightarrow \R$ given by
\begin{equation} \label{eqOperator}
(\mathcal{T} Q )(s,a)= c(s,a) + \E_{s'} \min_{a' \in \A} Q(s',a') 
\end{equation}
is contractive. More precisely, for some $\rho \in (0,1)$ and $\omega(s) > 0$, if $\| x \| = \max_{(s,a) \in \S \times \A} \omega(s) | x(s,a) |$, then
\begin{align} \label{eqContraction}
\| \mathcal{T} Q_1 - \mathcal{T} Q_2 \| \leq \rho \| Q_1 - Q_2 \| .
\end{align}
Define $\chi = \max_{s \in \S } \omega(s) / \min_{s \in \S } \omega(s)$. Assuming nontrivial upper bounds $\bar{\rho} \in [\rho,1)$ and $\bar{\chi} \in [\chi,\infty)$ are known, our next result establishes $K^{3/4}$ regret.
\begin{thm}[All proper] \label{thmProp}
Suppose Assumptions \ref{assSSPreg} and \ref{assSSPlin} hold, $B_\star \geq 1$, all stationary policies are proper, $\kappa_t = 54 d t^{1/4} \sqrt{N_t}$ with $N_t = \log(3 t \bar{\chi}) / (1-\bar{\rho})$, and Algorithm \ref{algIterate} is the oracle. With probability at least $1-\delta$, Algorithm \ref{algIterate} returns an OAFP within $N_t$ iterations for each $t \in [T]$ it is called, and
\begin{align}
R(K) & = \tilde{O} \Big(  B_\star^{\frac{7}{4}} d^{\frac{3}{2}} ( K / c_{min} )^{\frac{3}{4}} N_t^{1/2} + B_\star^4 d^6 N_t^2 c_{min}^{-3} \Big) .
\end{align}
\end{thm}
\begin{proof}[Proof sketch]
Recall in the Theorem \ref{thmProb} proof sketch, we showed $\| \hat{G}_t^n 0 - \hat{G}_t^{n-1} \|_{\Lambda_t} = O ( \sqrt{t} ( \frac{1}{n} + \frac{n}{\alpha_t } ) )$, where $\frac{1}{n}$ was the $U_t$ convergence rate. Under the stronger assumption of Theorem \ref{thmProp}, $U_t$ inherits a contraction property from \eqref{eqContraction}, which improves the rate to $\rho^n$. Hence, after $N_t$ iterations, we have $\| \hat{G}_t^n 0 - \hat{G}_t^{n-1} \|_{\Lambda_t} = \tilde{O} ( \sqrt{t} / \alpha_t ) = \tilde{O}(\alpha_t)$ by the choice $\kappa_t = \tilde{O}(t^{1/4})$.
\end{proof}

Finally, we demonstrate a nontrival case where Algorithm \ref{algIterate} returns OAFPs for the best case $\kappa_t$.
\begin{thm}[Orthogonal features] \label{thmOrth}
Suppose Assumptions \ref{assSSPreg} and \ref{assSSPlin} hold, $B_\star \geq 1$, $\{ \phi(s,a) \}_{(s,a) \in \S \times \A} \subset \{ q_i \}_{i=1}^d$ for some orthonormal set $\{ q_i \}_{i=1}^d$, $\kappa_t = 9d$, and Algorithm \ref{algIterate} is the oracle. With probability at least $1-\delta$, Algorithm \ref{algIterate} returns an OAFP within $\tilde{O}(t)$ iterations for each $t \in [T]$ it is called, and regret is bounded as in Corollary \ref{corBestCase}.
\end{thm}
\begin{proof}[Proof sketch]
The additional assumption yields an explicit expression for $\Lambda_t^{-1}$, which allows us to show $\hat{G}_t$ itself is contractive. This enables a direct convergence proof, i.e., without comparing to the iterates of $U_t$.
\end{proof}

\section{Extensions} \label{secExtensions}

Before closing, we mention some extensions of our results. We defer the details to Appendix \ref{appExtensions}.

{\bf Generalizing Theorem \ref{thmProp}:} When the upper bounds $\bar{\chi}$ and $\bar{\rho}$ are unavailable, we can instead set $N_t = t^{2 \gamma}$ for some absolute constant $\gamma \in (0,\frac{1}{4})$ and modify Algorithm \ref{algIterate} to terminate after $N_t$ iterations (if it has not already). This approach is efficient by design, returns OAFPs for $t \geq \Gamma = \tilde{O} ( ( \frac{ \log \chi }{ 1-\rho } )^{\frac{1}{2\gamma}} )$, and (combined with Algorithm \ref{algMain}) achieves the Theorem \ref{thmRegret} regret bound with $\lambda = \frac{1}{4} + \gamma$ and an additive $\Gamma$ term.

{\bf Generalizing Theorem \ref{thmOrth}:} When $\{ \phi(s,a) \}$ is not orthogonal but there at most $d'$ unique features, they can be orthogonalized to recover the $\sqrt{K}$ regret bound from Theorem \ref{thmOrth}, with $d'$ replaced by $d$. This is efficient if $d' \ll SA$, which is reminiscent of state aggregation. 

{\bf Zero/vanishing costs:} Suppose we modify Assumption \ref{assSSPreg} to allow for $c_{min} = 0$, which is the minimal assumption in tabular SSP (the upper bound $c(s,a) \leq 1$ can be easily generalized). In this setting, as in the tabular case, we define regret with respect to the optimal \textit{proper} policy $\pi_{\text{prop}}^\star$. We use the same algorithms but replace $c(s,a)$ with $c(s,a) + \eta$ for some small perturbation $\eta > 0$ in the definition of $\hat{G}_t$, invoke Theorem \ref{thmRegret} to bound the regret of this algorithm with respect to the optimal policy in the perturbed SSP (which remains linear), and compare the cost-to-go of the latter with that of $\pi_{\text{prop}}^\star$. With $\kappa_t$ scaling as $t^\lambda$ for some $\lambda \in [0,\frac{1}{2})$ (as in Theorem \ref{thmRegret}) and $\eta$ as $K^{(2\lambda-1)/(2\lambda+3)}$, this yields $K^{(4\lambda+2)/(2\lambda+3)}$ regret.\footnote{If $K$ is unknown, we can use a standard doubling trick.} Since $\frac{4\lambda+2}{2\lambda+3} < 1$ for any $\lambda \in [0,\frac{1}{2})$, Algorithms \ref{algMain}-\ref{algIterate} with the Theorem \ref{thmProb} parameters obtain \textit{statistical/computational efficiency with stationary policies under minimal SSP assumptions}. Note this approach also works if Assumption \ref{assSSPreg} holds but $c_{min}$ vanishes in $K$. For example, Corollary \ref{corBestCase} only promises linear regret when $c_{min} = K^{-1}$, but choosing $\eta = K^{-1/3}$ ensures $K^{2/3}$ regret.

\begin{rem}[$c_{min}^{-1}$ dependence] \label{}
As seen above, the $c_{min}^{-1}$ dependence of the leading term in our regret bound inflates the scaling in $K$ when dealing with small costs. Consequently, it would be ideal if this term was independent of $c_{min}^{-1}$. In the special case of tabular SSP, avoiding this issue seems to require Bernstein-style confidence sets, which for LFA have only been studied recently and only for simple finite horizon problems (see Section \ref{secRelated}). Given the unique LFA challenges that arise for SSP (see Remarks \ref{remInfReg}, \ref{remClipping}, \ref{remRandomParam}, and \ref{remFixedPt}), we leave such bounds for future work.
\end{rem}

\section{Conclusion} \label{secConclusion}

In this paper, we presented the first algorithms and regret bounds for SSP with LFA, and more generally, the first efficient LFA algorithm that uses stationary policies. Addressing the remaining statistical/computational gap (i.e., proving $\sqrt{K}$ regret in general) is an important open problem. Given the modular nature of the paper, one solution approach would be to combine our results with an improved oracle. 

\textit{Broader societal impact:} This paper is theoretical and has no immediate societal impact. Nevertheless, RL focuses on automated decision making, and training data can inject bias into these decisions. Care should be taken to minimize this bias when using our (or any other) algorithms in practice.

\section*{Acknowledgements}

This work was partially supported by ONR Grant N00014-19-1-2566, ARO Grant ARO W911NF-19-1-0379, NSF/USDA Grant AG 2018-67007-28379, and NSF Grants 1910112, 2019844, 1704970, and 1934986.

\bibliography{references}

\newpage \appendix \onecolumn \allowdisplaybreaks

\section{Section \ref{secExtensions} details} \label{appExtensions}

\subsection{Generalizing Theorem \ref{thmProp}} \label{appGenThmProp}

As discussed in Section \ref{secExtensions}, we can use the following OAFP oracle, which modifies Algorithm \ref{algIterate} by returning the $N_t$-th iterate if it reaches the $N_t$-th iteration. Let $N_t = t^{2 \gamma}$ for some absolute constant $\gamma \in (0,\frac{1}{4})$ and set $\kappa_t = 54 d t^{\frac{1}{4}} \sqrt{N_t}$ as in Theorem \ref{thmProp}. We show in Appendix \ref{appOracleProof} (see Remark \ref{remThmProbGen}) that with probability at least $1-\delta/2$, for any $t \geq (\log(3 t \chi) / (1-\rho) )^{\frac{1}{2\gamma}}$ that Algorithm \ref{algIterateTimeout} is called, it returns an OAFP within $t^{2 \gamma}$ iterations. 

Now suppose we run Algorithm \ref{algMain} with Algorithm \ref{algIterateTimeout} as the oracle. Let $\Gamma = \tilde{O} ( ( \frac{ \log \chi }{ (1-\rho ) } )^{\frac{1}{2\gamma}} )$. Then for the first $\Gamma$ time steps, Algorithm \ref{algIterateTimeout} need not return an OAFP (though it will terminate, so everything is well-defined) but does thereafter. Using Assumption \ref{assSSPreg}, we bound regret by $\Gamma$ for the first $\Gamma$ time steps, and by modifying the proof of Theorem \ref{thmRegret}, we can bound regret by $K^{\frac{3}{4}+\gamma}$ thereafter (in terms of $K$). Thus, regret will scale as $K^{\frac{3}{4} + \gamma }$ for this algorithm, with a second-order term $\Gamma$ in addition to the one from Theorem \ref{thmRegret}.

\IncMargin{1.2em}
\begin{algorithm}
\caption{Computing OAFPs with iteration limit} \label{algIterateTimeout}

Set $n=1$, compute $\hat{G}_t^n 0 = \hat{G}_t 0$ and $\hat{G}_t^{n-1} 0 = 0$

\While{$\| \hat{G}_t^n 0 - \hat{G}_t^{n-1} \|_{\Lambda_t} > \alpha_t$ \And $n \leq N_t$ }{

Set $n \leftarrow n+1$, compute $\hat{G}_t^n 0 = \hat{G}_t ( \hat{G}_t^{n-1} 0 )$

}

\Return{$w_t = \hat{G}_t^{n-1} 0$}
\end{algorithm}
\DecMargin{1.2em}

\subsection{Generalizing Theorem \ref{thmOrth}}

Let Assumption \ref{assSSPreg} hold and suppose $\phi(s,a)$, $\theta$, and $\mu(s')$ satisfy Assumption \ref{assSSPlin}. Denote by $\{ \varphi_i \}_{i=1}^{d'}$ the unique elements of $\{ \phi(s,a) \}_{(s,a) \in (\S \setminus \{ s_{goal} \}) \times \A}$. For any $d'' \in \{ d', d'+1, \ldots \}$, define
\begin{equation}
\Phi = \begin{bmatrix} \varphi_1 & \cdots & \varphi_{d'} \end{bmatrix} \in \R^{d \times d'} , \quad \Xi = \begin{bmatrix} \Phi & 0_{d \times (d''-d')}  \end{bmatrix} \in \R^{d \times d''} .
\end{equation}
Let $\Xi = R \tilde{\Phi}$ be an RQ decomposition, i.e., $R \in \R^{d \times d''}$ is upper triangular $\tilde{\Phi} \in \R^{d'' \times d''}$ is orthogonal. For $(s,a) \in (\S \setminus \{s_{goal} \}) \times \A$, let $\tilde{\phi}(s,a)$ be the $i(s,a)$-th column of $\tilde{\Phi}$, where $i(s,a) \in [d']$ is such that $\phi(s,a) = \varphi_{i(s,a)}$, and set $\tilde{\phi}(s_{goal},a) = 0\ \forall\ a \in \A$. We claim that $\tilde{\phi}(s,a)$, $R^\trans \theta$, and $R^\trans \mu(s')$ satisfy Assumption \ref{assSSPlin}. To prove \eqref{eqSSPlin}, we first observe that for any $(s,a,s') \in (\S \setminus \{s_{goal}\}) \times \A$,
\begin{equation}
\tilde{\phi}(s,a)^\trans R^\trans = e_{i(s,a)}^\trans \tilde{\Phi}^\trans  R^\trans  = e_{i(s,a)}^\trans \Xi^\trans  = e_{i(s,a)}^\trans \Phi^\trans =  \phi(s,a)^\trans ,
\end{equation}
so $\tilde{\phi}(s,a)^\trans R^\trans \theta = \phi(s,a)^\trans \theta = c(s,a)$ and $\tilde{\phi}(s,a)^\trans R^\trans \mu(s') = \phi(s,a)^\trans  \mu(s') = P(s'|s,a)$, as desired. The first inequality in \eqref{eqSSPnorm1} holds by construction. For the second inequality in \eqref{eqSSPnorm1}, note $\varphi_i^\trans \theta$ is the cost of some state-action pair and thus lies in $[0,1]$ by Assumption \ref{assSSPreg}. Combined with the fact that $\tilde{\Phi}$ is orthogonal,
\begin{equation}
\| R^\trans \theta \|_2^2 = \| \tilde{\Phi}^\trans R^\trans \theta \|_2^2 = \| \Xi^\trans \theta \|_2^2 = \sum_{i=1}^{d'} (\varphi_i^\trans \theta)^2  \leq d' \leq d'' .
\end{equation}
Similarly, for \eqref{eqSSPnorm2}, since $\varphi_i^\trans \mu(\cdot)$ is a probability distribution over $\S$, for any $h \in \R^\S$, we have
\begin{equation}
\left\| \sum_{s' \in \S} h(s') R^\trans \mu(s')  \right\|_2^2 = \sum_{i=1}^{d'} \left( \sum_{s' \in \S} h(s') \varphi_i^\trans \mu(s')  \right)^2 \leq d' \| h \|_\infty^2 \leq d'' \| h \|_\infty^2 .
\end{equation}
Algorithmically, this means that if $\Phi$ is known \textit{a priori}, we can set $d'' = d'$, compute $\tilde{\Phi}$, and use features $\tilde{\phi}(s,a) \in \R^{d'}$ instead of $\phi(s,a)$. Alternatively, if a nontrivial bound $d'' =  O(d')$ is known, we can iteratively compute $\tilde{\Phi}$ via Gram–Schmidt (computing the $i$-th column when we observe unique features for the $i$-th time), increasing the dimension to $d''$. In the respective cases, our results follow with $d$ replaced by $d'$ and $d''$, respectively.

\subsection{Zero/vanishing costs}

Finally, we extend our results to the case where only Assumption \ref{assSSPlin} and the following hold.
\begin{ass}[Weaker than Assumption \ref{assSSPreg}] \label{assSSPregW}
There exists a proper policy and $c(s,a) \in [0,1]\ \forall\ (s,a) \in \S \times \A$.
\end{ass}
Now suppose the SSP instance $(\S,\A,P,c,s_{goal})$ only satisfies Assumptions \ref{assSSPlin} and \ref{assSSPregW}. Let $c_\eta(s,a) = c(s,a) + \eta$ be the perturbed cost discussed in Section \ref{secExtensions}. Then the instance $(\S,\A,P,c_\eta,s_{goal})$ satisfies Assumption \ref{assSSPreg}, with $c_{min} = \eta$ (up to a small constant, since $c_\eta(s,a)$ may be as large as $1+\eta$). Also define $\theta_\eta = \theta + \eta \sum_{s \in \S} \mu(s)$. Then since Assumption \ref{assSSPlin} holds for the original instance, we have
\begin{equation}
c_\eta(s,a) = c(s,a) + \eta = c(s,a) +  \eta \sum_{s' \in \S} P(s'|s,a) = \phi(s,a)^\trans \theta + \eta \sum_{s' \in \S} \phi(s,a)^\trans \mu(s') = \phi(s,a)^\trans \theta_\eta  ,
\end{equation}
so it also holds for the perturbed instance (again, up to a constant, since we can only assert $\| \theta_\eta \|_2 \leq \sqrt{d}(1+\eta)$). Thus, if we run Algorithm \ref{algMain} on the original instance but replace $c(s,a)$ with $c_\eta(s,a)$ in the definition of $\hat{G}_t$, and if $J_\eta^\star$ is the optimal cost-to-go function on the perturbed instance, Theorem \ref{thmRegret} ensures that
\begin{align}
\sum_{t=1}^T c_\eta(s_t,a_t) - \sum_{k=1}^K J_\eta^\star(s_1^k) = \tilde{O} \left( \left( B_\star^{\frac{3}{2}+\lambda} + B_\star^{\frac{1}{2}+\lambda} \right) d^{\frac{1}{2}} \Psi ( K / \eta )^{\frac{1}{2}+\lambda} + (B_\star+1)^{\frac{2}{1-2\lambda}} d^{\frac{1}{1-2\lambda} } \Psi^{\frac{2}{1-2\lambda}} \eta^{-\frac{1+2\lambda}{1-2\lambda}} \right) .
\end{align}
Also, since $J_\eta^\star$ is optimal on the perturbed instance and both instances have the same transition kernel, we have
\begin{equation}
J_\eta^\star(s) - J^\star(s) \leq J_\eta^{\pi_{\text{prop}}^\star}(s) - J^{\pi_{\text{prop}}^\star}(s) \leq \eta T_\star\ \forall\ s \in \S ,
\end{equation}
where (we recall from Section \ref{secExtensions}) $\pi_{\text{prop}}^\star$ is the optimal proper policy and $T_\star$ is the maximum expected time it takes $\pi_{\text{prop}}^\star$ to reach the goal state from any starting state (since this policy is proper, $T_\star < \infty$). Therefore, since $c(s,a) \leq c_\eta(s,a)$ by definition, we can bound regret (defined with respect to $\pi_{\text{prop}}^\star$, as in Section \ref{secExtensions}) by
\begin{align}
R(K) & \leq  \sum_{t=1}^T c_\eta(s_t,a_t) - \sum_{k=1}^K J_\eta^{\pi_{\text{prop}}^\star}(s_1^k) +  \sum_{k=1}^K \left( J_\eta^{\pi_{\text{prop}}^\star}(s_1^k)  -  J^{\pi_{\text{prop}}^\star}(s_1^k) \right) \\
& = \tilde{O} \left( \left( B_\star^{\frac{3}{2}+\lambda} + B_\star^{\frac{1}{2}+\lambda} \right) d^{\frac{1}{2}} \Psi ( K / \eta )^{\frac{1}{2}+\lambda} + \eta T_\star K + (B_\star+1)^{\frac{2}{1-2\lambda}} d^{\frac{1}{1-2\lambda} } \Psi^{\frac{2}{1-2\lambda}} \eta^{-\frac{1+2\lambda}{1-2\lambda}}  \right) . 
\end{align}
Choosing $\eta$ to decay as $K^{ (2\lambda-1)/(2\lambda+3)}$ ensures the first two terms scale as $K^{(4\lambda+2)/(2\lambda+3)}$ and the third as $K^{(2\lambda+1)/(2\lambda+3)}$. Since $\frac{2\lambda+1}{2\lambda+3} \leq \frac{4\lambda+2}{2\lambda+3} < 1$ for any $\lambda \in [0,\frac{1}{2})$, we thus have sublinear regret.

\section{Proof preliminaries} \label{appProofGeneral}

In this appendix, we collect some notation and results used in the proofs of multiple theorems.

\subsection{Additional notation}

We write $\E_t$ for expectation conditioned on the first $t-1$ state-action-state triples and the $t$-th state action pair, $\E_t [ \cdot ] = \E [ \cdot | \{ ( s_\tau, a_\tau , s_\tau' \}_{\tau=1}^{t-1} \cup \{ s_t,a_t \} ]$. We let $\E_{s_t'}$ denote expectation with respect to $s_t'$. Hence, for $h : \S \rightarrow \R$,
\begin{equation} \label{eqEtEstFixedFunc}
\E_t [ h(s_t') ] = \E_{s_t'} [ h(s_t') ] = \sum_{ s \in \S } h(s) P(s|s_t,a_t) .
\end{equation}
However, we emphasize that since $g_t$ in Definition \ref{defnOAFP} is a random function of the first $t$ state-action pairs, if $\tau < t$, we may have $\E_\tau [ g_t ( s_\tau',w) ] \neq \sum_{ s \in \S } g_t(s,w) P(s|s_\tau,a_\tau)$ for some $w \in \R^d$. On the other hand, $\E_{s_\tau'} [ g_t(s_\tau',w) ] = \sum_{ s \in \S } g_t(s,w) P(s|s_\tau,a_\tau)$ does hold for any $w \in \R^d$.

As discussed in Section \ref{secOAFP}, we also define the (random) operators $U_t , E_t : \R^d \rightarrow \R^d$ by
\begin{gather}
U_t w = \theta + \sum_{s \in \S} g_t(s,w) \mu(s)\ \forall\ w \in \R^d , \quad E_t = \hat{G}_t - U_t .
\end{gather}
Here $U_t$ can be roughly viewed as the expected value of $\hat{G}_t$, so $E_t$ is the error between $\hat{G}_t$ and its mean. Note, however, that since $g_t$ is a random function, $U_t w$ is a random vector (even for fixed $w \in \R^d$). We also note the following identity, which is an immediate consequence of Assumption \ref{assSSPlin} and is frequently used:
\begin{equation}
\phi(s,a)^\trans U_t w = \phi(s,a)^\trans \theta + \sum_{s' \in \S} g_t(s',w) \phi(s,a)^\trans \mu(s') = c(s,a) + \sum_{s' \in \S} g_t(s',w) P(s'|s,a) .
\end{equation}
As in Section \ref{secOracle}, we iterate these operators in the usual way, e.g., $\hat{G}_t^n 0 = \hat{G}_t ( \hat{G}_t^{n-1} 0 )$ for $n \in \N$ with $\hat{G}_t^0 0 = 0$.

Finally, for any $b > 0$, we define the clipping function $\Pi_{[0,b]} : \R \rightarrow [0,b]$ by
\begin{equation} \label{eqClip}
\Pi_{[0,b]} ( x ) = \min \{ \max \{ x , 0 \} , b \} = \max \{ \min \{ x , b \} , 0 \} .
\end{equation}
Note that with this notation, we can more compactly write $g_t(\cdot,\cdot) = \Pi_{[0,B_t]}( f_t(\cdot,\cdot) )$ in Definition \ref{defnOAFP}.

\subsection{Simple results}

\begin{clm}[Eigenvalues and norms] \label{clmEigAndNorm}
If Assumption \ref{assSSPlin} holds, then the eigenvalues of $\Lambda_t$ lie in $[1,t+1]$, and
\begin{equation}
\| w \|_{\Lambda_t^{-1}} \leq \| w \|_2 \leq \| w \|_{\Lambda_t} \leq \sqrt{t+1} \| w \|_2 \leq \sqrt{(t+1)d} \| w \|_\infty\ \forall\ w \in \R^d .
\end{equation}
\end{clm}
\begin{proof}
Let $\{ \lambda_i \}_{i=1}^d$ and $\{ q_i \}_{i=1}^d$ be the eigenvalues and (unit) eigenvectors of $\Lambda_t$. Then
\begin{equation} 
\lambda_i = \lambda_i q_i^\trans q_i = q_i^\trans  \Lambda_t q_i  =  q_i^\trans  q_i + \sum_{\tau=1}^t ( \phi(s_\tau,a_\tau)^{\trans} q_i )^2 = 1 + \sum_{\tau=1}^t ( \phi(s_\tau,a_\tau)^{\trans} q_i )^2 .
\end{equation}
The eigenvalue bounds follow, since $0 \leq ( \phi(s_\tau,a_\tau)^{\trans} q_i )^2 \leq \| \phi(s_\tau,a_\tau) \|_2 \| q_i \|_2 \leq 1$ by Cauchy-Schwarz. For the norm equivalences, we first use the eigenvalue bounds to write
\begin{equation}
\| w \|_{\Lambda_t^{-1}}^2 = \sum_{i=1}^d \frac{ ( q_i^\trans w )^2 }{ \lambda_i } \leq \sum_{i=1}^d ( q_i^\trans w )^2 \leq \| w \|_{\Lambda_t}^2 = \sum_{i=1}^d \lambda_i (  q_i^\trans w )^2 \leq (t+1) \sum_{i=1}^d ( q_i^\trans w )^2 .
\end{equation}
Since $\sum_{i=1}^d ( q_i^\trans w )^2 =  \| w \|_2^2$ by orthogonality, this proves the first three norm bounds. The fourth is standard.
\end{proof}

\begin{clm}[$\Pi_{[0,b]}$ properties] \label{clmClipNonex}
For any $b > 0$ and $x,y \in \R$, $\Pi_{[0,b]}(x) \leq \max \{ x , 0 \}$ and $|\Pi_{[0,b]}(x) - \Pi_{[0,b]}(y) | \leq |x-y|$.
\end{clm}
\begin{proof}
The first bound holds by \eqref{eqClip}. For the second, assume without loss of generality that $x \geq y$. By monotonicity, it suffices to show $\Pi_{[0,b]}(x) - \Pi_{[0,b]}(y) \leq x-y$. If $x < 0$ or $y > b$, then $\Pi_{[0,b]}(x) = \Pi_{[0,b]}(y)$, so this is immediate. Otherwise, \eqref{eqClip} implies $\Pi_{[0,b]}(x) - \Pi_{[0,b]}(y) \leq \max \{ x , 0 \} - \min \{ y , b \} = x- y$.
\end{proof}

\begin{clm}[$g_t$ bounds] \label{clmGtNonex}
If Assumption \ref{assSSPlin} holds, then for any $t \in [T]$, $s \in \S$, and $w_1, w_2 \in \R^d$,
\begin{equation}
| g_t(s,w_1) - g_t(s,w_2) | \leq | f_t(s,w_1) - f_t(s,w_2) | \leq \max_{a \in \A} | \phi(s,a)^\trans (w_1-w_2) | \leq \| w_1 - w_2 \|_2 \leq \sqrt{d} \| w_1 - w_2 \|_\infty .
\end{equation}
\end{clm}
\begin{proof}
The first bound follows from Claim \ref{clmClipNonex}. For the second, let $\bar{a} \in \A$ be any action attaining the minimum in the definition of $f_t(s,w_1)$, i.e., $\bar{a} \in \argmin_{a \in \A} ( \phi(s,a)^\trans w_1 - \alpha_t \| \phi(s,a) \|_{\Lambda_t^{-1}} )$. Then
\begin{align}
f_t(s,w_1) - f_t(s,w_2) & \geq f_t(s,w_1) - \left( \phi(s,\bar{a})^\trans w_2 - \alpha_t \| \phi(s,\bar{a}) \|_{\Lambda_t^{-1}}  \right) \\
& = \phi(s,\bar{a})^\trans(w_1-w_2) \geq - \max_{a \in \A} | \phi(s,a)^\trans (w_1-w_2) | .
\end{align}
By symmetry, we also have $f_t(s,w_1) - f_t(s,w_2) \leq \max_{a \in \A} | \phi(s,a)^\trans (w-w') |$; the second bound follows. The third follows from Cauchy-Schwarz and the fourth from a standard norm equivalence.
\end{proof}

\begin{clm}[Operator bounds] \label{clmOpBounds}
If Assumptions \ref{assSSPreg} and \ref{assSSPlin} hold, then for any $t \in [T]$, $w \in \R^d$, and $(s,a) \in \S \times \A$,
\begin{gather}
\| \hat{G}_t w \|_\infty \leq \sqrt{td} \left( 1 + \max_{s \in \S} g_t(s,w ) \right) , \quad \| \hat{G}_t 0 \|_{\Lambda_t} \leq \sqrt{t+1} \| \hat{G}_t 0 \|_2 \leq 2\sqrt{ (t+1) d } , \\
\| U_t w \|_2 \leq \sqrt{d} \left( 1 + \max_{s \in \S} g_t(s,w) \right) , \quad \phi(s,a)^\trans U_t w \in [0,B_t+1] .
\end{gather}
\end{clm}
\begin{proof}
First observe that by a standard norm equivalence and Claim \ref{clmEigAndNorm}, for any $w \in \R^d$, we have
\begin{equation} \label{eqLamInvToL2}
\| \Lambda_t^{-1} w \|_\infty \leq \| \Lambda_t^{-1} w \|_2 = \| \Lambda_t^{-1/2} w \|_{\Lambda_t^{-1}} \leq \| \Lambda_t^{-1/2} w \|_2 = \| w \|_{\Lambda_t^{-1}} \leq \| w \|_2 .
\end{equation}
Combined with Cauchy-Schwarz and \cite[Lemma D.1]{jin2020provably}, we obtain
\begin{equation} 
\sum_{\tau=1}^t \| \Lambda_t^{-1} \phi(s_\tau,a_\tau) \|_\infty \leq \sum_{\tau=1}^t \| \phi(s_\tau,a_\tau) \|_{\Lambda_t^{-1}} \leq \sqrt{ t \sum_{\tau=1}^t \| \phi(s_\tau,a_\tau) \|_{\Lambda_t^{-1}}^2 } \leq \sqrt{td} ,
\end{equation}
so the first $\hat{G}_t$ bound follows from the triangle inequality. Next, because $g_t(s,0) = 0\ \forall\ s \in \S$, we have
\begin{equation}
\hat{G}_t 0 = \Lambda_t^{-1} \sum_{\tau=1}^t \phi(s_\tau,a_\tau) c(s_\tau,a_\tau) = \Lambda_t^{-1} \sum_{\tau=1}^t \phi(s_\tau,a_\tau)\phi(s_\tau,a_\tau)^\trans \theta = \Lambda_t^{-1} ( \Lambda_t - I ) \theta = ( I - \Lambda_t^{-1} ) \theta .
\end{equation}
Therefore, by Claim \ref{clmEigAndNorm} and \eqref{eqLamInvToL2}, we obtain
\begin{equation}
\| \hat{G}_t 0 \|_{\Lambda_t} \leq \sqrt{t+1} \| \hat{G}_t 0 \|_2 \leq \sqrt{t+1} ( \| \theta \|_2 + \| \Lambda_t^{-1} \theta \|_2 ) \leq 2 \sqrt{t+1} \| \theta \|_2 \leq 2 \sqrt{(t+1)d} .
\end{equation}
Finally, the $U_t$ bounds hold by assumption.
\end{proof}

\begin{clm}[$B_\star$ estimate] \label{clmBtBound}
If Assumption \ref{assSSPreg} holds, then $\sup_{t \geq 0} B_t \leq 2 B_\star$.
\end{clm}
\begin{proof}
Since $B_0 = c_{min} \leq B_\star$ by assumption, it suffices to show $B_\tau \leq 2 B_\star\ \forall\ \tau \in \N$. Suppose instead that $B_\tau > 2 B_\star$ for some such $\tau$. Let $t = \min \{ \tau \in \N : B_\tau > 2 B_\star \}$ be the first time it occurs. Then by Algorithm \ref{algMain}, we have $B_\star < B_t / 2 = B_{t-1} < f_{M_{l-1}} ( s_t' , w_{M_{l-1}} )$ for some $l \in \N$. If $l = 1$, this contradicts the fact that $w_0 = 0$; otherwise, it contradicts the fact that $w_{M_l}$ is an OAFP (see Definition \ref{defnOAFP}).
\end{proof}

\subsection{Operator concentration} \label{appOperatorConc}

Define the random variables $W_t$ and $\varepsilon_t$, and the event $\mathcal{E}$, by
\begin{equation}
W_t = \alpha_t + \sqrt{td} (B_t+1) , \quad \varepsilon_t = 5 (B_t+1) d \sqrt{ \log(t \alpha_t / \delta )} , \quad \mathcal{E} = \left\{ \sup_{ w \in [-W_t,+W_t]^d } \| E_t w \|_{\Lambda_t} \leq \varepsilon_t\ \forall\ t \in [T] \right\} .
\end{equation}

The following is our main concentration result; the proof is lengthy so is deferred to Appendix \ref{appProofMainEkTail}.
\begin{lem}[Error operator tail bound] \label{lemMainEkTail2}
If Assumptions \ref{assSSPreg} and \ref{assSSPlin} hold and $\min_{t \in \N} \kappa_t \geq 9d$, then $\P(\mathcal{E}) \geq 1-\delta/2$.
\end{lem}

The error bound $\varepsilon_t$ in the lemma is related to the exploration parameter $\alpha_t$ in the following manner.
\begin{clm}[Lower bound on $\alpha_t$] \label{clmParameters}
For any $t \in \N$, if $\kappa_t \geq 9 m d$ for some $m \geq 1$, then $\alpha_t \geq \max \{ m \varepsilon_t , (B_t+1) \kappa_t \}$.
\end{clm} 
\begin{proof}
For the first bound, since $\log x \leq x\ \forall\ x \in \R$, we have
\begin{equation}
\alpha_t = (B_t+1) \kappa_t \sqrt{ \log( t (B_t+1) \kappa_t / \delta)} \leq ( B_t+1 )^{3/2} \kappa_t^{3/2} t^{1/2} / \delta^{1/2} .
\end{equation}
Combined with $5 \sqrt{3/2} \leq 9$ and the assumption $\kappa_t \geq 9m d$, we obtain
\begin{align}
m \varepsilon_t & =  5 m d \sqrt{  \log(t  \alpha_t / \delta )} (B_t+1) \leq  5 \sqrt{3/2} m d \sqrt{ \log ( t (B_t+1) \kappa_t / \delta ) } (B_t+1) \\
& \leq 9 m d \sqrt{ \log ( t (B_t+1) \kappa_t / \delta ) } (B_t+1) \leq   \kappa_t \sqrt{ \log ( t (B_t+1) \kappa_t / \delta ) } (B_t+1) = \alpha_t . 
\end{align}
For the second bound, simply note $\log ( t (B_t+1) \kappa_t / \delta ) \geq \log ( 9md ) \geq 1$ and use the definition of $\alpha_t$.
\end{proof}

As corollaries, we have the following special cases of operator concentration.
\begin{cor}[Error at OAFP] \label{corConcOafpNew}
For any $t \in \N$, if $\kappa_t \geq 9d$ and $w_t$ is an OAFP, then on the event $\mathcal{E}$,
\begin{equation}
| \phi(s,a)^\trans ( \hat{G}_t w_t - U_t w_t ) | \leq \alpha_t \| \phi(s,a) \|_{\Lambda_t^{-1}}\ \forall\ (s,a) \in \S \times \A .
\end{equation}
\end{cor}
\begin{proof}
By Claims \ref{clmEigAndNorm} and \ref{clmOpBounds} and Definition \ref{defnOAFP}, we have
\begin{align}
\| w_t \|_\infty \leq \| w_t - \hat{G}_t w_t \|_2 + \| \hat{G}_t w_t \|_\infty \leq \| w_t - \hat{G}_t w_t \|_{\Lambda_t} + \sqrt{td}(B_t+1) \leq \alpha_t + \sqrt{td}(B_t+1) = W_t .
\end{align}
Hence, using $\kappa_t \geq 9d$ and Claim \ref{clmParameters}, we conclude $\| E_t w_t \|_{\Lambda_t} \leq \alpha_t$ on $\mathcal{E}$. Combined with Cauchy-Schwarz,
\begin{align}
& | \phi(s,a)^\trans ( \hat{G}_t w_t - U_t w_t ) | = | \phi(s,a)^\trans E_t w_t | \leq \| \phi(s,a) \|_{\Lambda_t^{-1}} \| E_t \|_{\Lambda_t} \leq \| \phi(s,a) \|_{\Lambda_t^{-1}} \alpha_t . \qedhere
\end{align}
\end{proof}

\begin{cor}[Error at $\hat{G}_t$ iterates] \label{corConcGhatNew}
For any $t \in \N$, if $\kappa_t \geq 9md$ for some $m \geq 1$, then on the event $\mathcal{E}$, for any $n \in \N$ and $(s,a) \in \S \times \A$,
\begin{equation}
| \phi(s,a)^\trans ( \hat{G}_t^n 0 - U_t ( \hat{G}_t^{n-1} 0 ) ) | \leq \| \hat{G}_t^n 0 - U_t ( \hat{G}_t^{n-1} 0 ) \|_{\Lambda_t} \| \phi(s,a) \|_{\Lambda_t^{-1}} \leq \varepsilon_t \| \phi(s,a) \|_{\Lambda_t^{-1}} \leq \alpha_t \| \phi(s,a) \|_{\Lambda_t^{-1}} / m .
\end{equation}
\end{cor}
\begin{proof}
First note $\hat{G}_t^n 0 - U_t ( \hat{G}_t^{n-1} 0 ) = E_t ( \hat{G}_t^{n-1} 0 )$. For $n \geq 2$, $\| \hat{G}_t^{n-1} 0 \|_\infty = \| \hat{G}_t ( \hat{G}_t^{n-2} 0 ) \|_\infty \leq \sqrt{td}(B_t+1) \leq W_t$ by Claim \ref{clmOpBounds}, and for $n = 1$, $\|\hat{G}_t^{n-1} 0\|_\infty = 0$. Hence, for any $n \in \N$, we have $\| E_t ( \hat{G}_t^{n-1} 0 ) \|_{\Lambda_t} \leq \varepsilon_t$ on $\mathcal{E}$ by definition. The desired bounds follow from Cauchy-Schwarz and Claim \ref{clmParameters} similar to Corollary \ref{corConcOafpNew}.
\end{proof}

\subsection{Operator convergence}

We next show the operator $U_t$ converges in a certain sense. We begin by proving some basic properties.
\begin{clm}[$U_t$ properties] \label{clmUtProperties}
If Assumptions \ref{assSSPreg} and \ref{assSSPlin} hold, then for any $t \in [T]$,
\begin{gather}
0 \leq \cdots \leq \phi(s,a)^\trans U_t^{n-1} 0 \leq \phi(s,a)^\trans U_t^n 0 \leq \cdots \leq B_t+1\ \forall\ (s,a) \in \S \times \A , \label{eqUtMonotone} \\
\max_{(s,a) \in \S \times \A} | \phi(s,a)^\trans ( U_t^{n+1} 0 - U_t^n 0 ) | \leq \max_{(s,a) \in \S \times \A}  | \phi(s,a)^\trans ( U_t^n 0 - U_t^{n-1} 0 ) |\ \forall\ n \in \N . \label{eqUtNonexpansive}
\end{gather}
\end{clm}
\begin{proof}
By Claim \ref{clmOpBounds}, we already know $\phi(s,a)^\trans U_t^n 0 \in [0,B_t+1]$. To complete the proof of \eqref{eqUtMonotone}, we show by induction on $n$ that $\min_{(s,a) \in \S \times \A} \phi(s,a)^\trans ( U_t^n 0 - U_t^{n-1} 0 ) \geq 0$. For $n=1$, we simply have 
\begin{equation}
\min_{(s,a) \in \S \times \A} \phi(s,a)^\trans ( U_t^n 0 - U_t^{n-1} 0 ) = \min_{(s,a) \in \S \times \A} \phi(s,a)^\trans U_t 0 = \min_{(s,a) \in \S \times \A} c(s,a) \geq 0 .
\end{equation}
Now assuming $\min_{(s,a) \in \S \times \A} \phi(s,a)^\trans ( U_t^n 0 - U_t^{n-1} 0 ) \geq 0$, we have $\min_{s' \in \S} (g_t(s',U_t^n 0) - g_t(s',U_t^{n-1} 0) ) \geq 0$, so
\begin{equation}
\min_{(s,a) \in \S \times \A} \phi(s,a)^\trans ( U_t^{n+1} 0 - U_t^n 0 ) = \min_{(s,a) \in \S \times \A} \sum_{s' \in \S} ( g_t(s',U_t^n 0) - g_t(s',U_t^{n-1} 0) ) P(s'|s,a) \geq 0 . 
\end{equation}
Finally, \eqref{eqUtNonexpansive} follows from Claim \ref{clmGtNonex}:
\begin{align}
\max_{(s,a) \in \S \times \A} | \phi(s,a)^\trans ( U_t^{n+1} 0 - U_t^n 0 ) | & \leq \max_{(s,a) \in \S \times \A} \sum_{s' \in \S} | g_t(s',U_t^n 0) - g_t(s',U_t^{n-1} 0) | P(s'|s,a) \\
& \leq \max_{(s,a) \in \S \times \A} | \phi(s,a)^\trans ( U_t^n 0 - U_t^{n-1} 0 ) | . \qedhere
\end{align}
\end{proof}

The proof of Claim \ref{clmUtProperties} shows that $U_t$ is nonexpansive in the induced $\ell_\infty$ norm $\| \cdot \| = \| \Phi^\trans \cdot \|_\infty$, where $\Phi$ is the matrix with columns $\{\phi(s,a)\}_{(s,a) \in \S \times \A}$. Combined with the claim's monotonicity result, this is enough to show that $U_t$ converges at rate $1/n$. However, because the induced norm lifts to $|\S \times \A|$-dimensional space, a naive convergence proof yields a constant that scales with $|\S \times \A|$. The next claim will allow us to avoid this.
\begin{clm}[A linear algebra result] \label{clmUpsilon}
Let $\mathcal{Z} = \S \times \A$, $\Upsilon \in \R^{d \times \mathcal{Z}}$, and $r = rank(\Upsilon)$. For any $\mathcal{Z}' \subset \mathcal{Z}$, denote by $\Upsilon(\mathcal{Z}')$ the submatrix of $\Upsilon$ with columns $\mathcal{Z}'$, with $\Upsilon(z) = \Upsilon(\{z\})$ for any $z \in \mathcal{Z}$ for simplicity. Then there exists $\mathcal{Z}' \subset \mathcal{Z}$ such that $|\mathcal{Z}'| = r$ and $\|\Upsilon^\trans x\|_\infty \leq r \|\Upsilon(\mathcal{Z}')^\trans x\|_\infty$ for any $x \in \R^d$.
\end{clm}
\begin{proof}
We first assume $r = d$. Then we can find $\mathcal{Z}'' \subset \mathcal{Z}$ such that $|\mathcal{Z}''| = d$ and $det(\Upsilon(\mathcal{Z}'')) \neq 0$. Let $\mathcal{Z}'$ be whichever such $\mathcal{Z}''$ maximizes $|det(\Upsilon(\mathcal{Z}''))|$. Set $H = \Upsilon^\trans ( \Upsilon(\mathcal{Z}')^\trans )^{-1}$, which is well-defined by choice of $\mathcal{Z}'$. Then letting $\| H \|_\infty = \max_{ z \in \mathcal{Z} } \sum_{z' \in \mathcal{Z}'} |H(z,z')|$ denote the operator norm, for any $x \in \R^d$, we obtain
\begin{equation}
\| \Upsilon^\trans x \|_\infty  = \| H \Upsilon(\mathcal{Z}')^\trans x \|_\infty \leq \| H \|_\infty \| \Upsilon(\mathcal{Z}')^\trans x \|_\infty \leq d \max_{z \in \mathcal{Z},z' \in \mathcal{Z}'} | H(z,z') | \| \Upsilon(\mathcal{Z}')^\trans x \|_\infty .
\end{equation}
Thus, it suffices to show $|H(z,z')| \leq 1$. Toward this end, for any $y \in \R^d$, let $\Upsilon(\mathcal{Z}',y)$ be the matrix that results from replacing the $z'$-th column of $\Upsilon(\mathcal{Z}')$ with $y$. Then since $\Upsilon(z) = \sum_{z'' \in \mathcal{Z}'} H(z,z'') \Upsilon(z'')$, we have
\begin{equation}
\Upsilon ( \mathcal{Z}' \cup \{ z \} \setminus \{ z' \} ) = \Upsilon(\mathcal{Z}', \Upsilon(z)) = \Upsilon \left( \mathcal{Z}' , \sum_{z'' \in \mathcal{Z}'} H(z,z'') \Upsilon(z'') \right) .
\end{equation}
Next, observe that $\Upsilon ( \mathcal{Z}' , \Upsilon(z'') )$ is rank deficient when $z'' \neq z'$; otherwise, when $z'' = z'$, we have $\Upsilon ( \mathcal{Z}' , \Upsilon(z'') ) = \Upsilon ( \mathcal{Z}' )$. Hence, by multilinearity of the determinant, we obtain
\begin{equation}
det \left( \Upsilon \left( \mathcal{Z}' , \sum_{z'' \in \mathcal{Z}'} H(z,z'') \Upsilon(z'') \right) \right) = \sum_{z'' \in \mathcal{Z}'} H(z,z'') det ( \Upsilon ( \mathcal{Z}' , \Upsilon(z'') ) ) = H(z,z')  det ( \Upsilon ( \mathcal{Z}'  ) ) .
\end{equation}
Combining the previous two identities with the definition of $\mathcal{Z}'$ yields the desired bound:
\begin{equation}
|H(z,z')| = | det (\Upsilon ( \mathcal{Z}' \cup \{ z \} \setminus \{ z' \} )) | / | det(\Upsilon(\mathcal{Z}')) | \leq 1 .
\end{equation}

If instead $r < d$, let $\Upsilon = U \Sigma V^\trans$ be the SVD. Then $\Upsilon^\trans U = V \Sigma^\trans = [ \tilde{\Upsilon}^\trans\ 0  ]$, where $\tilde{\Upsilon} \in \R^{r \times \mathcal{Z}}$ has full rank. Hence, by the previous case, we can find $\mathcal{Z}' \subset \mathcal{Z}$ such that $|\mathcal{Z}'| = r$ and $\| \tilde{\Upsilon}^\trans \tilde{x} \|_\infty \leq r \| \tilde{\Upsilon}(\mathcal{Z}')^\trans \tilde{x} \|_\infty$ for any $\tilde{x} \in \R^r$. Let $U = [ U_1\ U_2 ]$ with $U_1 \in \R^{d \times r}$. Then for any $x \in \R^d$, if we let $\tilde{x} = U_1^\trans x \in \R^r$, we obtain
\begin{equation} \label{eqUpsXtoTilde}
\Upsilon^\trans x = \Upsilon^\trans U U^\trans x = \begin{bmatrix} \tilde{\Upsilon}^\trans & 0 \end{bmatrix} \begin{bmatrix} \tilde{x} \\ U_2^\trans x \end{bmatrix}  = \tilde{\Upsilon}^\trans \tilde{x} .
\end{equation}
Therefore, by the choice of $\mathcal{Z}'$, we have
\begin{align}
\| \Upsilon^\trans x \|_\infty & = \| \tilde{\Upsilon}^\trans \tilde{x} \|_\infty \leq r \| \tilde{\Upsilon}(\mathcal{Z}')^\trans \tilde{x}\|_\infty =r \| \Upsilon(\mathcal{Z}')^\trans x  \|_\infty . \qedhere
\end{align}
\end{proof}

We can now show that $U_t$ converges at rate $1/n$ with a constant depending only $B_t+1$ and $d$.
\begin{lem}[$U_t$ convergence]\label{lemUtConvergence}
If Assumptions \ref{assSSPreg} and \ref{assSSPlin} hold, then for any $t \in [T]$ and $n \in \N$,
\begin{equation}
\max_{(s,a) \in \S \times \A} | \phi(s,a)^\trans ( U_t^n 0 - U_t^{n-1} 0 ) | \leq \frac{  (B_t+1) d^2 }{n} .
\end{equation}
\end{lem}
\begin{proof}
Suppose instead that for some $n \in \N$, we have
\begin{equation} 
\max_{(s,a) \in \S \times \A} | \phi(s,a)^\trans ( U_t^n 0 - U_t^{n-1} 0 ) | > \frac{  (B_t+1) d^2 }{n} .
\end{equation}
Then combining Claims \ref{clmUtProperties} and \ref{clmUpsilon}, we can find $\mathcal{Z}' \subset \S \times \A$ such that, for any $m \in [n]$,
\begin{align} 
\frac{  (B_t+1) | \mathcal{Z}'| }{n} & \leq \frac{ (B_t+1) d^2 }{ n d } < \frac{1}{d} \max_{(s,a) \in \S \times \A} | \phi(s,a)^\trans ( U_t^m 0 - U_t^{m-1} 0 ) | \leq \max_{(s,a) \in \mathcal{Z}'} | \phi(s,a)^\trans ( U_t^m 0 - U_t^{m-1} 0 ) | .
\end{align}
Thus, for each $m \in [n]$, we can find $z_m \in \mathcal{Z}'$ with $| \phi(z_m)^\trans ( U_t^m 0 - U_t^{m-1} 0 ) | >  (B_t+1) |\mathcal{Z}'| / n$. But by Claim \ref{clmUtProperties},
\begin{align}
n = \sum_{z \in \mathcal{Z}'} \sum_{m=1}^n \ind ( z_m = z )  < \frac{n}{ (B_t+1) |\mathcal{Z}'| } \sum_{z \in \mathcal{Z}'} \sum_{m=1}^n \phi(z)^\trans ( U_t^m 0 - U_t^{m-1} 0 ) = \frac{n}{ (B_t+1) |\mathcal{Z}'| } \sum_{z \in \mathcal{Z}'}  \phi(z)^\trans U_t^n 0 \leq n ,
\end{align}
which is a contradiction.
\end{proof}

\section{Proof of Theorem \ref{thmRegret}} \label{appRegretProof}

In this appendix, we prove Theorem \ref{thmRegret} in two steps. First, in Appendix \ref{appExistenceProof}, we show that OAFPs exist on the event $\mathcal{E}$. Second, in Appendix \ref{appRegretBound}, we prove the regret bound on the intersection of $\mathcal{E}$ and an event $\mathcal{F}$ defined in Lemma \ref{lemMDS} that occurs with probability at least $1-\delta/2$. Thus, by the union bound and Lemma \ref{lemMainEkTail2}, OAFPs exist and the regret bound holds with probability at least $1-\delta$, which establishes the theorem.

\subsection{Existence of OAFPs} \label{appExistenceProof}

We begin with optimism lemma for the operator $U_t$.
\begin{lem}[$U_t$ optimism] \label{lemUtOptimism}
Under the assumptions of Theorem \ref{thmRegret}, for any $t \in [T]$, $n \in \N$, and $(s,a) \in \S \times \A$, 
\begin{equation}
\phi(s,a)^\trans U_t^{n-1} 0 \leq Q^\star(s,a) .
\end{equation}
\end{lem}
\begin{proof}
We fix $t$ and use induction on $n$. For $n = 1$, we simply have $\phi(s,a)^\trans U_t^{n-1} 0 = 0 \leq Q^\star(s,a)$. Assuming true for $n \in \N$, the Bellman optimality equation \eqref{eqBellman} implies
\begin{equation}
f_t(s,U_t^{n-1}0) = \min_{a \in \A} \left( \phi(s,a)^\trans U_t^{n-1} 0 - \alpha_t \| \phi(s,a) \|_{\Lambda_t^{-1}}  \right) \leq \min_{a \in \A} Q^\star(s,a) = J^\star(s)\ \forall\ s \in \S .
\end{equation}
Hence, by Claim \ref{clmClipNonex}, $g_t(s,U_t^{n-1}0) \leq \max \{ J^\star(s) , 0 \} = J^\star(s)$. Again using Bellman optimality, we thus obtain
\begin{align}
 \phi(s,a)^\trans U_t^n 0 & = c(s,a) + \sum_{s' \in \S}  g_t(s',U_t^{n-1} 0 )  P(s'|s,a) \leq c(s,a) + \sum_{s' \in \S} J^\star(s') P(s'|s,a) = Q^\star(s,a) . \qedhere
\end{align}
\end{proof}

We now establish existence of OAFPs. First note that by Claim \ref{clmGtNonex} and Lemma \ref{lemUtConvergence}, for any norm $\| \cdot \|$, we have
\begin{align}
\| U_t^{n+1} 0 - U_t^n 0 \| & \leq \sqrt{d} \max_{s \in \S} | g_t(s,U_t^n0) - g_t(s,U_t^{n-1}0) | \sqrt{d} \leq \max_{(s,a) \in \S \times  \A} | \phi(s,a)^\trans ( U_t^n 0 - U_t^{n-1} 0 ) | \xrightarrow[n \rightarrow \infty]{} 0 .
\end{align}
Hence, $w_t^\star = \lim_{n \rightarrow \infty} U_t^n 0$ exists. By continuity, it is a fixed point:
\begin{equation}
w_t^\star = \lim_{n \rightarrow \infty} U_t^{n+1} 0 = \lim_{n \rightarrow \infty} U_t ( U_t^n 0 ) = U_t \left( \lim_{n \rightarrow \infty} U_t^n 0 \right) = U_t ( w_t^\star ) .
\end{equation}
Thus, by a standard norm equivalence and Claim \ref{clmOpBounds}, we have
\begin{equation}
\| w_t^\star \|_\infty \leq \| w_t^\star \|_2 = \| U_t w_t^\star \|_2 \leq \sqrt{d} \left( 1 + \max_{s \in \S}  g_t(s,w_t^\star  ) \right) \leq \sqrt{d} (1+B_t) \leq \alpha_t + \sqrt{td}(B_t+1) = W_t .
\end{equation}
By Claim \ref{clmParameters}, on the event $\mathcal{E}$, this implies
\begin{equation}
\| \hat{G}_t w_t^\star - w_t^\star \|_{\Lambda_t} = \| \hat{G}_t w_t^\star - U_t w_t^\star \|_{\Lambda_t} = \| E_t w_t^\star \|_{\Lambda_t} \leq \varepsilon_t \leq \alpha_t .
\end{equation}
Finally, for any $s \in \S$, using continuity, Lemma \ref{lemUtOptimism}, and Bellman optimality, we obtain
\begin{align}
f_t(s,w_t^\star) \leq \min_{a \in \A}  \phi(s,a)^\trans w_t^\star = \min_{a \in \A} \phi(s,a)^\trans \lim_{n \rightarrow \infty} U_t^n 0 = \min_{a \in \A}  \lim_{n \rightarrow \infty} \phi(s,a)^\trans U_t^n 0 \leq \min_{a \in \A} Q^\star(s,a) = J^\star(s) .
\end{align}
Hence, on the event $\mathcal{E}$, for any $t \in [T]$, $w_t^\star$ is an OAFP by the previous two inequalities and Definition \ref{defnOAFP}.

\subsection{Regret bound} \label{appRegretBound}

Recall from Algorithm \ref{algMain} that $M_l$ is the time the $l$-th interval ended and $L$ is the total number of intervals completed. Fix $\tilde{T} \in \N$ and let $\tilde{L} = \min \{ l \in [L] : M_l \geq T \wedge \tilde{T} \}$ denote the least number of intervals that encompass the times $1, \ldots , T \wedge \tilde{T}$. Also let $\tilde{K} = | \{ t \in [T \wedge \tilde{T}] : s_t' = s_{goal} \} |$ denote the number of episodes completed by time $T \wedge \tilde{T}$. Finally, define the regret incurred up to time $T \wedge \tilde{T}$ by
\begin{equation}
\tilde{R}(\tilde{T}) = \sum_{t=1}^{T \wedge \tilde{T}} c(s_t,a_t) - \sum_{k=1}^{\tilde{K}} J^\star(s_1^k) < \infty .
\end{equation}

\begin{lem}[Regret decomposition] \label{lemRegretDecomp}
Under the assumptions of Theorem \ref{thmRegret},
\begin{equation}
\tilde{R}(\tilde{T})  \leq \sum_{l=1}^{\tilde{L}-1} \left( \sum_{t=1+M_l}^{M_{l+1} \wedge \tilde{T}} c(s_t,a_t) - J^\star(s_{1+M_l}) \right)  + 2 (B_\star+1) d \log_2 ( 4 B_\star (T \wedge \tilde{T} ) / c_{min} ) .
\end{equation}
\end{lem}
\begin{proof}
Since $M_1 = 1$ in Algorithm \ref{algMain} and $c(s_1,a_1) \leq 1$, we can bound the total cost incurred by 
\begin{equation} \label{eqDecompAlg}
\sum_{t=1}^{T \wedge \tilde{T}} c(s_t,a_t) \leq 1 + \sum_{t=2}^{T \wedge \tilde{T}} c(s_t,a_t) = 1 +  \sum_{l=1}^{\tilde{L}-1} \sum_{t=1+M_l}^{M_{l+1} \wedge \tilde{T}} c(s_t,a_t)  ,
\end{equation}
where the equality holds because $M_{l+1} \wedge \tilde{T} = M_{l+1}$ for $l < \tilde{L}-1$ and $M_{\tilde{L}} \wedge \tilde{T} = T \wedge \tilde{T}$ by definition. On the other hand, the expected cost for the optimal policy can be written as
\begin{equation} \label{eqDecompOpt}
\sum_{k=1}^{\tilde{K}} J^\star(s_1^k) = \sum_{l=1}^{\tilde{L}} J^\star ( s_{1+M_l} ) + \left( \sum_{k=1}^{\tilde{K}} J^\star(s_1^k) - \sum_{l=1}^{\tilde{L}} J^\star ( s_{1+M_l} ) \right) .
\end{equation}
Thus, we seek a lower bound for the term in parentheses. First note that since Algorithm \ref{algMain} ends an interval each time an episode ends, for each $k \in [\tilde{K}]$, we can find $l \in [ \tilde{L} ]$ such that $s_1^k = s_{1+M_l}$. Hence, all summands cancel, except those corresponding to intervals $\mathcal{L} = \{ l \in [ \tilde{L} ] : M_l = 1 \text{ or } B_{M_l} = 2 B_{M_l-1} \text{ or } det ( \Lambda_{M_l} ) \geq 2 det ( \Lambda_{M_{l-1}} ) \}$, which may not have reached the goal state. Taken together, and since $J^\star(s_{1+M_l}) \leq B_\star$, we obtain
\begin{equation} \label{eqDecompToScriptL}
\sum_{k=1}^{\tilde{K}} J^\star(s_1^k) - \sum_{l=1}^{\tilde{L}} J^\star ( s_{1+M_l} ) \geq - \sum_{ l \in \mathcal{L} } J^\star ( s_{1+M_l} ) \geq - B_\star | \mathcal{L} | .
\end{equation}
It remains to bound $|\mathcal{L}|$. Clearly, $|\mathcal{L}| \leq 1 + \sum_{i=1}^2 |\mathcal{L}_i|$, where $\mathcal{L}_1 = \{ l \in [ \tilde{L} ] : B_{M_l} = 2 B_{M_l-1} \}$ and $\mathcal{L}_2 = \{ l \in [ \tilde{L} ] : det ( \Lambda_{M_l} ) \geq 2 det ( \Lambda_{M_{l-1}} ) \}$. For $\mathcal{L}_1$, note $\sup_{t \geq 0} B_t \geq 2^{ |\mathcal{L}_1| } c_{min}$ in Algorithm \ref{algMain}, so by Claim \ref{clmBtBound},
\begin{equation} \label{eqDecompScriptL1}
|\mathcal{L}_1| = \log_2 \left( 2^{ |\mathcal{L}_1| } c_{min}  / c_{min} \right) \leq \log_2  \left( \sup_{t \geq 0} B_t / c_{min} \right) \leq \log_2 ( 2 B_\star / c_{min} ) .
\end{equation}
For $\mathcal{L}_2$, we have $|\mathcal{L}_2| \leq |\mathcal{L}_2'| + 1$, where $\mathcal{L}_2' = \{ l \in [ \tilde{L}-1 ] : det ( \Lambda_{M_l} ) \geq 2 det ( \Lambda_{M_{l-1}} ) \}$ excludes $\tilde{L}-1$ if it belongs to $\mathcal{L}_2$. By definition, $M_{\tilde{L}-1} < T \wedge \tilde{T}$, which by Claim \ref{clmEigAndNorm} implies $det(\Lambda_{M_{\tilde{L}-1}}) \leq ( 1 + (T\wedge\tilde{T}))^d \leq ( 2 (T\wedge\tilde{T}))^d$. Hence, because $det(\Lambda_{M_{\tilde{L}-1}}) \geq 2^{|\mathcal{L}_2'|} det(\Lambda_0) = 2^{|\mathcal{L}_2'|} $ by definition of $\mathcal{L}_2'$, we obtain
\begin{equation} \label{eqDecompScriptL2}
|\mathcal{L}_2| \leq \log_2 ( 2^{ |\mathcal{L}_2'| } ) + 1 \leq \log_2 ( det(\Lambda_{M_{\tilde{L}-1}} ) ) + 1 \leq d \log_2 ( 2 (T \wedge \tilde{T} ) ) + 1 .
\end{equation} 
Recalling $|\mathcal{L}| \leq 1 + \sum_{i=1}^2 |\mathcal{L}_i|$ and combining \eqref{eqDecompAlg}, \eqref{eqDecompOpt}, \eqref{eqDecompToScriptL}, \eqref{eqDecompScriptL1}, and \eqref{eqDecompScriptL2}, we obtain
\begin{align}
\tilde{R}(\tilde{T})  - \sum_{l=1}^{\tilde{L}-1} \left( \sum_{t=1+M_l}^{M_{l+1} \wedge \tilde{T}} c(s_t,a_t) - J^\star(s_{1+M_l}) \right) &  \leq 1 +  B_\star  ( \log_2(2B_\star/c_{min}) + d \log_2 ( 2 (T \wedge \tilde{T} )  ) + 2 ) \\
&  \leq 2 (B_\star+1) d \log_2 ( 4 B_\star (T \wedge \tilde{T} ) / c_{min} ) ,
\end{align}
where the last inequality uses $B_\star \geq c_{min}$ and $d \geq 2$.
\end{proof}

We next bound the summand in Lemma \ref{lemRegretDecomp} by a martingale difference sequences and sum of bonuses.
\begin{lem}[Per-interval regret] \label{lemPerUpdateRegret} 
Under the assumptions of Theorem \ref{thmRegret} and on the event $\mathcal{E}$, for any $l \in [ \tilde{L}-1 ]$,
\begin{align}
\sum_{t=1+M_l}^{M_{l+1} \wedge \tilde{T}} c(s_t,a_t) - J^\star(s_{1+M_l}) & \leq \sum_{t=1+M_l}^{M_{l+1}\wedge \tilde{T}} \left( g_{M_l}(s_t',w_{M_l}) - \E_t [ g_{M_l}(s_t',w_{M_l}) ] \right)  +  3 \alpha_{M_l}  \sum_{t=1+M_l}^{M_{l+1}\wedge \tilde{T}}  \| \phi(s_t,a_t) \|_{\Lambda_{M_l}}^{-1}  .
\end{align}
\end{lem}
\begin{proof}
Define $\gamma_\tau = \sum_{t = \tau}^{M_{l+1}\wedge \tilde{T}} c(s_t,a_t) - f_{M_l} ( s_\tau , w_{M_l} )$ for each $\tau \in \{1+M_l, \ldots,M_{l+1}\wedge \tilde{T}\}$ and $\gamma_{ 1 + M_{l+1}\wedge \tilde{T} } = 0$. We claim, and will return to prove, that for any $\tau \in \{1+M_l, \ldots,M_{l+1}\wedge \tilde{T}\}$,
\begin{equation} \label{eqRecursion}
\gamma_\tau \leq \gamma_{\tau+1} + g_{M_l}(s_\tau',w_{M_l}) - \E_\tau [ g_{M_l}(s_\tau',w_{M_l}) ] + 3\alpha_{M_l} \| \phi(s_\tau,a_\tau ) \|_{\Lambda_{M_l}}^{-1} .
\end{equation}
Assuming \eqref{eqRecursion} holds, we prove the lemma. First, since $w_{M_l}$ is an OAFP, Definition \ref{defnOAFP} implies
\begin{equation} \label{eqOAFPkey}
f_{M_l} ( s_{1+M_l} , w_{M_l} ) \leq J^\star ( s_{1+M_l} ) , \quad \| \hat{G}_{M_l} w_{M_l} - w_{M_l} \|_{ \Lambda_{M_l} } \leq \alpha_{M_l}  .
\end{equation}
Using the first inequality in \eqref{eqOAFPkey}, we obtain
\begin{align}
\sum_{t=1+M_l}^{M_{l+1}\wedge \tilde{T}} c(s_t,a_t) - J^\star(s_{1+M_l}) \leq \sum_{t=1+M_l}^{M_{l+1}\wedge \tilde{T}} c(s_t,a_t) - f_{M_l} ( s_{1+M_l} , w_{M_l} ) = \gamma_{1+M_l}  ,
\end{align}
so the lemma follows from recursively applying \eqref{eqRecursion}. Hence, it only remains to prove \eqref{eqRecursion}. First observe
\begin{align} 
f_{M_l} ( s_\tau , w_{M_l} ) & = \phi(s_\tau,a_\tau)^\trans \hat{G}_{M_l} w_{M_l} + \phi(s_\tau,a_\tau)^\trans ( w_{M_l} - \hat{G}_{M_l} w_{M_l} ) - \alpha_{M_l} \| \phi(s_\tau,a_\tau) \|_{ \Lambda_{M_l}^{-1} } \quad \\
& \geq \phi(s_\tau,a_\tau)^\trans \hat{G}_{M_l} w_{M_l} - 2\alpha_{M_l} \| \phi(s_\tau,a_\tau) \|_{ \Lambda_{M_l}^{-1} } \geq \phi(s_\tau,a_\tau)^\trans U_{M_l} w_{M_l} - 3\alpha_{M_l} \| \phi(s_\tau,a_\tau) \|_{ \Lambda_{M_l}^{-1} } ,
\end{align}
where the equality holds by the policy update in Algorithm \ref{algMain} and the inequalities use Cauchy-Schwarz, the second bound in \eqref{eqOAFPkey}, and Corollary \ref{corConcOafpNew}. Now by definition, we have
\begin{align}
\phi(s_\tau,a_\tau)^\trans U_{M_l} w_{M_l} & = c(s_\tau,a_\tau) + \sum_{ s \in \S }  g_{M_l}(s,w_{M_l}) P(s|s_\tau,a_\tau)  \\
& = c(s_\tau,a_\tau) + g_{M_l}(s_\tau',w_{M_l} ) + \E_\tau [ g_{M_l}(s_\tau',w_{M_l} ) ] - g_{M_l}(s_\tau',w_{M_l} ) .
\end{align}
Combining the previous two inequalities and rearranging, we obtain
\begin{equation} \label{eqRecursionBeforeTwoCase}
c(s_\tau,a_\tau) - f_{M_l} ( s_\tau , w_{M_l} )  \leq - g_{M_l}(s_\tau',w_{M_l}) +  g_{M_l}(s_\tau',w_{M_l} )  - \E_\tau [ g_{M_l}(s_\tau',w_{M_l} ) ]  + 3 \alpha_{M_l}  \| \phi(s_\tau,a_\tau) \|_{ \Lambda_{M_l}^{-1} } .
\end{equation}
We complete the proof separately in each of two cases.
\begin{itemize}[leftmargin=10pt,itemsep=0pt,topsep=0pt]
\item If $\tau = M_{l+1} \wedge \tilde{T}$, the left side of \eqref{eqRecursionBeforeTwoCase} is $\gamma_\tau$ and $-g_{M_l}(s_\tau',w_{M_l}) \leq 0 = \gamma_{\tau+1}$, so \eqref{eqRecursionBeforeTwoCase} implies \eqref{eqRecursion}.
\item Otherwise, an interval did \textit{not} end between times $1+M_l$ and $\tau$ (inclusive). This implies (A) $s_\tau' \neq s_{goal}$, so $s_\tau' = s_{\tau+1}$, (B) $B_\tau = B_{M_l}$, and (C) $f_{M_l}(s_\tau',w_{M_l}) \leq B_\tau$. Taken together, (B) and (C) give $f_{M_l}(s_\tau',w_{M_l}) \leq B_{M_l}$, so (D) $- g_{M_l}(s_\tau',w_{M_l}) \leq - f_{M_l}(s_\tau',w_{M_l})$ by definition. Combining (A) and (D) with \eqref{eqRecursionBeforeTwoCase}, we obtain
\begin{align}
c(s_\tau,a_\tau) - f_{M_l} ( s_\tau , w_{M_l} ) \leq - f_{M_l}(s_{\tau+1},w_{M_l}) +  g_{M_l}(s_\tau',w_{M_l} )  - \E_\tau [ g_{M_l}(s_\tau',w_{M_l} ) ]  + 3\alpha_{M_l} \| \phi(s,a) \|_{ \Lambda_{M_l}^{-1} } .
\end{align}
Hence, recalling $\tau < M_{l+1} \wedge \tilde{T}$, we can use the definitions of $\gamma_\tau$ and $\gamma_{\tau+1}$ to obtain
\begin{align}
\gamma_\tau & = \sum_{t=\tau+1}^{M_{l+1}} c(s_t,a_t) + c(s_\tau,a_\tau) - f_{M_l} ( s_\tau , w_{M_l} ) \\
& \leq  \sum_{t=\tau+1}^{M_{l+1}} c(s_t,a_t) - f_{M_l}(s_{\tau+1},w_{M_l}) +  g_{M_l}(s_\tau',w_{M_l} )  - \E_\tau [ g_{M_l}(s_\tau',w_{M_l} ) ]  + 3\alpha_{M_l} \| \phi(s,a) \|_{ \Lambda_{M_l}^{-1} }  \\
& = \gamma_{\tau+1} +  g_{M_l}(s_\tau',w_{M_l} )  - \E_\tau [ g_{M_l}(s_\tau',w_{M_l} ) ] + 3\alpha_{M_l}   \| \phi(s,a) \|_{ \Lambda_{M_l}^{-1} } . \qedhere
\end{align}
\end{itemize}
\end{proof}

We next bound the martingale difference sequence from Lemma \ref{lemPerUpdateRegret}.
\begin{lem}[Martingale difference sequence] \label{lemMDS}
Under the assumptions of Theorem \ref{thmRegret}, for any $\delta > 0$, if we define
\begin{equation}
\mathcal{F} = \left\{ \sum_{l=1}^{\tilde{L}-1} \sum_{t=1+M_l}^{M_{l+1} \wedge \tilde{T}} \left( g_{M_l}(s_t',w_{M_l}) - \E_t [ g_{M_l}(s_t',w_{M_l}) ] \right) \leq 2 B_\star \sqrt{ (T \wedge \tilde{T}) \log \frac{8 (T \wedge \tilde{T})}{\delta} } \right\} ,
\end{equation}
then $\P(\mathcal{F}) \geq 1-\delta/2$.
\end{lem}
\begin{proof}
The left side of the inequality is a martingale difference sequence. By definition and Claim \ref{clmBtBound}, each term satisfies $g_{M_l}(s_t',w_{M_l}) - \E_t [ g_{M_l}(s_t',w_{M_l}) ] | \leq B_t \leq 2 B_\star$. The number of terms is $\leq M_{\tilde{L}} \wedge \tilde{T} \leq M_L \wedge \tilde{T} = T \wedge \tilde{T}$. The lemma follows from \cite[Theorem D.1]{rosenberg2020near} (an anytime version of Azuma's inequality).
\end{proof}

Finally, we bound the sum of bonuses from Lemma \ref{lemPerUpdateRegret}.
\begin{lem}[Sum of bonuses] \label{lemBonuses}
Under the assumptions of Theorem \ref{thmRegret},
\begin{equation}
\sum_{l=1}^{\tilde{L}-1}  3 \alpha_{M_l}    \sum_{t=1+M_l}^{M_{l+1}\wedge \tilde{T}} \| \phi(s_t,a_t) \|_{\Lambda_{M_l}^{-1}} \leq 6  \sqrt{ (T \wedge \tilde{T} ) d  \log(2(T \wedge \tilde{T} ))} \max_{t \in [T \wedge \tilde{T}]} \alpha_t  .
\end{equation}
\end{lem}
\begin{proof}
For each $l \in [\tilde{L}-1]$ and $t \in \{ 2 + M_l , \ldots , M_{l+1} \wedge \tilde{T} \}$, the $(l+1)$-th interval did \textit{not} end at time $t-1$, which implies $det(\Lambda_{t-1}) \leq 2 det ( \Lambda_{M_l})$. By \cite[Lemma 12]{abbasi2011improved} this implies $\Lambda_{M_l} - \Lambda_{t-1} / 2$ is positive semidefinite, so $\Lambda_{M_l}^{-1} - 2 \Lambda_{t-1}^{-1}$ is negative semidefinite, so
\begin{align} \label{eqBonusBound}
\| \phi(s_t,a_t) \|_{\Lambda_{M_l}^{-1}}  = \sqrt{ \phi(s_t,a_t)^\trans \Lambda_{M_l}^{-1} \phi(s_t,a_t) } \leq \sqrt{ 2 \phi(s_t,a_t)^\trans \Lambda_{t-1}^{-1} \phi(s_t,a_t) } = \sqrt{2} \| \phi(s_t,a_t) \|_{\Lambda_{t-1}^{-1}} .
\end{align} 
For any $l \in [ \tilde{L} - 1 ]$, the inequality clearly holds at time $t = 1+M_l$ as well. Combined with and Cauchy-Schwarz and the fact that $M_{\tilde{L}} \wedge \tilde{T} \leq M_L \wedge \tilde{T} = T \wedge \tilde{T}$ by definition, we thus obtain
\begin{align}
\sum_{l=1}^{\tilde{L}-1}  3 \alpha_{M_l}    \sum_{t=1+M_l}^{M_{l+1}\wedge \tilde{T}} \| \phi(s_t,a_t) \|_{\Lambda_{M_l}^{-1}} & \leq 3 \sqrt{2} \max_{t \in [T \wedge \tilde{T}]} \alpha_t \sum_{t=1}^{T\wedge \tilde{T}} \| \phi(s_t,a_t) \|_{\Lambda_{t-1}^{-1}} \\
& \leq 3 \sqrt{2}  \max_{t \in [T \wedge \tilde{T}]} \alpha_t \sqrt{ (T\wedge \tilde{T}  ) \sum_{t=1}^{T\wedge \tilde{T}} \| \phi(s_t,a_t) \|_{\Lambda_{t-1}^{-1}}^2 } .
\end{align}
Finally, by \cite[Lemma 11]{abbasi2011improved} and Claim \ref{clmEigAndNorm}, we have
\begin{align} 
& \sum_{t=1}^{ T \wedge \tilde{T} } \| \phi(s_t,a_t) \|_{\Lambda_{t-1}^{-1}}^2  \leq 2 \log \frac{ det ( \Lambda_{ T \wedge \tilde{T} } ) }{ det ( \Lambda_0 ) } \leq 2 d \log (  (T \wedge \tilde{T} ) + 1) \leq 2 d \log ( 2 ( T \wedge \tilde{T}) ) . \qedhere 
\end{align}
\end{proof}

We can now prove the regret bound on $\mathcal{E} \cap \mathcal{F}$. By Lemmas \ref{lemRegretDecomp}, \ref{lemPerUpdateRegret}, \ref{lemMDS}, and \ref{lemBonuses}, we know
\begin{align} \label{eqRegretT}
\tilde{R}(\tilde{T}) & \leq 6 \max_{t \in T \wedge \tilde{T}} \alpha_t \sqrt{ (T \wedge \tilde{T} ) d  \log(2(T \wedge \tilde{T} ))}  + 2 B_\star \sqrt{ (T \wedge \tilde{T}) \log ( 8 (T \wedge \tilde{T}) / \delta ) } \\
& \quad + 2 (B_\star+1) d \log_2 ( 4 B_\star (T \wedge \tilde{T} ) / c_{min} ) 
\end{align}
Next, recall $B_t + 1 \leq 2 B_\star + 1 \leq 2 ( B_\star+1)$ by Claim \ref{clmBtBound}. Combined with the assumption that $\kappa_t \leq \Psi t^\lambda \log(t+1)$ with $\Psi \geq 9d$ and $\lambda \in [0,\frac{1}{2})$ for any $t \in [ T \wedge \tilde{T} ]$, we have
\begin{align}
\alpha_t = (B_t+1) \kappa_t \sqrt{ \log( t (B_t+1) \kappa_t / \delta) } \leq 2(B_\star+1)\Psi t^\lambda \log(t+1) \sqrt{ \log (  2 ( B_\star + 1 ) \Psi t^{1+\lambda} \log(t+1) / \delta ) } .
\end{align}
Since $(B_\star+1) \Psi t^\gamma \log(t+1) / \delta \geq 9 \log 2 \geq 1$, $\log(t+1) \leq t$, and $\gamma \leq 1/2$, we also have 
\begin{equation}
t + 1 \leq 2 t \leq 2 ( B_\star + 1 ) \Psi t^{1+\lambda} \log(t+1) / \delta \leq 2 ( B_\star+1 ) \Psi t^{5/2} / \delta \leq ( 2 ( B_\star+1 ) \Psi t / \delta )^{5/2} .
\end{equation}
Hence, combining the previous two inequalities, we obtain
\begin{align}
6 \alpha_t & \leq 6 \cdot 2 (B_\star + 1 ) \Psi t^\gamma (5/2) \log ( 2 ( B_\star+1 ) \Psi t / \delta ) \sqrt{ (5/2) \log ( 2 ( B_\star+1 ) \Psi t / \delta ) } \\
& = (30 \sqrt{5/2} ) (B_\star + 1 ) \Psi t^\gamma \log^{3/2} ( 2 ( B_\star+1 ) \Psi t / \delta ) < 48 (B_\star + 1 ) \Psi t^\gamma \log^{3/2} ( 2 ( B_\star+1 ) \Psi t / \delta )
\end{align}
Since the right side is increasing in $t$, the first summand in \eqref{eqRegretT} can thus be upper bounded by $48 \bar{R}(\tilde{T})$, where
\begin{equation} \label{eqBarR}
\bar{R}(\tilde{T}) =   ( B_\star + 1 ) \sqrt{d} \Psi (T \wedge \tilde{T})^{\frac{1}{2}+\lambda} \log^2 (  2 (B_\star+1) \Psi ( T \wedge \tilde{T} ) / ( c_{min} \delta)   ) .
\end{equation}
Finally, $\Psi \geq 9d$ implies the other summands in \eqref{eqRegretT} are bounded by $\bar{R}(\tilde{T})$, so $\tilde{R}(\tilde{T}) \leq 50 \bar{R}(\tilde{T})$. 

Next, we show $R(K) = \tilde{R}(\tilde{T})$ when $\tilde{T}$ is large enough. Toward this end, first note that  by Assumption \ref{assSSPreg} and definition of $\tilde{R}(\tilde{T})$ and $\tilde{K}$, the bound $\tilde{R}(\tilde{T}) \leq 50 \bar{R}(\tilde{T})$ from the previous paragraph implies
\begin{equation} \label{eqRegToCmin}
( T \wedge \tilde{T}) c_{min} \leq \sum_{t=1}^{T \wedge \tilde{T}} c(s_t,a_t) = \tilde{R}(\tilde{T}) + \sum_{k=1}^{\tilde{K}} J^\star(s_1^k) \leq 50 \bar{R}(\tilde{T}) + K B_\star .
\end{equation}
Now consider two cases. First, if $T \wedge \tilde{T} \geq 100 \bar{R}(\tilde{T}) / c_{min}$, then $50 \bar{R}(\tilde{T}) \leq ( T \wedge \tilde{T}) c_{min} / 2$, so $T \wedge \tilde{T} \leq 2 K B_\star / c_{min}$ by \eqref{eqRegToCmin}. Otherwise, $T \wedge \tilde{T} \leq 100 \bar{R}(\tilde{T}) / c_{min}$, which by definition \eqref{eqBarR} implies
\begin{equation}
T \wedge \tilde{T} \leq \left( 100  ( B_\star + 1 ) \sqrt{d} \Psi / c_{min} \right) (T \wedge \tilde{T})^{\frac{1}{2}+\lambda} \log^2 \left( \left( 2 (B_\star+1) \Psi / ( c_{min} \delta ) \right) ( T \wedge \tilde{T} ) \right).
\end{equation}
By Claim \ref{clmLogIneq} below, this implies that for some $\iota_1, \iota_2 > 0$ depending only on $\lambda$ (which, by assumption on $\lambda$, means that $\iota_1, \iota_2$ are absolute constants), we have
\begin{equation}
T \wedge \tilde{T} \leq T_1 \triangleq \left( \frac{ \iota_1  (B_\star+1) \sqrt{d} \Psi }{ c_{min} } \log^2 \left( \frac{\iota_2 (B_\star+1) d  \Psi}{c_{min} \delta} \right) \right)^{\frac{2}{1-2\lambda}} .
\end{equation}
Combining the cases, we conclude $T \wedge \tilde{T} \leq T_2 \triangleq \max \{ 2 K B_\star / c_{min} , T_1 \} < \infty$. Hence, choosing $\tilde{T} \geq T_2$, we obtain $T \wedge \tilde{T} = T$ and $\tilde{K} = K$, which together imply $R(K) = \tilde{R}(\tilde{T}) < \infty$.

Finally, we establish the bounds of the theorem. Recall we have shown $R(K) = \tilde{R}(\tilde{T}) \leq 100 \bar{R}(\tilde{T})$ for large $\tilde{T}$ and $T \wedge \tilde{T} \leq \max \{ 2 K B_\star / c_{min} , T_1 \}$. We again consider two cases. First, if the bound $T \wedge \tilde{T} \leq 2 K B_\star / c_{min}$ holds, then by $R(K) = O (  \bar{R}(\tilde{T}) )$ and the definition \eqref{eqBarR},
\begin{equation} \label{eqRegProofLargeK}
R(K)= \tilde{O} \left( \left( B_\star^{\frac{3}{2}+\lambda} + B_\star^{\frac{1}{2}+\lambda} \right) d^{\frac{1}{2}} \Psi ( K / c_{min} )^{\frac{1}{2}+\lambda} \right) .
\end{equation}
If instead only $T \wedge \tilde{T} \leq T_1$ holds, then \eqref{eqBarR} implies
\begin{equation} \label{eqRegProofSmallK}
R(K) = \tilde{O} \left(  ( B_\star + 1 ) \sqrt{d} \Psi T_1^{\frac{1}{2}+\lambda} \right) = \tilde{O} \left( (B_\star+1)^{\frac{2}{1-2\lambda}} d^{\frac{1}{1-2\lambda} } \Psi^{\frac{2}{1-2\lambda}}  c_{min}^{-\frac{1+2\lambda}{1-2\lambda}} \right) .
\end{equation}
Finally, bounding $R(K)$ by the max of the cases, then the max by the sum, yields the desired bound.

\begin{rem}[Bound for $T$] \label{remTbound}
As shown above, for $\tilde{T} \geq T_2$, we have the bound $T \leq T_2$, or (by definition)
\begin{equation}
T = O \left( \max \left\{ \frac{2KB_\star}{c_{min}} , \left( \frac{   (B_\star+1) \sqrt{d} \Psi }{ c_{min} } \log^2 \left( \frac{ (B_\star+1) d  \Psi}{c_{min} \delta} \right) \right)^{\frac{2}{1-2\lambda}} \right\} \right) .
\end{equation}
\end{rem}

\begin{rem}[Sharpening the tabular case] \label{remSharpeningTabular2}
When $K$ is large, the bound \eqref{eqRegProofLargeK} holds and has $\sqrt{d} \Psi$ dependence on $d$. We required $\Psi$ (the constant component of $\kappa_t$) to be linear in $d$ in order to match the scaling of the error bound $\varepsilon_t$ from Appendix \ref{appOperatorConc}. In the tabular case, we can reduce $\varepsilon_t$'s dependence to $\sqrt{d}$, (see Remark \ref{remSharpeningTabular1} in Appendix \ref{appProofMainEkTail}), so we can choose $\Psi$ to scale as $\sqrt{d}$, after which the regret bound's dependence becomes linear in $d$.
\end{rem}

\begin{clm} \label{clmLogIneq}
Suppose $x \leq a x^{c_1}\log^2 ( b x )$ for some $c_1 \in (0,1)$ and $a , b , x \geq 1$. Then $x \leq ( c_2 a \log^2 (c_3 a b ) )^{1/(1-c_1)}$ for some constants $c_2, c_3 > 0$ that depend only on $c_1$.
\end{clm}
\begin{proof}
By the assumed inequality and $\log y \leq y\ \forall\ y \geq 0$, we have
\begin{equation}
x \leq  \frac{ 16 a x^c_1 }{ (1-c_1)^2 } \left( \frac{1-c_1}{4} \log ( b x ) \right)^2 =  \frac{ 16 a x^{c_1} }{ (1-c_1)^2 }  \left(  \log \left( (bx)^{\frac{1-c_1}{4}} \right) \right)^2 \leq \frac{ 16 a x^{c_1} }{ (1-c_1)^2 } ( b x )^{\frac{1-c_1}{2} } = \frac{ 16 a b^{\frac{1-c_1}{2}}}{ (1-c_1)^2 } x^{\frac{1+c_1}{2}} .
\end{equation}
Solving for $x$, we obtain $x \leq ( 16 a / (1-c_1)^2 )^{2/(1-c_1)} b$. Plugging back into the log term of the assumed inequality, and since $2 \leq 2/(1-c_1)$, we obtain
\begin{equation}
x \leq a x^{c_1}  \left( \log \left( \left( \frac{16 a }{ (1-c_1)^2 } \right)^{ \frac{2}{1-c_1} }  b^2 \right) \right)^2 \leq a x^{c_1} \left( \frac{2}{1-c_1} \log \left( \frac{ 16 a b }{ (1-c_1)^2 } \right) \right)^2 = c_2 a \log^2 ( c_3 a b ) x^{c_1} ,
\end{equation}
where we define $c_2 = 4 / ( 1-c_1)^2$ and $c_3 = 16 / ( 1-c_1)^2$. Solving for $x$ gives the desired bound.
\end{proof}

\section{Proofs of Theorems \ref{thmProb}-\ref{thmOrth}} \label{appOracleProof}

We begin with an optimism lemma used for all three proofs.
\begin{lem}[$\hat{G}_t$ optimism] \label{lemGtOptimism}
If Assumptions \ref{assSSPreg} and \ref{assSSPlin} hold and $\kappa_t \geq 9d$, then on the event $\mathcal{E}$,
\begin{equation}
\phi(s,a)^\trans \hat{G}_t^{n-1} 0 - Q^\star(s,a) \leq \varepsilon_t \| \phi(s,a) \|_{\Lambda_t^{-1}} \leq \alpha_t \| \phi(s,a) \|_{\Lambda_t^{-1}}\ \forall\ (s,a) \in \S \times \A , n \in \N , t \in [T] .
\end{equation}
\end{lem}
\begin{proof} 
The proof is similar to that of Lemma \ref{lemUtOptimism}. For $n = 1$, $\phi(s,a)^\trans \hat{G}_t^{n-1} 0 = 0 \leq Q^\star(s,a)$, so the result is immediate. Now assume the bound holds for $n \in \N$. Then by the Bellman optimality equation \eqref{eqBellman}, we obtain
\begin{equation} \label{eqFtOptimism}
f_t(s, \hat{G}_t^{n-1} 0 ) = \min_{a \in \A} \left( \phi(s,a)^\trans \hat{G}_t^{n-1} 0 - \alpha_t \| \phi(s,a) \|_{\Lambda_t^{-1}} \right) \leq \min_{a \in \A} Q^\star(s,a) = J^\star(s) .
\end{equation}
Hence, by Claim \ref{clmGtNonex}, $g_t(s, \hat{G}_t^{n-1} 0) \leq J^\star(s)$ as well. Again by Bellman optimality, this implies
\begin{equation}
\phi(s,a)^\trans U_t ( \hat{G}_t^{n-1} 0 ) = c(s,a) + \sum_{s' \in \S} g_t(s, \hat{G}_t^{n-1} 0) P(s'|s,a) \leq c(s,a) + \sum_{s' \in \S}  J^\star(s') P(s'|s,a) = Q^\star(s,a) .
\end{equation}
On the other hand, by Corollary \ref{corConcGhatNew} and the assumption $\kappa_t \geq 9 d$, we have
\begin{equation}
\phi(s,a)^\trans ( \hat{G}_t^n 0 - U_t ( \hat{G}_t^{n-1} 0 ) ) \leq \varepsilon_t \| \phi(s,a) \|_{\Lambda_t^{-1}} \leq \alpha_t \| \phi(s,a) \|_{\Lambda_t^{-1}}  .
\end{equation}
Combining the last two inequalities completes the inductive step.
\end{proof}

As a simple corollary, we have the following formal version of Lemma \ref{lemGtOptimismInformal} from the main text.
\begin{cor}[$\hat{G}_t$ optimism] \label{corGtOptimism}
If Assumptions \ref{assSSPreg} and \ref{assSSPlin} hold and $\kappa_t \geq 9d$, then on the event $\mathcal{E}$, for any $t \in [T]$, $n \in \N$, and $(s,a) \in \S \times \A$, we have $f_t(s, \hat{G}_t^{n-1} 0 ) \leq J^\star(s)$.
\end{cor}
\begin{proof}
Rearrange the Lemma \ref{lemGtOptimism} bound and take minimum over $a \in \A$ as in \eqref{eqFtOptimism}.
\end{proof}

The preceding corollary implies the optimism inequality in \eqref{eqOAFP}. For the fixed point inequality $\| \hat{G}_t^n 0 - \hat{G}_t^{n-1} 0 \|_{\Lambda_t} \leq \alpha_t$, we use a similar approach for Theorems \ref{thmProb} and \ref{thmProp}, so we will provide a general result (Lemma \ref{lemGtContraction} below) for use in both theorems. Toward this end, we begin with an intermediate claim. Note that, while the bound grows with $n$ for fixed $t$, we will later choose $n$ in terms of $t$ so that the bound vanishes as $t \rightarrow \infty$.
\begin{clm}[$\hat{G}_t$ tracks $U_t$] \label{clmGtTracksUt}
If Assumptions \ref{assSSPreg} and \ref{assSSPlin} hold and $\kappa_t \geq 18d$, then on the event $\mathcal{E}$, 
\begin{equation} 
| \phi(s,a)^\trans ( \hat{G}_t^{n-1} 0 - U_t^{n-1} 0 ) | \leq \varepsilon_t \| \phi(s,a) \|_{\Lambda_t^{-1}} + \frac{2 (B_t+1) (n-1) \varepsilon_t}{\alpha_t}\ \forall\ (s,a) \in \S \times \A , n \in \N , t \in [T]
\end{equation}
\end{clm}
\begin{proof}
First, for any $n \in \N$, we use Corollary \ref{corConcGhatNew} and $\kappa_t \geq 18 d$ to write
\begin{equation}
\phi(s,a)^\trans \hat{G}_t^n 0 \leq \phi(s,a)^\trans U_t ( \hat{G}_t^{n-1} 0 ) + \alpha_t \| \phi(s,a) \|_{\Lambda_t^{-1}}  / 2 . 
\end{equation}
By Claim \ref{clmOpBounds}, we also know $\phi(s,a)^\trans U_t ( \hat{G}_t^{n-1} 0 ) \leq B_t + 1$; combined with the previous inequality, we obtain
\begin{equation} \label{eqBeforeExplore}
\phi(s,a)^\trans  \hat{G}_t^n 0 - \alpha_t \| \phi(s,a) \|_{\Lambda_t^{-1}} \leq B_t + 1 - \alpha_t \| \phi(s,a) \|_{\Lambda_t^{-1}}  / 2 .
\end{equation}
Now define the ``explored'' states at time $t$ (i.e., those with small bonuses across actions) by
\begin{equation}
\S_t = \left\{ s \in \S : \max_{a \in \A} \| \phi(s,a) \|_{\Lambda_t^{-1}} \leq \frac{2 (B_t+1)}{  \alpha_t }  \right\} .
\end{equation}
Then for any \textit{un}explored state $s \in \S \setminus \S_t$, \eqref{eqBeforeExplore} implies that for some $a \in \A$,
\begin{align}
\phi(s,a)^\trans \hat{G}_t^n 0 - \alpha_t \| \phi(s,a) \|_{\Lambda_t^{-1}} \leq B_t + 1 - \frac{\alpha_t}{2} \times \frac{2(B_t+1)}{ \alpha_t } = 0 .
\end{align}
Taking minimum over $a \in \A$ on both sides gives $f_t ( s , \hat{G}_t^n 0 )  \leq 0$, which implies $g_t ( s , \hat{G}_t^n 0 ) = 0$. Again using Claim \ref{clmOpBounds}, we similarly obtain that for any $s \in \S \setminus \S_t$ and some $a \in \A$,
\begin{equation}
\phi(s,a)^\trans U_t^n 0 - \alpha_t \| \phi(s,a) \|_{\Lambda_t^{-1}} \leq B_t+1 - \alpha_t \| \phi(s,a) \|_{\Lambda_t^{-1}} \leq - (B_t+1) < 0 ,
\end{equation}
so $g_t(s, U_t^n0) = 0$ as well. Since $n \in \N$ was arbitrary and $g_t(s, \hat{G}_t^0 0) = g_t(s, U_t^0 0) = g(s,0)= 0$, we conclude
\begin{equation}
g_t(s, \hat{G}_t^{n-1} 0) = g_t(s, U_t^{n-1} , 0) = 0\ \forall\ s \in \S \setminus \S_t , n \in \N .
\end{equation}
Combined with Claim \ref{clmGtNonex}, this implies that for any $n \in \N$ and $(s,a) \in \S \times \A$,
\begin{align}
| \phi(s,a)^\trans ( U_t ( \hat{G}_t^{n-1} 0 ) - U_t ( U_t^{n-1} 0 ) ) |  & \leq \sum_{ s' \in \S } | g_t(s',\hat{G}_t^{n-1}0) - g_t(s',U_t^{n-1}0) | P(s'|s,a)  \\
& \leq \max_{s' \in \S_t} | g_t(s',\hat{G}_t^{n-1}0) - g_t(s',U_t^{n-1} 0) | \\
& \leq \max_{(s',a') \in \S_t \times \A} | \phi(s',a')^\trans ( \hat{G}_t^{n-1}0 - U_t^{n-1} 0 ) | .
\end{align}
Hence, again using Corollary \ref{corConcGhatNew}, for any $n \in \N$ and $(s,a) \in \S \times \A$, we obtain
\begin{align} \label{eqNotJustExplored}
| \phi(s,a)^\trans ( \hat{G}_t^n 0 - U_t^n 0 ) | & \leq | \phi(s,a)^\trans ( \hat{G}_t^n 0 - U_t ( \hat{G}_t^{n-1} 0 ) ) | + | \phi(s,a)^\trans ( U_t ( \hat{G}_t^{n-1} 0 ) - U_t ( U_t^{n-1} 0 ) ) | \\
& \leq \varepsilon_t \| \phi(s,a) \|_{\Lambda_t^{-1}} + \max_{(s',a') \in \S_t \times \A} | \phi(s',a')^\trans ( \hat{G}_t^{n-1}0 - U_t^{n-1}0 ) |  .
\end{align}
Thus, taking the maximum over $(s,a) \in \S_t \times \A$ on both sides, by definition of $\S_t$, we have shown
\begin{equation}
\max_{(s,a) \in \S_t \times \A} | \phi(s,a)^\trans ( \hat{G}_t^n 0 - U_t^n 0 ) | \leq \frac{2(B_t+1)\varepsilon_t}{\alpha_t} + \max_{(s,a) \in \S_t \times \A} | \phi(s,a)^\trans ( \hat{G}_t^{n-1}0 - U_t^{n-1} 0 ) |\ \forall\ n \in \N .
\end{equation}
Iterating this inequality, and since $\hat{G}_t^0 0 = U_t^0 0 = 0$, we conclude
\begin{equation}
\max_{(s,a) \in \S_t \times \A} | \phi(s,a)^\trans ( \hat{G}_t^{n-1} 0 - U_t^{n-1} 0 ) | \leq \frac{2(B_t+1)\varepsilon_t (n-1)}{\alpha_t}\ \forall\ n \in \N .
\end{equation}
Substituting back into \eqref{eqNotJustExplored} and again using $\hat{G}_t^0 0 = U_t^0 0 = 0$, we obtain the desired result.
\end{proof}

We can now state the aforementioned Lemma \ref{lemGtConvergence}. Note that while we have already established a polynomial rate of convergence for $U_t$ in Lemma \ref{lemUtConvergence}, we keep the rate general here, since Theorem \ref{thmProp} will use an improved rate.
\begin{lem}[$\hat{G}_t$ convergence] \label{lemGtConvergence}
If Assumptions \ref{assSSPreg} and \ref{assSSPlin} hold and $\kappa_t \geq 18d$, then on the event $\mathcal{E}$,
\begin{equation} 
\| \hat{G}_t^n 0 - \hat{G}_t^{n-1} 0 \|_{\Lambda_t} \leq 5 \sqrt{d} \varepsilon_t + \frac{ 10 \sqrt{t} (B_t+1) n \varepsilon_t}{\alpha_t} + 3 \sqrt{t} \max_{(s,a) \in \S \times \A} | \phi(s,a)^\trans ( U_t^n 0 - U_t^{n-1} 0 ) |\ \forall\ n \in \N , t \in [T] .
\end{equation}
\end{lem}
\begin{proof}
We consider three cases (the last two are corner cases). For the first and most natural case, we assume that $\| \hat{G}_t^n 0 - \hat{G}_t^{n-1} 0 \|_2 \leq \| \hat{G}_t^n 0 - \hat{G}_t^{n-1} 0 \|_{\Lambda_t} / \sqrt{2}$. Then by definition of the induced norm,
\begin{equation}
\| \hat{G}_t^n 0 - \hat{G}_t^{n-1} 0 \|_{\Lambda_t}^2 \leq 2 \sum_{\tau=1}^t ( \phi(s_\tau,a_\tau)^\trans ( \hat{G}_t^n 0 - \hat{G}_t^{n-1} 0 ) )^2 .
\end{equation}
Next, for any $(s,a) \in \S \times \A$ and $m \in \{ n-1 , n \}$, Claim \ref{clmGtTracksUt} and Cauchy-Schwarz  imply
\begin{equation}
( \phi(s,a)^\trans ( \hat{G}_t^m 0 - U_t^m 0 ) )^2 \leq \left( \varepsilon_t \| \phi(s,a) \|_{\Lambda_t^{-1}} + \frac{2 (B_t+1) n \varepsilon_t}{\alpha_t}  \right)^2 \leq 2 \varepsilon_t^2 \| \phi(s,a) \|_{\Lambda_t^{-1}}^2 + \frac{8 (B_t+1)^2 n^2 \varepsilon_t^2}{\alpha_t^2}  ,
\end{equation}
which, after another application of Cauchy-Schwarz, gives
\begin{align}
( \phi(s,a)^\trans ( \hat{G}_t^n0 - \hat{G}_t^{n-1} 0 ) )^2 & \leq  3 \sum_{m=n-1}^n  ( \phi(s,a)^\trans ( \hat{G}_t^m 0 - U_t^m 0 ) )^2 + 3 ( \phi(s,a)^\trans ( U_t^n 0 - U_t^{n-1} 0 ) )^2 \\
& \leq 12 \varepsilon_t^2 \| \phi(s,a) \|_{\Lambda_t^{-1}}^2 + \frac{48 (B_t+1)^2 n^2 \varepsilon_t^2}{\alpha_t^2} + 3 \max_{(s',a') \in \S \times \A} ( \phi(s',a')^\trans ( U_t^n 0 - U_t^{n-1} 0 ) )^2 .
\end{align}
By \cite[Lemma D.1]{jin2020provably}, we know $\sum_{\tau=1}^t  \| \phi(s_\tau,a_\tau) \|_{\Lambda_t^{-1}}^2 \leq d$. Combined with previous three inequalities, 
\begin{align}
\| \hat{G}_t^n 0 - \hat{G}_t^{n-1} 0 \|_{\Lambda_t}^2 \leq 24 d \varepsilon_t^2 + \frac{96 t (B_t+1)^2 n^2 \varepsilon_t^2}{\alpha_t^2} + 6 t \max_{(s,a) \in \S \times \A} ( \phi(s,a)^\trans ( U_t^n 0 - U_t^{n-1} 0 ) )^2 .
\end{align}
Taking square roots on both sides, bounding the square root of sum by the sum of square roots, and using $\sqrt{24} \leq 5$, $\sqrt{96} \leq 100$, and $\sqrt{6} \leq 3$ yields the desired bound.

For the second case, suppose $\| \hat{G}_t^n 0 - \hat{G}_t^{n-1} 0 \|_2 > \| \hat{G}_t^n 0 - \hat{G}_t^{n-1} 0 \|_{\Lambda_t} / \sqrt{2}$ and $n \geq 2$. Then 
\begin{align} \label{eqWeirdL2case}
\| \hat{G}_t^n 0 - \hat{G}_t^{n-1} 0 \|_{\Lambda_t} < 2 \sum_{m=n-1}^n \left( \| \hat{G}_t^m 0 - U_t ( \hat{G}_t^{m-1} 0 ) \|_2 + \| U_t ( \hat{G}_t^{m-1} 0 ) \|_2 \right) \leq 4 ( \varepsilon_t + d(B_t+1) ) ,
\end{align}
where we used Claim \ref{clmEigAndNorm}, Corollary \ref{corConcGhatNew}, and Claim \ref{clmOpBounds}. By assumption $\kappa_t \geq 18d$, we also know
\begin{equation} \label{eqEpsilonLower}
\varepsilon_t / 5 =  ( B_t+1) d \sqrt{ \log ( t (B_t+1) \kappa_t \sqrt{ \log (t(B_t+1)\kappa_t/\delta) } / \delta ) } \geq  ( B_t+1 ) d \geq 2 .
\end{equation}
Hence, combining the previous two bounds, we obtain $\| \hat{G}_t^n 0 - \hat{G}_t^{n-1} 0 \|_{\Lambda_t} \leq 24 \varepsilon_t / 5 \leq 5 \sqrt{d} \varepsilon_t$.

Finally, suppose $\| \hat{G}_t^n 0 - \hat{G}_t^{n-1} 0 \|_2 > \| \hat{G}_t^n 0 - \hat{G}_t^{n-1} 0 \|_{\Lambda_t} / \sqrt{2}$ and $n = 1$. Then by Claim \ref{clmOpBounds} and \eqref{eqEpsilonLower},
\begin{align}
\| \hat{G}_t^n 0 - \hat{G}_t^{n-1} 0 \|_{\Lambda_t} & < \sqrt{2} \| \hat{G}_t^n 0 - \hat{G}_t^{n-1} 0 \|_2 = \sqrt{2} \| \hat{G}_t 0 \|_2 \leq \sqrt{ 8 d } \leq 5 \sqrt{ d } \varepsilon_t . \qedhere
\end{align}
\end{proof}

We now proceed to the proofs of the theorems.

\subsection{Proof of Theorem \ref{thmProb}} \label{appProofProb}

We begin with a corollary of Lemmas \ref{lemUtConvergence} and \ref{lemGtConvergence} in the setting of Theorem \ref{thmProb}. The proof is mostly algebra.
\begin{cor}[$\hat{G}_t$ convergence] \label{corGtConvergenceProb}
Under the Assumptions of Theorem \ref{thmProb} and on the event $\mathcal{E}$, for any $t \in [T]$,
\begin{equation}
\min_{n \in [ \ceil{ 2 d t^{1/6}} ]} \| \hat{G}_t^n 0 - \hat{G}_t^{n-1} 0 \|_{\Lambda_t} \leq \alpha_t .
\end{equation}
\end{cor}
\begin{proof}
Since $\kappa_t = 54 d t^{1/3} = 9 ( 6 t^{1/3} ) d$ by assumption in Theorem \ref{thmProb}, Claim \ref{clmParameters} implies
\begin{equation} \label{eqParamThmProb}
\alpha_t \geq \max \{ 6 t^{1/3} \varepsilon_t , (B_t+1) \kappa_t \} .
\end{equation}
Hence, using the bound $\alpha_t \geq (B_t+1) \kappa_t$, we have
\begin{equation}
\frac{ \alpha_t }{ 9 \varepsilon_t } = \frac{ \kappa_t \sqrt{ \log ( t(B_t+1) \kappa_t / \delta ) } }{ 45 d \sqrt{ \log ( t \alpha_t / \delta ) } } \leq \frac{ \kappa_t }{ 40 d } = \frac{ 54 t^{1/3} }{ 45 } < 4 t^{1/3} .
\end{equation}
Thus, if we define $N_t = d \sqrt{ \alpha_t / \varepsilon_t } / 3$, we are guaranteed that $N_t \leq 2 d t^{1/6}$, so $\ceil{N_t} \in [ \ceil{ 2 d t^{1/6} } ]$. Combining this result with Lemmas \ref{lemUtConvergence} and \ref{lemGtConvergence} (we can invoke the latter since $\kappa_t \geq 18d$), we obtain
\begin{equation} \label{eqThmProbAlg}
\min_{n \in [\ceil{ 2d t^{1/6} } ] } \| \hat{G}_t^n 0 - \hat{G}_t^{n-1} 0 \|_{\Lambda_t} \leq \| \hat{G}_t^{\ceil{N_t}} 0 - \hat{G}_t^{\ceil{N_t}-1} 0 \|_{\Lambda_t} \leq 5 \sqrt{ d} \varepsilon_t + \frac{ 10 \sqrt{ t} (B_t+1) \ceil{N_t} \varepsilon_t}{\alpha_t} + \frac{3\sqrt{t} (B_t+1) d^2 }{ \ceil{N_t} } .
\end{equation}
For the third term in \eqref{eqThmProbAlg}, since $\ceil{N_t} \geq N_t$, we have
\begin{equation}
3\sqrt{t} (B_t+1) d^2 / \ceil{N_t}  \leq  3\sqrt{t} (B_t+1) d^2 / N_t  = 9 \sqrt{t}(B_t+1) d / \sqrt{  \alpha_t / \varepsilon_t } .
\end{equation}
For the second term, since $N_t \geq 2 \sqrt{ 6 } / 3 \geq 1$ by \eqref{eqParamThmProb}, we have $\ceil{N_t} \leq 2N_t$, so
\begin{equation}
10 \sqrt{ t} (B_t+1) \ceil{N_t} \varepsilon_t / \alpha_t \leq 21 \sqrt{t}(B_t+1) N_t \varepsilon_t / \alpha_t = 7 \sqrt{t} (B_t+1) d / \sqrt{ \alpha_t / \varepsilon_t  } .
\end{equation}
Hence, because $\sqrt{ \alpha_t / \varepsilon_t } \geq 2 t^{1/6}$ and $\alpha_t \geq \kappa_t(B_t+1)$ by \eqref{eqParamThmProb}, the last two terms in \eqref{eqThmProbAlg} can be bounded by
\begin{equation}
\frac{ 10 \sqrt{ t} (B_t+1) \ceil{N_t} \varepsilon_t}{\alpha_t} + \frac{3\sqrt{t} (B_t+1) d^2 }{ \ceil{N_t} } \leq  \frac{16 \sqrt{t} d (B_t+1) }{ \sqrt{ \alpha_t / \varepsilon_t } } \leq 8 t^{1/3} d (B_t+1) = \frac{ 8 \kappa_t (B_t+1) }{ 54 } \leq \frac{ 8 \alpha_t }{ 54 } \leq \frac{ \alpha_t }{ 6 } .
\end{equation}
Substituting into \eqref{eqThmProbAlg}, and assuming for the moment that $\sqrt{d} \leq t^{1/3}$, we can use \eqref{eqParamThmProb} to obtain
\begin{equation}
\min_{n \in [\ceil{ 2d t^{1/6} } ] } \| \hat{G}_t^n 0 - \hat{G}_t^{n-1} 0 \|_{\Lambda_t} \leq 5 \sqrt{d} \cdot \varepsilon_t + \frac{ \alpha_t }{ 6 } \leq 5 t^{1/3} \cdot \frac{ \alpha_t }{ 6 t^{1/3} } + \frac{ \alpha_t }{ 6 } = \alpha_t .
\end{equation}
If instead $\sqrt{d} > t^{1/3}$, then $t^{1/6} < \sqrt{d}$ as well, so we can instead use Claim \ref{clmOpBounds} and \eqref{eqParamThmProb} to obtain
\begin{align}
\min_{n \in [\ceil{ 2d t^{1/6} } ] } \| \hat{G}_t^n 0 - \hat{G}_t^{n-1} 0 \|_{\Lambda_t} & \leq \| \hat{G}_t 0 \|_{\Lambda_t} \leq \sqrt{ 8 t d } = \sqrt{8} t^{1/3} t^{1/6} d^{1/2} \leq \sqrt{8} t^{1/3} d \leq \kappa_t \leq \alpha_t . \qedhere
\end{align}
\end{proof}

We can now prove Theorem \ref{thmProb}. Recall from Appendix \ref{appRegretProof} that the regret bound in Theorem \ref{thmRegret} holds on $\mathcal{E} \cap \mathcal{F}
$. Hence, on this event, and since $\kappa_t = 60 d t^{1/3}$ in Theorem \ref{thmProb}, we can set $\Psi = 60d$ and $\lambda = 1/3$ to obtain
\begin{align} \label{eqThmProbReg}
R(K) 
& = \tilde{O} \left( \left( B_\star^{\frac{11}{6}} + B_\star^{\frac{5}{6}} \right) d^{\frac{3}{2}} ( K / c_{min} )^{\frac{5}{6}} + (B_\star+1)^6 d^9 c_{min}^{-5}\right) .
\end{align}
Finally, Corollaries \ref{corGtOptimism} and \ref{corGtConvergenceProb} imply that on $\mathcal{E}$, for any $t \in [T]$ Algorithm \ref{algIterate} is called, it returns an OAFP within $O(d t^{1/6})$ iterations. Together with Lemmas \ref{lemMainEkTail2} and \ref{lemMDS}, which ensure $\P(\mathcal{E} \cap \mathcal{F}) \geq 1-\delta$, this completes the proof.

\subsection{Proof of Theorem \ref{thmProp}}

As discussed above, we first establish geometric convergence using the contraction property.
\begin{lem}[Geometric $U_t$ convergence] \label{lemUtConvergenceGeom}
If Assumptions \ref{assSSPreg} and \ref{assSSPlin} hold and all stationary policies are proper,
\begin{equation}
| \phi(s,a)^\trans ( U_t^n 0 - U_t^{n-1} 0 ) | \leq \chi \rho^{n-1}\ \forall\ (s,a) \in \S \times \A , n \in \N , t \in [T] .
\end{equation}
\end{lem}
\begin{proof}
Fix $t$. When $n=1$, since $\chi \geq 1$ by definition, we simply have
\begin{equation} \label{eqUtGeomBase}
| \phi(s,a)^\trans ( U_t^n 0 - U_t^{n-1} 0 ) | = | \phi(s,a)^\trans \theta | = | c(s,a) | \leq 1 \leq \chi = \chi \rho^{n-1} .
\end{equation}
It remains to show $| \phi(s,a)^\trans ( U_t^{n+1} 0 - U_t^n 0 ) | \leq \chi \rho^n$ for all $n \in \N$. Fix such an $n$. By monotoncity of $\Pi_{[0,B_t]}$,
\begin{equation}
g_t ( s , U_t^n 0 )  = \min_{a \in \A} \Pi_{[0,B_t]}  \left( \phi(s,a)^\trans U_t^n 0 - \alpha \| \phi(s,a) \|_{\Lambda_t^{-1}}  \right)  .
\end{equation}
Hence, if we define $Q_n \in \R^{\S \times \A}$ to be the matrix with $(s,a)$-th element
\begin{equation}
Q_n(s,a) = \Pi_{[0,B_t]} \left( \phi(s,a)^\trans U_t^n 0 - \alpha_t \| \phi(s,a) \|_{\Lambda_t^{-1}} \right) ,
\end{equation}
we have $g_t(s, U_t^n 0 ) = \min_{a \in \A} Q_n(s,a)$. Thus, by definition of $U_t$, we obtain
\begin{align}
\phi(s,a)^\trans ( U_t^{n+1} 0 - U_t^n 0 ) = \sum_{ s' \in \S }  \left( \min_{a' \in \A} Q_n(s',a') - \min_{a' \in \A} Q_{n-1}(s',a') \right) P(s'|s,a) = ( \mathcal{T} Q_n - \mathcal{T} Q_{n-1} ) (s,a) ,
\end{align}
where $\mathcal{T}$ is the state-action operator defined in \eqref{eqOperator}. By \eqref{eqContraction}, this implies
\begin{align}
\max_{(s,a) \in \S \times \A} \omega(s) | \phi(s,a)^\trans ( U_t^{n+1} 0 - U_t^n 0 ) | & \leq \rho \max_{(s,a) \in \S \times \A} \omega(s) | ( Q_n- Q_{n-1} ) (s,a) | \\
& \leq \rho \max_{(s,a) \in \S \times \A} \omega(s) | \phi(s,a)^\trans ( U_t^n 0 - U_t^{n-1} 0 ) | ,
\end{align}
where the last inequality holds by Claim \ref{clmClipNonex}. Iterating the previous inequality and using the bound $|\phi(s,a)^\trans ( U_t^1 0 - U_t^0 0 ) | \leq 1$ from \eqref{eqUtGeomBase}, we obtain that for any $n \in \N$,
\begin{equation}
\max_{(s,a) \in \S \times \A} \omega(s) | \phi(s,a)^\trans ( U_t^{n+1} 0 - U_t^n 0 ) | \leq \rho^n \max_{s \in \S} \omega(s) .
\end{equation}
Hence, for any $(s,a) \in \S \times \A$ and $n \in \N$, we obtain
\begin{align}
| \phi(s,a)^\trans ( U_t^{n+1} 0 - U_t^n 0 ) | & \leq \frac{\omega(s) | \phi(s,a)^\trans ( U_t^{n+1} 0 - U_t^n 0 ) |} {\min_{s' \in \S} \omega(s')} \leq  \frac{\max_{s' \in \S} \omega(s') \rho^n}{ \min_{s' \in \S} \omega(s') } = \chi \rho^n .  \qedhere
\end{align}
\end{proof}

Next, we have an analogue of Corollary \ref{corGtConvergenceProb}, whose proof is also similar.
\begin{cor}[$\hat{G}_t$ convergence] \label{corGtConvergenceProp}
If Assumptions \ref{assSSPreg} and \ref{assSSPlin} hold, all stationary policies are proper, and $\kappa_t = 54 d t^{1/4} \sqrt{N_t'}$ for some $N_t' \geq \log(3 t \chi ) / (1-\rho )$, then on the event $\mathcal{E}$, for any $t \in [T]$,
\begin{equation}
\min_{n \in [ \ceil{N_t'} ] } \| \hat{G}_t^n 0 - \hat{G}_t^{n-1} 0 \|_{\Lambda_t}  \leq \alpha_t .
\end{equation}
\end{cor}
\begin{proof}
Because $\chi \geq 1$ and $\rho \in (0,1)$ by definition, we know $N_t' \geq\log 3 \geq 1$, so $1 \leq \ceil{N_t'} \leq 2N_t'$. Combined with Lemmas \ref{lemGtConvergence} and \ref{lemUtConvergenceGeom}, we obtain
\begin{equation} \label{eqThmPropAlg}
\min_{n \in [\ceil{ N_t' } ] } \| \hat{G}_t^n 0 - \hat{G}_t^{n-1} 0 \|_{\Lambda_t} \leq \| \hat{G}_t^{\ceil{N_t'}} 0 - \hat{G}_t^{\ceil{N_t'}-1} 0 \|_{\Lambda_t} \leq 5 \sqrt{ d} \varepsilon_t + \frac{ 24 \sqrt{ t} (B_t+1) N_t' \varepsilon_t}{\alpha_t} + 3 \sqrt{t} \chi \rho^{N_t'} .
\end{equation}
On the other hand, since $6 t^{1/4} \sqrt{N_t'} \geq 1$ (recall $N_t' \geq 1$), Claim \ref{clmParameters} implies
\begin{equation} \label{eqParamThmProp}
\alpha_t \geq \max \{ 6 t^{1/4} \sqrt{N_t'} \varepsilon_t , (B_t+1) \kappa_t \} .
\end{equation}
Thus, using the assumption $N_t' \geq \log(3t\chi)/(1-\rho)$, we can bound the third term in \eqref{eqThmPropAlg} by
\begin{equation} \label{eqThmPropAlgThird}
3 \sqrt{t} \chi \rho^{N_t'}  \leq 3 \sqrt{t} {\chi} e^{-(1-{\rho})N_t'} \leq 1 / \sqrt{t} \leq 1 \leq \kappa_t (B_t+1) / 54 \leq \alpha_t / 54 . 
\end{equation}
For the second term in \eqref{eqThmPropAlg}, we again use \eqref{eqParamThmProp} to obtain
\begin{equation} \label{eqThmPropAlgSecond}
\frac{ 24 \sqrt{ t} (B_t+1) N_t' \varepsilon_t}{\alpha_t} \leq \frac{ 24 \sqrt{ t} (B_t+1) N_t'}{ 6 t^{1/4} \sqrt{N_t'} } = 4 t^{1/4} \sqrt{N_t'} (B_t+1) \leq \frac{ 4 \kappa_t (B_t+1) }{ 54 } \leq \frac{ 4 \alpha_t }{ 54 } .
\end{equation} 
Plugging the previous two inequalities into \eqref{eqThmPropAlg}, and assuming $\sqrt{d} \leq t^{1/4}$, we can use \eqref{eqParamThmProp} and $N_t' \geq 1$ to obtain
\begin{equation} \label{eqThmPropAlgCombine}
\min_{n \in [\ceil{ N_t } ] } \| \hat{G}_t^n 0 - \hat{G}_t^{n-1} 0 \|_{\Lambda_t} \leq 5 \sqrt{d} \cdot \varepsilon_t + \frac{5 \alpha_t}{54} \leq 5 t^{1/4} \cdot \frac{ \alpha_t }{ 6 t^{1/4} } + \frac{5 \alpha_t}{54}  = \frac{ 50 \alpha_t }{ 54 } < \alpha_t .
\end{equation}
If instead $\sqrt{d} > t^{1/4}$, we simply use Claim \ref{clmOpBounds} and $N_t' \geq 1$ to obtain
\begin{align} \label{eqThmPropAlgSmallT}
& \min_{n \in [\ceil{ N_t } ] } \| \hat{G}_t^n 0 - \hat{G}_t^{n-1} 0 \|_{\Lambda_t} \leq \| \hat{G}_t 0 \|_{\Lambda_t} \leq \sqrt{8td} = \sqrt{8} t^{1/4} t^{1/4} \sqrt{d} \leq \sqrt{8} t^{1/4} d \leq \kappa_t \leq \alpha_t . \qedhere
\end{align}
\end{proof}

We now prove Theorem \ref{thmProp}. As for Theorem \ref{thmProb}, it suffices to prove the guarantees on $\mathcal{E} \cap \mathcal{F}$. Corollaries \ref{corGtOptimism} and \ref{corGtConvergenceProp} establish the OAFP guarantee on $\mathcal{E} \cap \mathcal{F}$ (we choose $N_t' = N_t = \log(3 t \bar{\chi} ) / (1-\bar{\rho})$ in the latter). Next, setting $\Psi = 54 d \sqrt{ 3 \log(3 \bar{\chi}) / (1-\bar{\rho})}$ and $\lambda = 1/4$, we have
\begin{equation}
\kappa_t = 54 d t^{1/4} \sqrt{\log(3 t \bar{\chi})/ (1-\bar{\rho}) } = \Psi t^\lambda \sqrt{ \log(3 t \bar{\chi}) / ( 3 \log(3 \bar{\chi}) )} \leq \Psi t^\gamma \log(t+1) , 
\end{equation}
where the inequality holds because by $\bar{\chi} \geq 1$, we have
\begin{equation}
\frac{\log(3 t \bar{\chi})}{3 \log(3 \bar{\chi}) } = \frac{ 1 }{ 3 } + \frac{ \log(t) }{ 3 \log ( 3 \bar{\chi}) } \leq \frac{ 1 + \log(t) }{ 3 } \leq \frac{ \frac{ \log(t+1) }{ \log 2 } + \log(t) }{3} \leq \left( \frac{ \frac{1}{\log 2} + 1 }{3} \right) \log(t+1)  \leq \log(t+1) .
\end{equation}
Hence, on $\mathcal{E} \cap \mathcal{F}$, we can use the Theorem \ref{thmRegret} regret bound with this $\Psi$ and $\lambda$ to obtain
\begin{align} \label{eqThmPropReg}
R(K) 
& = \tilde{O} \left( \left( B_\star^{\frac{7}{4}} + B_\star^{\frac{3}{4}} \right) d^{\frac{3}{2}} ( K / c_{min} )^{\frac{3}{4}} N_t^{1/2} + (B_\star+1)^4 d^6 N_t^2 c_{min}^{-3}\right) .
\end{align}

\begin{rem}[Unknown $\bar{\chi}$ and $\bar{\rho}$] \label{remThmProbGen}
Suppose $\kappa_t = 54 d t^{\frac{1}{4}} \sqrt{N_t}$ as in Theorem \ref{thmProp} but $N_t = t^{2\gamma}$ as in Appendix \ref{appGenThmProp}. Then $\kappa_t \geq 9d$ for any $t \in \N$ and $N_t \geq \log(3t\chi)/(1-\rho)$ as soon as $t \geq (\log(3 t \chi) / (1-\rho) )^{\frac{1}{2\gamma}}$, so we can use Corollaries \ref{corGtOptimism} and \ref{corGtConvergenceProp} and Lemma \ref{lemMainEkTail2} to obtain the following: with probability at least $1-\delta/2$, for any $t \geq (\log(3 t \chi) / (1-\rho) )^{\frac{1}{2\gamma}}$ that Algorithm \ref{algIterateTimeout} is called, it returns an OAFP in $t^{2 \gamma}$ iterations.
\end{rem}

\subsection{Proof of Theorem \ref{thmOrth}}

We begin by showing $\hat{G}_t$ is a contraction with respect to $\| \cdot \| = \| Q^\trans\ \cdot \|_\infty$, where $Q$ is the orthogonal matrix with columns $\{ q_i \}_{i=1}^d$. Note $\| \cdot \|$ is a norm by orthogonality of $Q$.
\begin{lem}[$\hat{G}_t$ contraction] \label{lemGtContraction}
Under the assumptions of Theorem \ref{thmOrth}, for any $t \in [T]$ and $w_1 , w_2 \in \R^d$, we have
\begin{equation}
\| Q^\trans ( \hat{G}_t w_1 - \hat{G}_t w_2 ) \|_\infty \leq e^{-t/(t+1)} \|Q ( w_1 - w_2 ) \|_\infty .
\end{equation}
\end{lem}
\begin{proof}
For $(s,a) \in \S \times \A$, let $i(s,a) \in [d]$ be such that $\phi(s,a) = q_{i(s,a)}$ (which exists by assumption). For $i \in [d]$, define $d_i = | \{ \tau \in [t] : i(s_\tau,a_\tau) = i \} |$. Let $D$ be the diagonal matrix with diagonal elements $\{ d_i+1 \}_{i=1}^d$. Then
\begin{equation}
\Lambda_t = I + \sum_{\tau=1}^t \phi(s_\tau,a_\tau) \phi(s_\tau,a_\tau)^\trans = \sum_{i=1}^d  q_i q_i^\trans + \sum_{i=1}^d d_i q_i q_i^\trans = \sum_{i=1}^d ( 1 + d_i ) q_i q_i^\trans = Q D Q^\trans .
\end{equation}
This implies $\Lambda_t^{-1} = Q D^{-1} Q^\trans$, so for any $i \in [d]$ and $(s,a) \in \S \times \A$, we have
\begin{equation}
e_i^\trans Q^\trans \Lambda_t^{-1} \phi(s,a) = e_i^\trans D^{-1} Q^\trans q_{i(s,a)} = \frac{ e_i^\trans Q^\trans q_{i(s,a)} }{ d_i + 1 } = \frac{ q_i^\trans q_{i(s,a)} }{ d_i + 1 } = \frac{ \ind ( i(s,a) = i ) }{ d_i+1 } .
\end{equation}
Using this identity, we obtain
\begin{align}
e_i^\trans Q^\trans ( \hat{G}_t w_1 - \hat{G}_t w_2 ) & = \sum_{\tau=1}^t  e_i^\trans Q^\trans \Lambda_t^{-1} \phi(s_\tau,a_\tau)  ( g_t(s_\tau' , w_1 ) - g_t(s_\tau' , w_2 ) ) \\
& = \frac{\sum_{\tau \in [t] : i(s_\tau,a_\tau) = i} ( g_t(s_\tau' , w_1 ) - g_t(s_\tau' , w_2 ) )}{d_i+1}
\end{align}
On the other hand, for any $s \in \S$, we know
\begin{align}
| g_t(s , w_1 ) - g_t(s , w_2 ) |&  \leq \max_{a \in \A} | \phi(s,a)^\trans ( w_1 - w_2 ) | = \max_{a \in \A} | e_{i(s,a)}^\trans Q^\trans ( w_1 - w_2 ) | \leq \| Q^\trans ( w_1 - w_2 ) \|_\infty ,
\end{align}
where we used Claim \ref{clmGtNonex} for the first inequality. Combining the last two expressions, we obtain
\begin{equation}
\| Q^\trans ( \hat{G}_t w_1 - \hat{G}_t w_2 ) \|_\infty = \max_{i \in [d]} \left| \frac{\sum_{\tau \in [t] : i(s_\tau,a_\tau) = i} ( g_t(s_\tau' , w_1 ) - g_t(s_\tau' , w_2 ) )}{d_i+1} \right| \leq  \max_{i \in [d]} \frac{d_i}{d_i+1} \| Q^\trans ( w_1 - w_2 ) \|_\infty .
\end{equation}
This completes the proof, since $d_i / (d_i+1) \leq t/(t+1) \leq e^{-t/(t+1)}$.
\end{proof}

Using Lemma \ref{lemGtContraction}, we can show Algorithm \ref{algIterate} terminates within $O ( t \log ( td ) )$ iterations.
\begin{cor} \label{corGtConvergenceOrth}
Under the assumptions of Theorem \ref{thmOrth}, for any $t \in [T]$,
\begin{equation}
\min_{ n \in \left[ \ceil*{ 1 + (t+1) \log ( (t+1) d ) / 2 } \right] } \| \hat{G}_t^n 0 - \hat{G}_t^{n-1} 0 \|_{\Lambda_t} \leq \alpha_t .
\end{equation}
\end{cor}
\begin{proof}
For any $n \in \N$, we can iterate the bound from Lemma \ref{lemGtContraction} to obtain
\begin{equation}
\| Q^\trans ( \hat{G}_t^n 0 - \hat{G}_t^{n-1} 0 ) \|_\infty \leq e^{-\frac{n-1}{t+1}} \| Q^\trans \hat{G}_t 0 \|_\infty .
\end{equation}
By  a standard norm equivalence, orthogonality, and Claim \ref{clmOpBounds}, we also have
\begin{equation}
\| Q^\trans \hat{G}_t 0 \|_\infty \leq \| Q^\trans \hat{G}_t 0 \|_2 = \| \hat{G}_t 0 \|_2 \leq  2 \sqrt{d} .
\end{equation}
By Claim \ref{clmEigAndNorm}, orthogonality, and a standard equivalence, we also know
\begin{equation}
\| w \|_{\Lambda_t} \leq \sqrt{(t+1)d} \| w \|_2 = \sqrt{(t+1)d}  \| Q^\trans w \|_2 \leq \sqrt{(t+1) d^2} \| Q^\trans w \|_\infty\ \forall\ w \in \R^d .
\end{equation}
Hence, combining the previous three inequalities, we obtain
\begin{equation}
\| \hat{G}_t^n 0 - \hat{G}_t^{n-1} 0 \|_{\Lambda_t} \leq \sqrt{(t+1)d^2} \| Q^\trans ( \hat{G}_t^n 0 - \hat{G}_t^{n-1} 0 ) \|_\infty \leq \sqrt{(t+1)d^2} e^{-\frac{n-1}{t+1}} \| Q^\trans \hat{G}_t 0 \|_\infty \leq 2 \sqrt{(t+1) d^3} e^{-\frac{n-1}{t+1}}  .
\end{equation}
Therefore, if $n \geq 1 + (t+1) \log ( (t+1) d ) / 2$, then the previous bound, the assumed choice $\kappa_t = 9d$ in Theorem \ref{thmOrth}, and Claim \ref{clmParameters} imply $\| \hat{G}_t^n 0 - \hat{G}_t^{n-1} 0 \|_{\Lambda_t} \leq 2 d \leq \kappa_t \leq \alpha_t$.
\end{proof}

Similar to the above, on $\mathcal{E} \cap \mathcal{F}$, Corollaries \ref{corGtOptimism} and \ref{corGtConvergenceOrth} show Algorithm \ref{algIterate} returns OAFPs in $O( t \log(td) )$ iterations, and since $\kappa_t = 9d$ in Theorem \ref{thmOrth}, we obtain the regret bound from Corollary \ref{corBestCase}.

\section{Other proofs}

\subsection{Proof of Lemma \ref{lemMainEkTail2}} \label{appProofMainEkTail}

For any $t \in \N$ and $b > 0$, define the following bad event:
\begin{equation}
\mathcal{B}_{t,b} = \left\{  \sup_{ w \in [-W_t,+W_t]^d } \| E_t w \|_{\Lambda_t} > \varepsilon_t\right\} \cap \{ B_t = b \} .
\end{equation}
Our main goal is to prove the following claim.
\begin{clm} \label{clmMainEkTail}
Under the assumptions of Lemma \ref{lemMainEkTail2}, for any $t \in \N$ and $b > 0$, we have $\P ( \mathcal{B}_{t,b}  ) \leq \delta / ( 2 t (t+1)^2 )$.
\end{clm}
Before proving the claim, we show it implies the lemma. First note $B_t$ is $\{ 2^{i-1} c_{min} \}_{i=1}^t$-valued, so
\begin{equation}
\mathcal{E}^C =  \cup_{t \in \N} \left\{  \sup_{ w \in [-W_t,+W_t]^d } \| E_t w \|_{\Lambda_t} > \varepsilon_t\right\} = \cup_{t \in \N} \cup_{ b \in \{ 2^{i-1} c_{min} \}_{i=1}^t } \mathcal{B}_{t,b}  .
\end{equation}
Hence, taking union bounds over $t$ and $b$ and invoking Claim \ref{clmMainEkTail}, we obtain
\begin{equation}
\P ( \mathcal{E}^C ) \leq \sum_{t=1}^\infty \sum_{ b \in \{ 2^{i-1} c_{min} \}_{i=1}^t } \P ( \mathcal{B}_{t,b}  ) \leq \frac{\delta}{2} \sum_{t=1}^\infty \frac{1}{(t+1)^2} \leq \frac{\delta}{2} \int_{t=1}^\infty \frac{dt}{t^2} = \frac{\delta}{2} .
\end{equation}

Thus, it only remains to prove Claim \ref{clmMainEkTail}. We fix $t$ and $b$ for the remainder of this appendix. For $x \in \R^d$ and $Y \in \R_{\succ 0}^{d \times d}$ (the set $d \times d$ positive definite matrices), we define $f_{x,Y} : \S \rightarrow \R$ and $g_{x,Y} : \S \rightarrow \R$ by
\begin{align}
f_{x,Y}(s) = \min_{a \in \A} \left( \phi(s,a)^{\trans} x - \| \phi(s,a) \|_Y \right) , \quad g_{x,Y}(s) = \Pi_{[0,b]} ( f_{x,Y}(s) ) .
\end{align}
Here $\Pi_{[0,b]} ( \cdot )$ clips between $0$ and $b$ as in \eqref{eqClip}. Hence, we have the following implication:
\begin{equation}\label{eqDifferentLittleGs}
B_t = b \quad \Rightarrow \quad g_t(s,w) = g_{w , \alpha_t^2 \Lambda_t^{-1}}(s)\ \forall\ s \in \S , w \in \R^d .
\end{equation}

\begin{clm} \label{clmGtNonexGen}
Under the assumptions of Lemma \ref{lemMainEkTail2}, for any $x_1,x_2 \in \R^d$, $Y_1,Y_2 \in \R_{\succ 0}^{d \times d}$, and $s \in \S$,
\begin{equation}
| g_{x_1,Y_1} (s) - g_{x_2,Y_2}(s) | \leq \sqrt{d} \| w_1 - w_2 \|_\infty + \max_{a \in \A} \left| \| \phi(s,a) \|_{Y_1} - \phi(s,a) \|_{Y_2} \right| .
\end{equation}
\end{clm}
\begin{proof}
The proof is almost identical to Claim \ref{clmGtNonex}, except the bonus terms $\| \phi(s,a) \|_{Y_i}$ do not cancel. 
\end{proof}

We now derive a bound on the error operator that removes the bias introduced by the regularizer. Here and moving forward, for any $\tau \in [t]$, we use the shorthand $\phi_\tau = \phi(s_\tau,a_\tau)$.
\begin{clm} \label{clmEkToGwTau}
Under the assumptions of Lemma \ref{lemMainEkTail2}, if $B_t = b$, then for any $w \in \R^d$,
\begin{equation}
\| E_t w \|_{ \Lambda_t } \leq \left\| \sum_{\tau=1}^t \phi_\tau ( g_{w,\alpha_t^2 \Lambda_t^{-1}} (s_\tau') - \E_{s_\tau'} g_{w,\alpha_t^2 \Lambda_t^{-1}}(s_\tau' ) ) \right\|_{\Lambda_t^{-1}} + \sqrt{ d } ( b + 1 ) .
\end{equation}
\end{clm}
\begin{proof}
Fix $w \in \R^d$. If $B_t = b$, then by \eqref{eqDifferentLittleGs},
\begin{equation}
\hat{G}_t w = \Lambda_t^{-1} \sum_{\tau=1}^t \phi_\tau ( c(s_\tau,a_\tau) + \E_{s_\tau'} g_t(s_\tau',w) ) + \Lambda_t^{-1} \sum_{\tau=1}^t \phi_\tau ( g_{w,\alpha_t^2 \Lambda_t^{-1}} (s_\tau') - \E_{s_\tau'} g_{w,\alpha_t^2 \Lambda_t^{-1}}(s_\tau') ) .
\end{equation}
The first term can be rewritten as
\begin{equation}
\Lambda_t^{-1} \sum_{\tau=1}^t \phi_\tau \phi_\tau^\trans \left( \theta + \sum_{s \in S} \mu(s) g_t(s,w) \right)  = ( I -  \Lambda_t^{-1} ) U_t w .
\end{equation}
Combining the previous two identities, we obtain
\begin{equation}
E_t w = \hat{G}_t w - U_t w = \Lambda_t^{-1} \left( \sum_{\tau=1}^t \phi_\tau ( g_{w,\alpha_t^2 \Lambda_t^{-1}} (s_\tau') - \E_{s_\tau'} g_{w,\alpha_t^2 \Lambda_t^{-1}}(s_\tau') ) -  U_t w \right) .
\end{equation}
Thus, by the triangle inequality, we have
\begin{equation} \label{eqEkToEkprime1}
\| E_t w \|_{ \Lambda_t } \leq \left\| \sum_{\tau=1}^t \phi_\tau ( g_{w,\alpha_t^2 \Lambda_t^{-1}} (s_\tau') - \E_{s_\tau'} g_{w,\alpha_t^2 \Lambda_t^{-1}}(s_\tau') )  \right\|_{\Lambda_t^{-1}} +  \| U_t w \|_{\Lambda_t^{-1}} .
\end{equation}
This completes the proof, because when $B_t= b$, $\| U_t w \|_{\Lambda_t^{-1}} \leq  \| U_t w \|_2 \leq \sqrt{d}(b+1)$ by Claims \ref{clmEigAndNorm} and \ref{clmOpBounds}.
\end{proof}

Since $g_{w,\alpha_t^2 \Lambda_t^{-1}}$ is a random function that depends on the random state-action pairs before time $t$, we take a union bound over it using a covering argument. Toward this end, let
\begin{equation}
\alpha_t | b = ( b + 1 ) \kappa_t \sqrt{ \log( t(b+1) \kappa_t / \delta ) } , \quad W_t | b = (\alpha_t | b) + \sqrt{td}(b+1) , \quad \varepsilon_t = 5(b+1) d \sqrt{ \log ( t (\alpha_t|b) / \delta ) } ,
\end{equation}
denote the values of the random variables $\alpha_t$, $W_t$, and $\varepsilon_t$ when $B_t = b$. Thus, $\alpha_t = \alpha_t|B_t$ (and similar for $W_t$ and $\varepsilon_t$). Next, let $\mathcal{X}$ be a $1/(\sqrt{d} t)$-net of $[-W_t | b,+W_t | b]^d$ in the $\ell_\infty$ norm; explicitly, we define
\begin{equation} \label{eqDefnXeps}
\mathcal{X} = \left\{ [ i_j / (\sqrt{d} t) ]_{j=1}^d : i_j \in \left\{ -\ceil*{ (W_t | b) \sqrt{d} t },\ldots,\ceil*{ (W_t | b) \sqrt{d} t } \right\}\ \forall\ j \right\} .
\end{equation}
Finally, let $\mathcal{Y}$ be a $1/(d t^2 )$-net of  $\{ Y \in \R_{\succ 0}^{d \times d} : |Y(i,j)| \leq (\alpha_t | b)^2\ \forall\ i , j \}$, where we view the matrices as vectors:
\begin{equation}
\mathcal{Y} = \left\{  \left[ i_{j_1,j_2}  / ( d t^2  ) \right]_{j_1,j_2=1}^d  : i_{j_1,j_2} \in \left\{ - \ceil*{  ( \alpha_t | b)^2 d t^2  } , \ldots , \ceil*{ (\alpha_t | b )^2 d t^2  } \right\}\ \forall\ j_1,j_2 \right\} \cap \R_{\succ 0}^{d \times d} .
\end{equation}
Moving forward, we discard the cumbersome $\cdot|b$ notation. However, we emphasize that $\mathcal{X}$ and $\mathcal{Y}$ are deterministic sets, irrespective of the value taken by $B_t$.

We next show that when $B_t=b$, $g_{w,\alpha_t^2 \Lambda_t^{-1}}$ is close to some element of the function class $\{ g_{x,Y} : \mathcal{X} \times \mathcal{Y} \}$.
\begin{clm} \label{clmWTauToXY}
Under the assumptions of Lemma \ref{lemMainEkTail2}, if $B_t = b$, then for any $w \in [-W_t,+W_t]^d$, there exists $x \in \mathcal{X}$ and $Y \in \mathcal{Y}$ such that
\begin{equation}
\left\| \sum_{\tau=1}^t \phi_\tau ( g_{w,\alpha_t^2 \Lambda_t^{-1}} (s_\tau') - \E_{s_\tau'} g_{w,\alpha_t^2 \Lambda_t^{-1}}(s_\tau' ) ) \right\|_{\Lambda_t^{-1}} \leq \left\| \sum_{\tau=1}^t \phi_\tau ( g_{x,Y} (s_\tau') - \E_{s_\tau'} g_{x,Y}(s_\tau' ) ) \right\|_{\Lambda_t^{-1}} + 2 . 
\end{equation}
\end{clm}
\begin{proof}
By Claim \ref{clmEigAndNorm} and a standard spectral norm inequality, we have $| \alpha_t^2 \Lambda_t^{-1} (i,j) | \leq \alpha_t^2  \| \Lambda_t^{-1}  \|_2 \leq \alpha_t^2\ \forall\ i,j$. Hence, we can find $Y \in \mathcal{Y}$ such that $\max_{i,j} | Y(i,j) - \alpha_t^2 \Lambda_t^{-1} (i,j) | \leq 1 / (d t^2)$. For such $Y$ and any $(s,a) \in \S \times \A$, we then obtain
\begin{align}
| \phi(s,a)^{\trans} ( Y - \alpha_t^2 \Lambda_t^{-1} ) \phi(s,a) ) | & \leq \sum_{i,j \in [d]} | \phi_i(s,a) | | \phi_j(s,a) | | Y(i,j) - \alpha_t^2 \Lambda_t^{-1} (i,j) | \\
& \leq \frac{\| \phi(s,a) \|_1^2}{ d t^2  }\leq \frac{\| \phi(s,a) \|_2^2}{ t^2 } \leq \frac{1}{ t^2  } ,
\end{align}
which implies that
\begin{equation}
\| \phi(s,a) \|_Y \leq \sqrt{ \alpha_t^2 \phi(s,a)^\trans \Lambda_t^{-1} \phi(s,a) + | \phi(s,a)^\trans ( Y-\alpha_t^2\Lambda_t^{-1} ) \phi(s,a) | } \leq \alpha_t \| \phi(s,a) \|_{\Lambda_t^{-1}} + 1/t .
\end{equation}
Hence, by symmetry, we conclude that
\begin{equation} \label{eqInducedNormCoverBound}
\left| \| \phi(s,a) \|_Y - \alpha_t  \| \phi(s,a) \|_{\Lambda_t^{-1}} \right| \leq 1/t .
\end{equation}
Also, we can clearly find $x \in \mathcal{X}$ such that $\| w-x \|_\infty \leq 1/ (\sqrt{d}t)$. Hence, for any $s \in \S$, we obtain
\begin{equation} \label{eqWTauToXY2eps}
| g_{x,Y}(s) - g_{w,\alpha_t^2 \Lambda_t^{-1}} (s) | \leq \sqrt{d} \| w - x \|_\infty + \max_{a \in \A} \left| \| \phi(s,a) \|_Y - \alpha_t \| \phi(s,a) \|_{\Lambda_t^{-1}} \right| \leq 2 / t ,
\end{equation}
where we used Claim \ref{clmGtNonexGen}, \eqref{eqInducedNormCoverBound} and the choice of $x$. Also, defining $\Delta(s) = g_{w,\alpha_t^2 \Lambda_t^{-1}}(s) - g_{x,Y}(s)\ \forall\ s \in \S$, we have
\begin{align} 
& \left\| \sum_{\tau=1}^t \phi_\tau ( g_{w,\alpha_t^2 \Lambda_t^{-1}} (s_\tau') - \E_{s_\tau'} g_{w,\alpha_t^2 \Lambda_t^{-1}}(s_\tau' ) ) \right\|_{\Lambda_t^{-1}}  \\
& \quad\quad \leq \left\| \sum_{\tau=1}^t \phi_\tau ( g_{x,Y} (s_\tau') - \E_{s_\tau'} g_{x,Y}(s_\tau' ) ) \right\|_{\Lambda_t^{-1}}  + \left\| \sum_{\tau=1}^t \phi_\tau ( \Delta (s_\tau') - \E_{s_\tau'} \Delta(s_\tau' ) ) \right\|_{\Lambda_t^{-1}} . \label{eqCoverAlmostDone}
\end{align}
By the triangle inequality, Claim \ref{clmEigAndNorm}, and \eqref{eqWTauToXY2eps}, the second term satisfies
\begin{align}
\left\| \sum_{\tau=1}^t \phi_\tau ( \Delta (s_\tau') - \E_{s_\tau'} \Delta(s_\tau' ) ) \right\|_{\Lambda_t^{-1}}  & \leq\sum_{\tau=1}^t \| \phi_\tau \|_2 | \Delta (s_\tau') - \E_{s_\tau'} \Delta(s_\tau')  | \leq 2 . \qedhere
\end{align}
\end{proof}

Our final ingredient for proving Claim \ref{clmMainEkTail} is the following bound on $\varepsilon_t$.
\begin{clm} \label{clmChoiceOfAlpha}
Under the assumptions of Lemma \ref{lemMainEkTail2}, if $B_t = b$, then 
\begin{align} \label{eqChoiceOfAlpha}
\varepsilon_t \geq \sqrt{ 2 b^2 \log \left( \sqrt{\frac{ det(\Lambda_t) }{ det (  \Lambda_0 ) }} \frac{ 2t (t+1)^2 |\mathcal{X}| |\mathcal{Y}| }{ \delta }  \right) }  +  \sqrt{ d}(b+1) + 2  .
\end{align}
\end{clm}
\begin{proof}
We first observe that by assumption $\kappa_t \geq 9 d$ and $d \geq 2$, we have
\begin{equation} \label{eqAlphaLower}
\alpha_t =  \kappa_t (b+1) \sqrt{ \log( t (b+1) \kappa_t / \delta)} \geq 9 d (b+1) \geq 9d \geq 18  .
\end{equation}
Using this bound, we (coarsely) bound the sizes of the nets. For $\mathcal{X}$, we first recall that $W_t = \alpha_t + \sqrt{td}(b+1)$, so by \eqref{eqAlphaLower} and $(1/\sqrt{9}) + (1/9) = (1/3) + (1/9) = 4/9$, we have
\begin{equation}
W_t \sqrt{d} t = \alpha_t \sqrt{d} t + (b+1) d t^{3/2} \leq \alpha_t \sqrt{ \alpha_t / 9 } t + ( \alpha_t / 9 ) t^{3/2}  \leq 4 (\alpha_t t )^{3/2} / 9 .
\end{equation}
Again using \eqref{eqAlphaLower}, we have $3 \leq 3 ( \alpha_t t)^{3/2} / 18^{3/2} \leq (\alpha_t t)^{3/2} / 9$. Thus, because $d \geq 2$, we obtain
\begin{equation}
| \mathcal{X} | \leq ( 1 + 2 \ceil{ W_t \sqrt{d} t } )^d \leq ( 3 + 2 W_t \sqrt{d} t )^d \leq ( \alpha_t t )^{3d/2} \leq ( \alpha_t t )^{d^2} .
\end{equation}
For $\mathcal{Y}$, we can use \eqref{eqAlphaLower} to obtain $3 \leq 3 \alpha_t^3 t^2 / 18^3 \leq \alpha_t^3 t^2 / 4$ and $2 d \leq \alpha_t / 4$, so
\begin{equation} 
|\mathcal{Y}| \leq  ( 3 + 2 d t^2 \alpha_t^2 )^{d^2} \leq ( \alpha_t^3 t^2 / 2 )^{d^2} = 2^{-d^2} \alpha_t^{3 d^2 } t^{2 d^2 }  .
\end{equation}
Next, observe $det ( \Lambda_t ) / det ( \Lambda_0 ) \leq (t+1)^d$ by Claim \ref{clmEigAndNorm}, so again using $d \geq 2$, we have
\begin{equation}
2 t  (t+1)^2 \sqrt{ det ( \Lambda_t ) / det ( \Lambda_0 ) } \leq 2 t ( t+1)^{2 + d / 2 } \leq ( 2 t )^{ 3 + d / 2 } \leq ( 2 t )^{ (3d^2/4) + (d^2/4) } = ( 2t )^{d^2 } .
\end{equation}
Combining the previous three inequalities, and since $\delta \geq \delta^{4 d^2}$, we obtain
\begin{equation}
2 t  (t+1)^2 \sqrt{ det ( \Lambda_t ) / det ( \Lambda_0 ) } | \mathcal{X} | | \mathcal{Y} | / \delta \leq ( \alpha_t  t / \delta )^{4 d^2  } .
\end{equation}
Since $d \geq 2$, we also have $\sqrt{d}(b+1) + 2 \leq 2 (b+1) d$. Combined with the previous inequality,
\begin{align}
& \sqrt{ 2 b^2 \log \left( \sqrt{\frac{ det(\Lambda_t) }{ det (  \Lambda_0 ) }} \frac{ 2t (t+1)^2 |\mathcal{X}| |\mathcal{Y}| }{ \delta }  \right) }  +  \sqrt{ d}(b+1) + 2 \\
& \quad \leq \sqrt{ 8 b^2 d^2 \log( \alpha_t t / \delta ) } + 2 (b+1) d \leq ( \sqrt{8} + 2 ) (b+1) d \sqrt{ \log(\alpha_t t / \delta ) }  \leq 5 (b+1) d \sqrt{ \log(\alpha_t t / \delta ) } = \varepsilon_t . \qedhere
\end{align}
\end{proof}

\begin{rem}[Sharpening the tabular case] \label{remSharpeningTabular1}
In the tabular case, $\Lambda_t^{-1}$ is diagonal, so we can replace $\mathcal{Y}$ with $\mathcal{Y}' = \{ Y \in \mathcal{Y} : Y \text{ is diagonal} \}$. Since $|\mathcal{Y}'|$ is exponential in $d$ (instead of $d^2$), we can define $\varepsilon_t$ to have square root (instead of linear) dependence on $d$.
\end{rem}

\begin{proof}[Proof of Claim \ref{clmMainEkTail}]
For each $(x,Y) \in \mathcal{X} \times \mathcal{Y}$, define the event
\begin{equation}
\mathcal{C}_{x,Y} = \left\{  \left\| \sum_{\tau=1}^t \phi_\tau ( g_{x,Y} (s_\tau') - \E_{s_\tau'} g_{x,Y}(s_\tau' ) ) \right\|_{\Lambda_t^{-1}}  > \sqrt{ 2 b^2 \log \left( \sqrt{\frac{ det(\Lambda_t) }{ det (  \Lambda_0 ) }} \frac{ 2t (t+1)^2 |\mathcal{X}| |\mathcal{Y}| }{ \delta }  \right) } \right\} . 
\end{equation}
Then since $g_{x,Y}$ is a deterministic $[0,b]$-valued function, $g_{x,Y} (s_\tau') - \E_{s_\tau'} g_{x,Y}(s_\tau')$ are conditionally zero-mean $[-b,b]$-valued random variables, so are $b$-subgaussian. Hence, by \cite[Theorem 1]{abbasi2011improved}, we have $\P (  \mathcal{C}_{x,Y,b} ) \leq \delta / ( 2 t (t+1)^2 | \mathcal{X} | | \mathcal{Y}| )$. Combined with Claims \ref{clmEkToGwTau}, \ref{clmWTauToXY}, and \ref{clmChoiceOfAlpha} and the union bound,
\begin{align}
\P ( \mathcal{B}_{t,b}  ) & \leq \P ( \cup_{(x,Y) \in \mathcal{X} \times \mathcal{Y}} \mathcal{C}_{x,Y} \cap \{ B_t = b \} ) \leq \P ( \cup_{(x,Y) \in \mathcal{X} \times \mathcal{Y}} \mathcal{C}_{x,Y} ) \leq \sum_{(x,Y) \in \mathcal{X} \times \mathcal{Y}} \P (  \mathcal{C}_{x,Y}  ) \leq \frac{\delta}{2t(t+1)^2} . \qedhere 
\end{align}
\end{proof}

\subsection{Proof of Proposition \ref{propFeatureBellman}} \label{appProofFeatureBellman}

By Assumption \ref{assSSPlin} and the definition of $Q^\star$ \eqref{eqQstar}, for any $s \in \S$, we have
\begin{equation}
\phi(s,a)^\trans w^\star = \phi(s,a)^\trans \left( \theta + \sum_{ s' \in \S }  J^\star(s') \mu(s') \right)  = c(s,a) + \sum_{ s' \in \S }  J^\star(s')P(s'|s,a) = Q^\star(s,a) .
\end{equation}
Hence, by the Bellman optimality equations \eqref{eqBellman},
\begin{gather} \label{eqFeatureBellman}
 J^\star(s) = \min_{a \in \A}  Q^\star(s,a)  = \min_{a \in \A} \phi(s,a)^\trans w^\star   , \quad \pi^\star(s) \in \argmin_{a \in \A} Q^\star(s,a) = \argmin_{a \in \A} \phi(s,a)^\trans w^\star .
\end{gather} 
The first equality also implies that $w^\star$ is a fixed point of $G$:
\begin{align}
G w^\star & = \theta + \sum_{ s \in \S } \mu(s) \min_{a \in \A} \phi(s,a)^\trans w^\star = \theta + \sum_{ s \in \S } \mu(s) J^\star(s) = w^\star .
\end{align}

\end{document}